\documentclass[onefignum,onetabnum]{siamonline190516}

\usepackage{lipsum}
\usepackage{amsfonts}
\usepackage{graphicx}
\usepackage{epstopdf}
\usepackage{algorithmic}
\ifpdf
  \DeclareGraphicsExtensions{.eps,.pdf,.png,.jpg}
\else
  \DeclareGraphicsExtensions{.eps}
\fi

\usepackage[T1]{fontenc}
\usepackage[utf8]{inputenc}

\usepackage{booktabs} 
\usepackage{adjustbox}

\usepackage{enumitem}
\setlist[enumerate]{leftmargin=.5in}
\setlist[itemize]{leftmargin=.5in}

\newsiamremark{remark}{Remark}
\newsiamremark{hypothesis}{Hypothesis}
\crefname{hypothesis}{Hypothesis}{Hypotheses}
\newsiamthm{claim}{Claim}

\headers{Wasserstein-based Projections with Applications to Inverse Problems}{H. Heaton, S. Wu Fung, A. Lin, S. Osher, W. Yin}

\title{Wasserstein-based Projections with Applications to Inverse Problems\thanks{Resubmitted to arXiv on April 14, 2021. 
\funding{SWF, ATL, and SO were supported by AFOSR MURI FA9550-18-1-0502, AFOSR Grant No. FA9550-18-1-0167, and ONR Grants N00014-18-1-2527 snf N00014-17-1-21. 
HH's work was supported by the National Science Foundation (NSF) Graduate Research Fellowship under Grant No. DGE-1650604. Any opinion, findings, and conclusions or recommendations expressed in this material are those of the authors and do not necessarily reflect the views of the NSF.}}}

\author{
Howard Heaton{\thanks{Equal contribution}\:\:\thanks{Department of Mathematics, UCLA, Los Angeles, CA (\email{hheaton@ucla.edu}\;, \email{swufung@math.ucla.edu}, \email{atlin@math.ucla.edu}, \email{sjo@math.ucla.edu}, \email{wotaoyin@math.ucla.edu}).}}
\hspace{-6pt}
\and 
Samy Wu Fung\footnotemark[2]\:\:\footnotemark[3]
\and 
Alex Tong Lin\footnotemark[2]\:\:\footnotemark[3]
\and 
Stanley Osher\footnotemark[3]  
\and 
Wotao Yin\footnotemark[3]  
}

\usepackage{amsopn}

\usepackage{braket}
\newtheorem{assumption}[theorem]{Assumption}

\usepackage{tikz}
\usepackage{subfig}
\usetikzlibrary{tikzmark,calc,,arrows,shapes,decorations.pathreplacing}
\tikzset{every picture/.style={remember picture}}
\usetikzlibrary{fit,shapes.misc}
\usetikzlibrary{positioning,backgrounds,spy}

\usepackage{xr-hyper}
\usepackage{hyperref}
\definecolor{bl}{RGB}{30,30,150}
\hypersetup{colorlinks=true, citecolor=bl, linkcolor=bl, urlcolor=bl}
\usepackage{}

\usepackage{xcolor}
\definecolor{projcolor}{rgb}{1,0.0,0.1}
\definecolor{avecolor}{rgb}{0.6,0.5,0.0} 
\definecolor{distcolor}{rgb}{0.1,0.4,0.9} 
\definecolor{gcolor}{rgb}{0.5,0.0,0.7}
\usepackage{amssymb}

\newenvironment{remarkx}
{
    \begin{remark}
}%
{
    \hfill $\blacklozenge$ \end{remark} 
}

\newcommand{\limk}{\lim_{k\rightarrow\infty}}

\newcommand{\newparagraph}{  \\[-6pt] }
\def\W{1.25}
\def\spycolor{projcolor!75!black}

\newcommand{\prox}[1]{\mathrm{prox}_{#1}}

\DeclareMathOperator*{\argmin}{argmin}

\newcommand{\PU}{ \bbP_{\mathrm{true}} }

\newcommand{\sA}{{\cal A}}

\newcommand{\sD}{{\cal D}}

\newcommand{\sF}{{\cal F}}

\newcommand{\sI}{{\cal I}}

\newcommand{\sM}{{\cal M}}
  
\newcommand{\sP}{{\cal P}}

\newcommand{\sX}{{\cal X}}
\newcommand{\sY}{{\cal Y}}

\newcommand{\bbN}{{\mathbb N}} 
 
\newcommand{\bbR}{{\mathbb R}}
\newcommand{\bbE}{{\mathbb E}}
\newcommand{\bbP}{{\mathbb P}}

\newcommand{\ie}{\textit{i.e.}\ }
\newcommand{\eg}{\textit{e.g.}\ }

\ifpdf  
\hypersetup{
  pdftitle={Wasserstein-based Projections with Applications to Inverse Problems}, 
  pdfauthor={Howard Heaton, Samy Wu Fung, Alex Tong Lin, Stanley Osher, Wotao Yin}
} 
\fi

\begin{document}

\maketitle

\begin{abstract}
    Inverse problems consist of recovering a signal from a collection of noisy measurements. These are typically cast as optimization problems, with classic approaches using a data fidelity term and an analytic regularizer that stabilizes recovery.  
    Recent  Plug-and-Play (PnP) works propose replacing the operator for analytic regularization in optimization methods by a  data-driven denoiser. 
    These schemes obtain state of the art results, but at the cost of   limited theoretical guarantees. To bridge this gap, we present a new algorithm that takes samples from the manifold of true data as input and outputs an approximation of the projection operator onto this manifold. 
    Under standard assumptions, we prove this algorithm  generates a learned operator, called Wasserstein-based projection (WP), that approximates the true projection with high probability.
    Thus, WPs can be inserted into optimization methods in the same manner as PnP, but now with theoretical guarantees.  Provided numerical examples show WPs obtain state of the art results for unsupervised PnP signal recovery.\footnote{All codes for this work can be found at: \url{https://github.com/swufung/WassersteinBasedProjections}.}    
\end{abstract}



\section{Introduction}
    Inverse problems arise in numerous applications such as medical imaging~\cite{arridge1999optical,arridge2009optical,hansen2006deblurring,osher2005iterative}, phase retrieval~\cite{bauschke2002phase,candes2015phase,fung2020multigrid,thao2020phase}, geophysics~\cite{bui2013computational, fung2019multiscale,fung2019uncertainty,haber2000fast,haber2004inversion,kan2020pnkh}, and machine learning~\cite{cucker2002best,fung2019large,haber2017stable,vito2005learning,fung2020admm}. The underlying goal of inverse problems is to recover a signal\footnote{We use the term \textit{signal} to  describe objects of interest that can be represented mathematically (\eg images, parameters of a differential equation, and points in a Euclidean space).} from a collection of indirect noisy measurements. 
    Formally stated, consider a finite dimensional Hilbert space $\sX$ (\eg $\bbR^n$) with scalar product $\braket{\cdot,\cdot}$ and norm $\|\cdot\|$ for the domain space, and similarly for  the measurement space $\sY$ (\eg $\bbR^m$). 
    Let $A: \mathcal{X} \to \mathcal{Y}$ be a mapping between $\mathcal{X}$ and $\mathcal{Y}$, and let $b \in \sY$ be the available measurement data given by
    \begin{equation}
        b = A(u^\star) + \varepsilon,        
    \end{equation}
    where $u^\star \in \mathcal{X}$ denotes the true signal and $\varepsilon \in \sY$ denotes the noise in the measurement. The specific task of inverse problems  is to recover $u^\star$ from the noisy measurements $b$. 
    A difficulty in recovering $u^\star$ is that inverse problems are often ill-posed, making their solutions unstable for noise-affected data. To overcome ill-posedness, many traditional approaches estimate the true signal $u^\star$ by a solution $\tilde{u}$ to the variational problem 
    \begin{equation}
        \min_{u \in \mathcal{X}} \; \ell(A(u), b) + J(u),\label{eq: variational-problem}
    \end{equation}
    where\footnote{Here we use $\overline{\bbR} \triangleq \bbR \cup \{\infty\}$.} $\ell \colon \sY \times \sY \to \overline{\bbR}$ is the fidelity term that measures the discrepancy between the measurements and the application of the forward operator $A$ to the signal estimate (\eg least squares). The function $J\colon \mathcal{X} \to \overline{\bbR}$ serves as a regularizer, which ensures both that the solution to (\ref{eq: variational-problem}) is unique and that its computation is stable. 
    In addition to ensuring well-posedness, regularizers are constructed in an effort to instill prior knowledge of the true signal. 
    Common model-based regularizers include, \eg sparsity $J(u) = \|u\|_1$~\cite{beck2009fast,candes2006quantitative,candes2006robust,donoho2006compressed}, Tikhonov $J(u) = \|u\|^2$~\cite{calvetti2003tikhonov,golub1999tikhonov}, Total Variation $J(u) = \| \nabla u \|_1$~\cite{chan2020two,rudin1992nonlinear,yin2008bregman}, and recently, data-driven regularizers~\cite{adler2018learned,kobler2017variational,lunz2019adversarial}.   
    A generalization of using data-driven regularization consists of Plug-and-Play (PnP) methods~\cite{chan2016plug,cohen2020regularization,venkatakrishnan2013Plug}, which replace the proximal operators in an optimization algorithm with data-driven operators. (This generalizes data-driven regularization since, in some instances, PnP operators cannot be expressed as the proximal of any regularizer.)
    \newparagraph
          
    An underlying theme of regularization is that signals represented in high dimensional spaces often exhibit redundancies.  
    For example, in a piece-wise constant image, any given pixel value is often highly correlated with adjacent pixel values.
    This leads to the insight that such signals possess intrinsically low dimensional manifold\footnote{\textit{Manifold} is used loosely in the widespread sense of machine learning rather than a topological definition.} $\sM$ representations~\cite{maaten2008visualizing,osher2017low,peyre2009manifold}.  
    However, explicitly approximating the manifold is highly nontrivial. Thus, a key question remains: \newparagraph
    
    \begin{center} 
        \textit{How can we guarantee reconstructed signal estimates are on the manifold of true signals?} \newparagraph \       
    \end{center}     
    This reconstruction guarantee can be ensured by solving a special  case of \eqref{eq: variational-problem} given by  
    \begin{equation}
        \min_{u\in\sM} \ell(A(u),b).
        \label{eq: VP-manifold}
    \end{equation}
    Standard   methods for solving the constrained problem (\ref{eq: VP-manifold}) require repeatedly projecting signal estimates onto the manifold $\sM$  (see Subsection \ref{subsec: motivation}).  However, in practice access is typically given to samples drawn from $\sM$ rather than an explicit representation of the projection $P_\sM$.

    \paragraph{Contribution} 
    The main difficulty in solving~\eqref{eq: VP-manifold} is that this task typically requires   projecting onto the manifold $\sM$.
    The key contribution of the present work is to provide the first algorithm, to our knowledge, that takes samples of $\sM$ as input and outputs a PnP operator that provably approximates the  projection operator $P_{\sM}$.\footnote{We assume the manifold is convex and the signals are not ``too noisy'' (see Assumptions~\ref{ass: manifold-support} and~\ref{ass: projection-manifold-recovery}).} We refer to the approximation    as a Wasserstein-based projection (WP).
     Once this approximation of $P_\sM$ is obtained, it can be directly incorporated into any optimization method for solving (\ref{eq: VP-manifold}). 
     We also emphasize our approach is  \textit{unsupervised} and  does not require directly representing the manifold.\footnote{Here unsupervised means a direct correspondence between noisy measurements and true signals is \textit{not} needed (\eg the number of noisy measurement data may differ from the number of true signal samples).}  
    

    \paragraph{Outline}
    We begin by describing Wasserstein-based projections in Section~\ref{sec:advProjections}. 
    The convergence analysis and some practical notes are covered in Section~\ref{sec:convergence-analysis}.
    We discuss related works in Section~\ref{sec:relatedWorks}.  
    We show the effectiveness of Wasserstein-based projections on low-dose CT problems and conclude with a brief discussion in Section~\ref{sec:conclusion}.

\section{Wasserstein-based Projections} \label{sec:advProjections} 
    In this section, we provide relevant background of projections for inverse problems. We then describe our proposed scheme to learn the projection (Algorithm \ref{alg: training}) and deploy it (Algorithm \ref{alg: adversarial-projection}).
    \subsection{Motivation} \label{subsec: motivation}
     Many standard optimization algorithms can be used to generate a sequence $\{u^k\}$ that converges to a solution $\tilde{u}_b$ of (\ref{eq: VP-manifold}).
     For example, fixing a step size $\lambda\in(0,\infty)$, the proximal gradient method uses updates of the form
     \begin{equation}
         u^{k+1} = P_{\sM}\left(u^k - \lambda \nabla \ell(A(u^k),b)\right).
     \end{equation}
         Letting $f(d) \triangleq \ell(d,b)$ and $\beta,\gamma \in (0,\infty)$, linearized ADMM \cite{ouyang2015accelerated} updates take the form
     \begin{subequations}
    \begin{align} 
        u^{k+1} & = P_{\sM}\left( u^k - \beta A^\top (\nu^k + Au^k - y^k  \right), \\
        y^{k+1} & =  \prox{\gamma f}\left( y^k + \gamma (\nu^k +  Au^{k+1}-y^k \right) ,\\
        \nu^{k+1} & = \nu^k + Au^{k+1}-y^{k+1}.
    \end{align}
     \end{subequations}
     The primal dual hybrid gradient (PDHG) method \cite{chambolle2011first,esser2010general} produces updates of the form 
     \begin{subequations}
        \begin{align}
        u^{k+1} & = P_{\sM}\left( u^k - \gamma A^\top \nu^k \right) \\
        \nu^{k+1} & = \prox{\beta f^*}\left(  \nu^k + \beta A(2u^{k+1}-u^k)  \right),
        \end{align}
     \end{subequations} 
     where $f^*$ is the convex conjugate\footnote{The proximal for $f^*$ can be directly obtained from that of $f$ using the Moreau decomposition \cite{moreau1965proximite}.} of $f$.
     Each example above (and many others) requires iteratively applying the projection operator $P_{\sM}$ to solve \eqref{eq: VP-manifold}.
     However, difficulty arises in most practical settings  where access is not provided to an analytic expression for $P_{\sM}$. 
     As is common in data science, access is typically provide to samples of the true signals in $\sM$. 
     To our knowledge, this work presents the first algorithm that provably approximates the projection $P_\sM$ from available sampled data, which we present below. The philosophical distinction between   classic and the data-driven schemes is further illustrated in Figure \ref{fig: classic-vs-learned}. 
    
    \begin{figure}
        \centering
        \subfloat[Classic Regularization]{
        \includegraphics[width=0.45\textwidth]{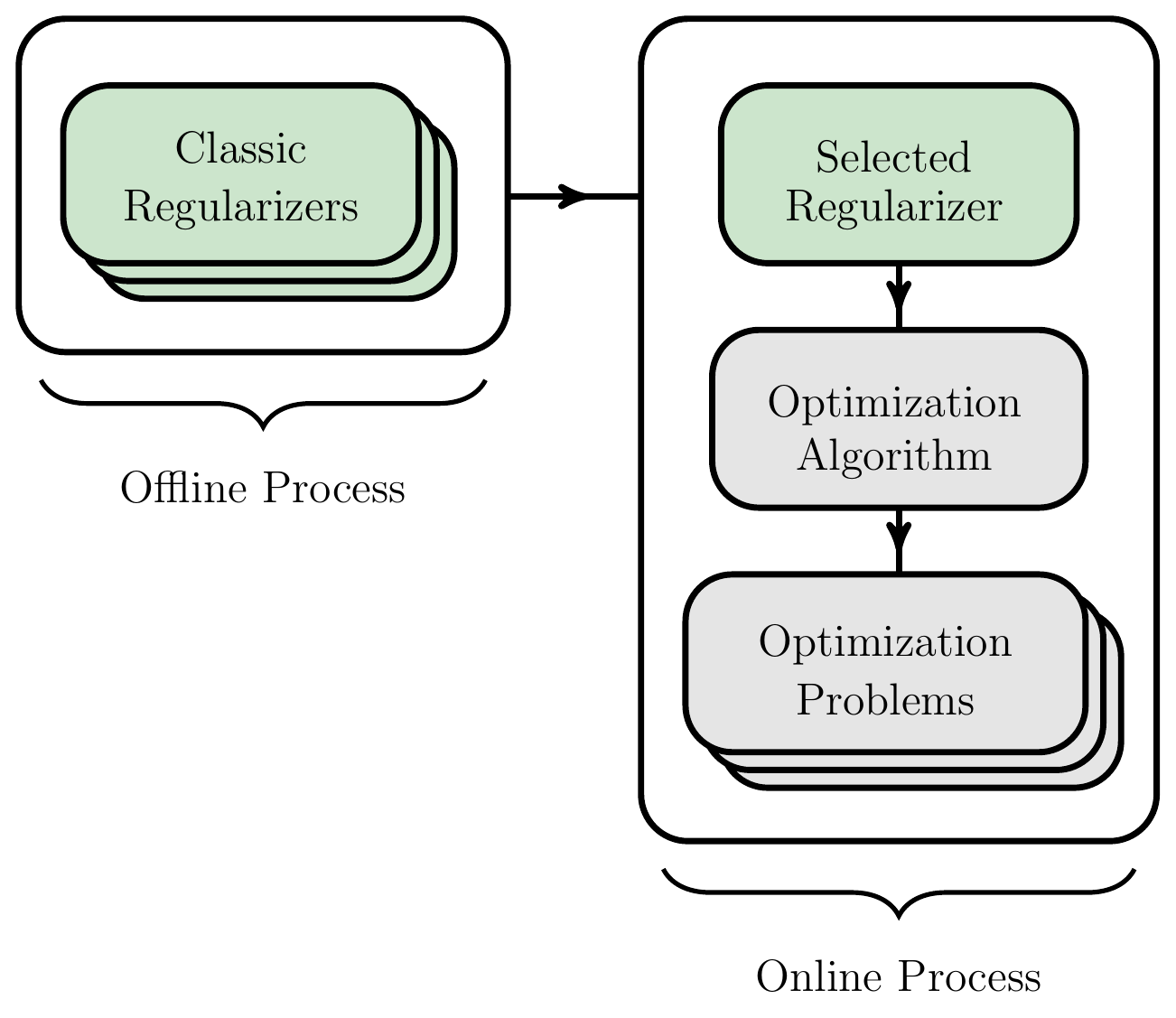}}
        \hspace{5pt}
        \subfloat[Learned Projection]{
        \includegraphics[width=0.49\textwidth]{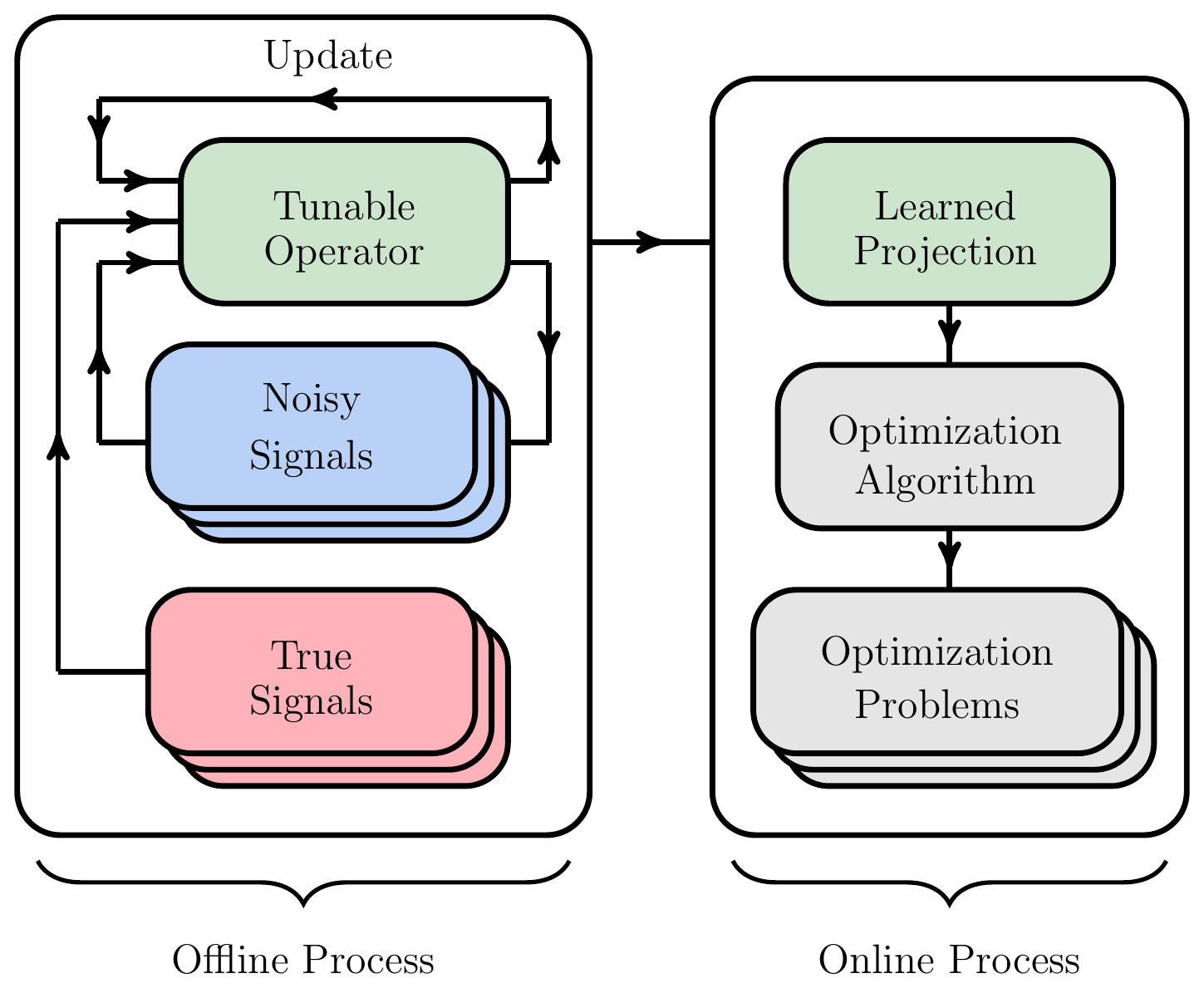}}        
        \caption{Comparison of offline/online processes for classic and learned projection schemes. Classic approaches select an analytic regularizer and incorporate it into an optimization algorithm, which is then applied to solve a class of optimization problems. Constructing the learned projection consists of iteratively refining a tunable operator using true and noisy samples. As the tunable operator is updated, the noisy signals are also updated, each time making the noisy signals more closely resemble the true signals. During the online process, the learned projection is deployed in an optimization algorithm analogously to the classic method. The key contribution of this paper is to present an algorithm for the offline process used to obtain the learned projection operator.}
        \label{fig: classic-vs-learned}
    \end{figure}
    
    \subsection{Projection Background}
    For a convex and compact set $\sM\subset \sX$,   the projection operator  $P_{\sM}:\sX\rightarrow\sX $   is defined by 
    \begin{equation}
        P_{\sM}(u) \triangleq \argmin_{v\in\sM} \;  \|v-u\|.
    \end{equation}   
    This projection can be used to express the pointwise distance function
    \begin{equation}
        d_{\sM}(u) \triangleq \inf_{v\in\sM} \| v-u\| = \| P_{\sM}(u)-u\|.
    \end{equation} 
    Indeed, for $\alpha \in \bbR$ and $\lambda = \alpha \cdot d_{\sM}(u)$, we obtain the inclusion relation\footnote{The case for $u\notin \sM$ follows from \cite[Lemma 2.2.28]{cegielski2012iterative}, and for $u\in\sM$ observe $P_{\sM}(u) - u = 0 \in \partial d_{\sM}(u)$.}
    \begin{equation}
        u + \alpha \left(P_{\sM}(u) - u\right)
        \in u - \lambda \partial d_{\sM}(u),
        \label{eq:relaxed-projection}
    \end{equation} 
    and the left hand side is called the $\alpha$-relaxed projection of $u$ onto $\sM$. 
    In particular, we can directly obtain the left hand side from the subgradient expression on the right (see (\ref{eq: gk-def})), where $0 \in \partial d_\sM(u)$ whenever $d_\sM$ is not differentiable at $u$. \newparagraph 
    
    We next introduce how to   approximate the distance function $d_{\sM}$.
    To do this, we solve an optimization problem over the set $\Gamma$ of nonnegative 1-Lipschitz functions (n.b. $d_{\sM} \in \Gamma$). 
    We assume access is provided to a collection of noisy signals $\tilde{u}$ drawn from a distribution\footnote{We  use ``distribution'' and ``measure'' interchangeably, using the former to emphasize the collection of signals and latter to emphasize assigned probabilities.} $\widetilde{\bbP}$.
    Similarly, we assume a distribution of true signals $\PU$ is provided\footnote{Typical  examples  of $\bbP_{\mathrm{true}}$ consist  of  true  signals/images for training data. Examples  of $\tilde{\bbP}$ consist of images reconstructed from a given distribution of observed measurements $b$ (\eg  TV-based reconstructions).}.
    With additional assumptions (see Section \ref{sec:convergence-analysis}), $\tau \in (0,\infty)$, and $p\in[1,\infty)$, each function that solves the optimization problem
    \begin{equation}
       \min_{f\in\Gamma}\;  \bbE_{u\sim\PU} \left[ f(u) + \tau f(u)^p \right] -\bbE_{\tilde{u}\sim\widetilde{\bbP}} \left[ f(\tilde{u}) \right]
        \label{eq:distance-inclusion-sup}
    \end{equation}
    coincides with the distance function $d_{\sM}$ over the support of $\widetilde{\bbP}$. 
    That is, if $\phi$ solves (\ref{eq:distance-inclusion-sup}), then $\phi(\tilde{u}) = d_{\sM}(\tilde{u})$ for all $\tilde{u} \in \mbox{supp}(\widetilde{\bbP})$.
    Since we are only interested in computing projections of signals $\tilde{u}$ drawn from the noisy distribution $\widetilde{\bbP}$, any solution $\phi$ to (\ref{eq:distance-inclusion-sup}) will work for our purposes.
    In summary, our task is to   solve the   problem (\ref{eq:distance-inclusion-sup}) and then use a subgradient of this solution with (\ref{eq:relaxed-projection}) to   form a relaxed projection.  \newparagraph

    Having obtained the pointwise distance function $d_{\sM}$, we propose estimating each projection onto the manifold $\sM$ by performing a sequence of subgradient descent steps.
    In an ideal situation, given a point $u\notin \sM$, we would perform a single step of size $d_{\sM}(u)$ in the direction\footnote{The distance function $d_{\sM}$ is differentiable outside of the manifold $\sM$.} $-\nabla d_{\sM}(u)$ to obtain the projection $P_{\sM}(u)$ in a single step.
    However, this is not a useful approach due to our limited ability to approximate $d_{\sM}$ in practice (\eg due to finite sampling of true signals and minimizing over a subset of $\Gamma$). Instead, we prescribe an iterative scheme with   dynamic step sizes, which (hopefully) overcomes most of the approximation errors obtained in practice.        
    Given an anchoring sequence of scalars $\{\gamma_k\} \subset (0,1)$ and a 1-Lipschitz (\ie nonexpansive) operator $T:\sX\rightarrow\sX$, Halpern \cite{halpern1967fixed} proposed finding the projection of $u^1 = \tilde{u}$ onto the fixed point set  $\mathrm{fix}(T) \triangleq \{ u : u=T(u)\}$ by generating a sequence $\{u^k\}$ via the update  
    \begin{equation} 
        u^{k+1} = \gamma_k u^1 + (1-\gamma_k) T(u^k), \ \ \ \mbox{for all $k\in\bbN$},
        \label{eq: halpern-update}
    \end{equation}
    where each update is a convex combination of $u^1$ and $T(u^k)$.
    Our method takes the form
    \begin{equation}
        u^{k+1} = \gamma_k u^1 + (1-\gamma_k)\Big( u^k + \alpha_k (P_{\sM}(u^k) - u^k) \Big), \ \ \ \mbox{for all $k\in\bbN$.}
        \label{eq: adv-proj-update}
    \end{equation} 
    The expression in (\ref{eq: adv-proj-update}) that replaces $T(u^k)$ in (\ref{eq: halpern-update}) corresponds to a nonexpansive operator provided $\alpha_k \in [0,2]$. For a typical update, we  use $\alpha_k\in[0,2]$; however, this is not always the case for our step size rule (see Remark \ref{remark: step-size}). 
    Algorithm \ref{alg: training}  articulates the training procedure for identifying the parameters in our Wasserstein-based projection algorithm (Algorithm \ref{alg: adversarial-projection}).
    
    \begin{remarkx}
        The pointwise distance function $d_{\sM}(u)$ is distinct from the  Wasserstein-1 distance $\mbox{Wass}(\widetilde{\bbP}, \PU)$ between the distribution $\widetilde{\bbP}$ of estimates and the true signal distribution $\PU$.
        The former measures the distance from an individual point to a set while the latter is a metric  for distributions.   
        The connection between these, in our setting, is that the expected value of the distance to the manifold among all $\tilde{u}\sim \widetilde{\bbP}$ is equivalent to the Wasserstein-1 distance,\footnotemark \ie 
        \begin{equation}
            \bbE_{\tilde{u}\sim \widetilde{\bbP}}\left[ d_{\sM}(\tilde{u})\right]
            = \mbox{Wass}(\widetilde{\bbP}, \PU).
        \end{equation} 
        This also illustrates that the name {\it Wasserstein-based projection} (WP) is derived from the fact that the loss function we use in (\ref{eq:distance-inclusion-sup}) coincides with the Wasserstein distance.
    \end{remarkx}
    \footnotetext{{This   follows from Lemma \ref{lemma:betak}  and the dual Kantarovich formulation of the Wasserstein-1 distance.}}

    \begin{remarkx}
        Throughout this work, we use probability spaces   $(\Omega, \sF, \bbP)$, where $\Omega$ is the sample space, $\sF$ the $\sigma$-algebra, and $\bbP:\sF\rightarrow\bbR$ is the probability measure.
        Everywhere in this work, we assume the sample space is the entire Hilbert space (\ie $\Omega = \sX$) and the event space is the power set (\ie $\sF = \sP(\Omega)$). We will use various measures.
        In particular, we assume a measure $\PU$ is provided for true data, which in practice is approximated through discrete sampling. Additionally, we find it practical to introduce a distribution $\bbP^k$ for each iteration $k$ of our algorithm so that we may write $u^k \sim \bbP^k$. 
        Note each distribution $\bbP^{k}$ is formed as a push forward operation of compositions of our algorithmic operator on $\bbP^1$ (see Remark \ref{remark: probability-push-forward}).
    \end{remarkx}

    \subsection{Algorithms}
  \begin{algorithm}[t]
    {\small
    \caption{Training to generate parameters for distance function $d_{\sM}$  estimates}
    \label{alg: training}
    \begin{algorithmic}[1]
        \STATE{\begin{tabular}{p{0.585\textwidth}r}
         \hspace*{-8pt} Choose nonzero parameter $\mu \in \{ \tilde{\mu}\in\bbR^2_{\geq 0} : \tilde{\mu}_1+\tilde{\mu}_2 < 2 \}$
         &  $\vartriangleleft$ Region bounded by simplex
         \end{tabular}}
          
        \STATE{\begin{tabular}{p{0.585\textwidth}r}
         \hspace*{-8pt} Choose class of function parameterizations $\sI$
         & 
         $\vartriangleleft$ See Assumption \ref{ass: lip}
         \end{tabular}}
         
        \STATE{\begin{tabular}{p{0.585\textwidth}r}
         \hspace*{-8pt} Choose anchoring sequence $\{\gamma_k\}\subset(0,1]$
         & 
         $\vartriangleleft$ See Assumption \ref{ass: gammak-props}
         \end{tabular}}         
         
        \STATE{\begin{tabular}{p{0.585\textwidth}r}
         \hspace*{-8pt} Choose signals $\{u^1_i\}_{i \in \sD}\subseteq \sX$ for initial distribution $\bbP^1$
         & 
         $\vartriangleleft$  $\sD$ is an index set
        \end{tabular}}
                  
    \STATE{{\bf for} $k=1,2,\ldots$}
        \vspace*{2pt}
        \STATE{\begin{tabular}{p{0.585\textwidth}r}
                \ \ $\theta^k \leftarrow \displaystyle \argmin_{\theta\in\mathcal{I}} \; \bbE_{u\sim\PU} \left[ J_\theta(u) + \tau J_\theta(u)^p\right] - \bbE_{u^k \sim \bbP^k} \left[ J_\theta(u^k) \right]$ 
                & {$\vartriangleleft$ Train to find weights}
                \end{tabular}}
                
        \STATE{\begin{tabular}{p{0.585\textwidth}r}
            \ \ $\beta_k \leftarrow  \displaystyle \bbE_{u^k\sim\bbP^k} \left[ J_{\theta^k}(u^k) \right] - \bbE_{u\sim\PU}\left[ J_{\theta^k}(u)\right]$
            & $\vartriangleleft$ Assign mean distance
            \end{tabular}}
        \vspace*{3pt}
        
        \STATE{\begin{tabular}{p{0.585\textwidth}r}
          \ \ $u^{k+1}_{i} \leftarrow  \gamma_k u^1_{i}   + (1-\gamma_k) g_k(u^k_{i})$, for all $i \in \sD$
            & 
            {$\vartriangleleft$ Update signals with (\ref{eq: gk-abusive}) }
            \end{tabular}}   
        
        \vspace{3pt}
        
    \RETURN $\{\mu, \beta_k, \gamma_k, \theta^k\}$
    \end{algorithmic}
    } 
    \end{algorithm}    
      
    The aim of Algorithm \ref{alg: training} is to determine a collection of function parameters $\{\theta^k\}$, step sizes $\{\beta_k,\gamma_k\}$, and relaxation parameters $\mu$.
    The relaxation parameters  $\mu$ in Step 1 determine whether the updates, on average, form under or over-relaxations (under if $(\mu_1+\mu_2) \in(0,1)$ and over if $(\mu_1+\mu_2) \in (1,2)$).
    The function parameterization $\sI$ in Step 2 defines the collection $\{J_{\theta}\}_{\theta\in\sI}$ of functions over which the optimization occurs in Step 6, which in practice forms an approximation of the set $\Gamma$ of all nonnegative  1-Lipschitz functions. 
    In our case, the collection of functions is parameterized by a 1-Lipschitz neural network, whose weights are $\theta \in \sI$.
    The anchoring sequence $\{\gamma_k\}$ in Step 3 is chosen to pull  successive updates closer to the initial iterate (\eg $\gamma_k = 1/k$), which is key to ensuring $\{u^k\}$ converges to the closest point to $u^1$ that is contained in $\sM$.
    The initial distribution $\bbP^1$ in Step 4 is, in practice, given by a collection of samples $\{u_i^1\}_{i\in\sD} \subset \sX$ where $\sD$ is an indexing set (n.b. in practice $\sD$ is an enumeration of a finite collection). The expectation is then approximated by an averaging sum over all the samples. Another collection of samples is used similarly for the true distribution $\PU$. Note, however, the true data samples are typically given whereas we are given some liberty in choosing the samples for $\bbP^1$ (described further in Section \ref{sec:numExperiments}).  
    A {\bf for} loop occurs in Lines 5-8 with each index $k$ corresponding to a distribution $\bbP^k$ of signal estimates of the form $u^k$ (see Remark \ref{remark: probability-push-forward}). In other words, we (informally) say that $\bbP^k$ consists of all signals $u_i^k$ for $i \in \sD$.
    The problem in Line 6 corresponds to (\ref{eq:distance-inclusion-sup}) and is used to obtain an estimate $J_{\theta^k}$ of the pointwise distance function $d_{\sM}$, which can be accomplished using a variant of stochastic gradient descent (SGD) \cite{robbins1951stochastic,bottou2010large} or ADAM~\cite{kingma2014adam}.    
    Line 7 then assigns the mean distance between points in the distribution of estimates $\bbP^k$ and the manifold $\sM$ to the variable $\beta_k$.
    Line 8 defines the updates for each iterate $u^k$ following the Halpern-type update described in (\ref{eq: adv-proj-update}). 
    In particular,  we use the definition 
    \begin{equation}
        g_k(u) \triangleq
        \begin{cases}
        \begin{array}{cl}
            u - (\mu_1 \beta_k + \mu_2 J_{\theta^k}(u)) \cdot \nabla J_{\theta^k}(u)
            &   \mbox{if $J_{\theta^k}$ is differentiable at $u$,}   \\
            u & \mbox{otherwise.}
        \end{array}
        \end{cases}
        \label{eq: gk-def}
    \end{equation}
    Assuming $d_{\sM} = J_{\theta^k}$ and setting
    \begin{equation}
        \lambda_k(u) \triangleq \mu_1 \beta_k + \mu_2 J_{\theta^k}(u),
        \label{eq:lambda-def}
    \end{equation}
    we obtain the relaxed projection
    \begin{equation}
        g_k(u) = u + \underbrace{\dfrac{\lambda_k(u)}{d_{\sM}(u)}}_{=:\alpha_k(u)} \left( P_{\sM}(u) - u\right)
        = u + \alpha_k(u) \left(P_{\sM}(u) - u\right)
        \in u - \lambda_k(u) \partial d_{\sM}(u),
        \label{eq: gk-relaxed-projection}
    \end{equation}
    where $\alpha_k(u)$ is defined to be the underbraced term and we adopt the convention of taking $\alpha_k(u) = 0$ when  $d_{\sM}(u) = 0$. (This is justified since $d_{\sM}(u) = 0$ implies $P_{\sM}(u) = u$.)
    Upon completion of  training, projections can be performed 
    by applying Algorithm \ref{alg: adversarial-projection}.

    \begin{remarkx}        
        In practice, because we perform numerical differentiation, we abusively write
        \begin{equation} \label{eq: gk-abusive}
            g_k(u) 
            = u - \lambda_k(u) \nabla J_{\theta^k}(u)
            = u - \big( \mu_1 \beta_k + \mu_2 J_{\theta^k}(u) \big) \nabla J_{\theta^k}(u),
        \end{equation}
        which is what is used in our experiments.
    \end{remarkx}

    \begin{algorithm}
    \caption{Wasserstein-based Projection (WP)     (Deployment of $P_{\sM}$ Approximation)}
    \label{alg: adversarial-projection}
    \begin{algorithmic}[1] 
        \def\TAB{\hspace*{10pt}}     
        
        \STATE{WP$(u)$: \hspace*{161pt} {$\vartriangleleft$ Provided signal}}         
                        
        \STATE{ \TAB Choose parameters $\{\mu ,\beta_k,\gamma_k,\theta^k\}$ 
        \hspace{29.5 pt} $\vartriangleleft$ Use result from Algorithm \ref{alg: training}} 
        
        \STATE{ \TAB $u^1 \leftarrow u$ \hspace*{154.5pt} {$\vartriangleleft$ Assume ${u}\sim \bbP^1$ and assign as initial iterate}} 
        
        \STATE{ \TAB {\bf for} $k=1,2,\ldots$ {\bf do}}
        \STATE{\TAB  \TAB 
            $u^{k+1} \leftarrow \gamma_k u^1  + (1-\gamma_k) g_k(u^k)$
            \hspace{35.5pt}  {$\vartriangleleft$ Halpern-type update (see  \eqref{eq: gk-abusive})}}
        \STATE{ \TAB {\bf return}  $u^{k+1}$
        \hspace{128.0pt}  {$\vartriangleleft$ Estimate of $P_{\sM}(u)$}}
    \end{algorithmic}
    \end{algorithm}  
    
    Once we have trained the parameters $\theta^k$, we can use them to approximate the projection $P_\sM$ using Algorithm~\ref{alg: adversarial-projection} as follows.
    First the parameters $\{\mu, \beta_k, \gamma_k, \theta^k\}$  are chosen according to Algorithm \ref{alg: training}.
    Then in Line 3 the point $u^1$ is initialized to the given estimate. A {\bf for} loop is formed in Lines 4-6 so that, for each $k$,   the Halpern-type update is computed using a relaxed projection with $g_k$ (Line 6). Since it is not explicit in the notation, we emphasize that $g_k$ is defined in terms of $\beta_k,$ $\gamma_k$, and $\theta^k$.
    Upon repeating this process the same number of times as the training iterations, we obtain our estimate $u^k$ in Line 7 of the projection of $u$ onto $\sM$.
    We emphasize Algorithm~\ref{alg: training} is performed \emph{once} in an offline process to obtain the projection using training data. Once trained, Algorithm~\ref{alg: adversarial-projection} can then be used in an online process for any signal that was not necessarily used during training.
    
      \begin{figure}[t]
        \centering
        \includegraphics[width= 4.25in]{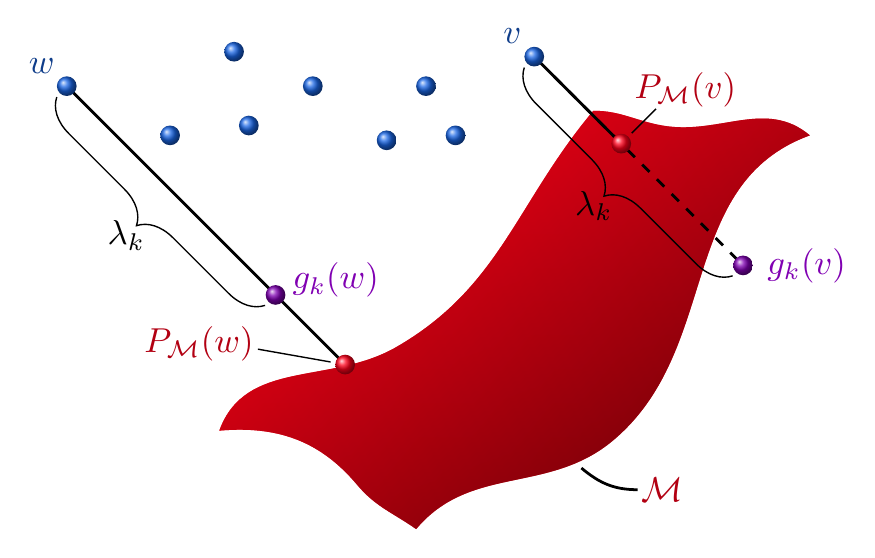}
        \caption{The \textcolor{distcolor!50!black}{blue points} are samples drawn from \textcolor{distcolor!50!black}{distribution $\bbP^k$}, which are updated to the \textcolor{gcolor!60!black}{purple relaxed projection $g_k$} of the  projection $P_{\sM}$ onto the visible portion of the \textcolor{projcolor!60!black}{red manifold $\sM$}. Here $g_k(v)$ and $g_k(w)$ are over and under-relaxations, respectively, and a common step size $\lambda_k$ is used.
        }
        \label{fig: manifold-projection}
    \end{figure}  

\section{Convergence Analysis} \label{sec:convergence-analysis} 
    This section formalizes the assumptions and states the main convergence result for the WP method (Algorithm \ref{alg: adversarial-projection}). 
    In similar fashion to \cite{lunz2019adversarial}, we first articulate one formalization of the idea that   true data is contained in a low dimensional manifold $\sM$.
    All proofs can be found in Appendix~\ref{sec:appendix-proofs}
    
    \begin{assumption} \label{ass: manifold-support}
        The support of the distribution of true signals $\PU$ is a convex, compact set $\mathcal{M} \subset \sX$, \ie $\mathrm{supp}(\PU) = \sM$.
    \end{assumption}
    
    \begin{remarkx} \label{remark: close-points}
        The convexity assumption may seem stringent since manifolds are not, in general, convex. However, from a practical perspective, we are interested in cases where a point is already ``close'' to the manifold $\sM$. This is because existing methods typically allow us to get ``close'' to the manifold (\eg use a variational method to get an estimate $\tilde{u}$ of a signal $u^\star$). Moreover, for such ``close'' points, the manifold looks like a Euclidean space and, thus, effectively appears to be a convex set from the perspective of individual noisy signals.
    \end{remarkx}
    
    The closeness idea in Remark \ref{remark: close-points} and   boundedness of $\sM$ yield the following assumption.
    
    \begin{assumption}
        \label{ass: bounded-D1}
        The support of the initial distribution $\bbP^1$ is bounded.
    \end{assumption}          
    
    \begin{remarkx} \label{remark: probability-push-forward}
        We formally define each distribution $\bbP^k$ of estimates by using a sequence of algorithmic operators $\{\sA_k : \sX\rightarrow\sX \}$. Set $\sA_1 (u) \triangleq u$ and inductively define
        \begin{equation}
            \sA_{k+1}(u)
            \triangleq \gamma_{k} u + (1-\gamma_{k}) g_{k}(\sA_{k}(u)),
            \ \ \ \mbox{for all $k\in\bbN$.}
        \end{equation}
        Then we see the $k$-th iteration $u^k$ in the WP algorithm (Algorithm \ref{alg: adversarial-projection}, Line 5), satisfies
        \begin{equation}
            u^{k} = \sA_{k}(u^1),
            \ \ \ \mbox{for all $k\in\bbN$.}
        \end{equation}
        We may view $\sA_k(u^1)$ as a random variable with sample $u^1 \sim \bbP^1$.
        To match notation with what is typical in algorithmic literature, rather than write $\{\sA_k(u^1)\}$, we refer to the sequence of random variables $\{u^k\}$ with the sample $u^1$ implicit.
        For each iteration, we also define the distribution $\bbP^{k}$ as the push forward of the algorithmic operator $\sA_k$ applied to $\bbP^1$, \ie
        \begin{equation}
            \bbP^{k} \triangleq (\sA_k)_\# \bbP^1, \ \ \ \mbox{for all $k\in\bbN$.}
        \end{equation}
        Recalling the definition of the push forward operation, we may equivalently write
        \begin{equation}
            \bbP^{k}[U] 
            = \bbP^1[ \sA_k^{-1}(U)]
            = \bbP^1\big[ \{ u^1 : \sA_k(u^1) \in U  \}  \big], 
            \ \ \ \mbox{for all $k\in\bbN$ and $U \subseteq \sX$.}
        \end{equation} 
    \end{remarkx}
    Our next assumption draws a connection between the distribution of signal estimates $\bbP^k$ and the distribution of true data $\PU$.
     This assumption effectively states the noise is not ``too large'' and the distribution $\bbP^k$ is sufficiently representative  (\ie the observed signals are not missing significant features from the true signals). This is weaker than assuming each individual signal can be recovered from its measurements. And, if our method is makes appropriate progress, truth of the assumption for $k=1$   naturally implies the truth for all subsequent values of $k$.
    
    \begin{assumption} \label{ass: projection-manifold-recovery}
        For all $k\in\bbN$, the distribution $\bbP^k$ is such that the push forward of the projection operation onto the manifold $\mathcal{M}$ recovers the true signal distribution $\PU$ up to a set of measure zero, \ie $\PU = (P_{\sM})_\# (\bbP^k)$.
    \end{assumption}

    With much credit to to \cite{lunz2019adversarial}, we extend their result \cite[Theorem 2]{lunz2019adversarial} to obtain the following theorem relating the set of nonnegative 1-Lipschitz functions to the   distance function $d_{\sM}$.
    \begin{theorem} \label{thm:distance-attains-sup}
        Under Assumptions \ref{ass: manifold-support} and \ref{ass: projection-manifold-recovery}, for all $k\in \bbN$, $\tau\in [0,\infty)$, and $p\in[1,\infty)$, the pointwise distance function $d_{\sM}$ is a solution to
        \begin{equation}
            \min_{f \in \Gamma} \;  \bbE_{u\sim \PU}\left[ f (u) + \tau f(u)^p \right] - \bbE_{u^k\sim \bbP^k} \left[f(u^k)\right],
            \label{eq:distance-sup-problem}
        \end{equation}
        where $\Gamma$ is the set of nonnegative 1-Lipschitz functions  mapping $\sX$ to $\bbR$. 
        Moreover, when $\tau > 0$, the restriction of each minimizer $f^\star$ of (\ref{eq:distance-sup-problem}) to the support of $\bbP^k$ is unique, \ie
        \begin{equation}
             f^\star(u^k) = d_{\sM}(u^k), \ \ \ \mbox{for all $u^k\in\mathrm{supp}(\bbP^k)$.} 
             \label{eq: distance-support-equality}
        \end{equation}
    \end{theorem}       
    
    Theorem \ref{thm:distance-attains-sup} is incredibly useful for our task since it provides a way to approximate the pointwise distance function $d_{\sM}$. In order to apply Theorem \ref{thm:distance-attains-sup}, we use the following assumption. 
   
    \begin{assumption} \label{ass: lip}
        The   parameter set   $\sI$ is such that the collection of functions
        $\{J_{\theta}\}_{\theta\in\sI}$ forms the set $\Gamma$ of nonnegative 1-Lipschitz functions mapping $\sX$ to $\bbR$, \ie $\{J_{{\theta}}\}_{\theta\in\sI} = \Gamma$.
    \end{assumption}   
    \begin{remarkx}
     Assumption \ref{ass: lip} can be approximately implemented by a few approaches. For example, one can choose standard 1-Lipschitz activations functions (\eg see \cite{anil2019sorting,combettes2020deep,combettes2020lipschitz,gao2017properties}). Linear mappings can be made 1-Lipschitz by spectral normalization   \cite{miyato2018spectral}, adding a gradient penalty to the loss function~\cite{gulrajani2017improved}, or   projecting   onto the set of orthonormal matrices~\cite{zaeemzadeh2020normpreservation}.
    \end{remarkx}
    
    Together the above assumptions and the following standard conditions on the anchoring sequence $\{\gamma_k\}$ allow us to state our main convergence result (n.b. we can choose $\gamma_k = 1/k$).
    
    \begin{assumption}
        \label{ass: gammak-props}
        The sequence $\{\gamma_k\}$ satisfies the following properties:         
            i) $\gamma_k \in (0,1]$ for all $k\in\bbN$, 
            ii) $\limk \gamma_k = 0$, and
            iii) $\sum_{k\in\bbN} \gamma_k = \infty.$         
    \end{assumption}
    
    \begin{theorem} \label{thm: main-result}
        $\mathrm{(Convergence\ of\ Wasserstein-based\ Projections)}$
        Suppose Assumptions \ref{ass: manifold-support}, \ref{ass: bounded-D1},  \ref{ass: projection-manifold-recovery}, \ref{ass: lip}, and \ref{ass: gammak-props} hold. 
        If the sequence $\{u^k\}$ is generated by Algorithm \ref{alg: adversarial-projection}, then the sequence $\{u^k\}$ converges  to $P_{\sM}(u^1)$ in mean square, and thus, in probability.
    \end{theorem} 
    By the definition of convergence in probability, this theorem implies, given $\varepsilon > 0$, the probability that the inequality $\|u^k - P_{\sM}(u^1)\| > \varepsilon$ holds goes to zero as $k\rightarrow\infty$. That is,
    \begin{equation}
        \limk \bbP^k\left[ \{  u^k  : \|u^k - P_{\sM}(u^1)\| > \varepsilon \} \right] = 0.
    \end{equation}
    This may be interpreted as saying that the probability that $u^k$ is not within distance $\varepsilon$ to $P_{\sM}(u^1)$ approaches zero as the iteration progresses (\ie as $k$ increases).

  \section{Related Works} \label{sec:relatedWorks}
     Here we present a brief overview of deep learning methods for inverse problems, Wasserstein GANs~\cite{arjovsky2017wasserstein,goodfellow2014generative} and their connections to optimal transport~\cite{lin2020apac,tanaka2019discriminator}, adversarial regularizers~\cite{lunz2019adversarial}, and expert regularizers~\cite{gilboa2013expert}. 
 
     \subsection{Deep Learning for Inverse Problems} \label{subsec:DeepLearningForInverseProblems}
    Our approach falls under the category of using deep learning to solve inverse problems~\cite{wang2016perspective}. 
    One approach, known as post-processing, first applies a pseudo-inverse operator to the measurement data (\eg FBP) and then learns a transformation in the image space. This approach has been investigated and found effective by several authors~\cite{jin2017deep,chen2017low,gulrajani2017improved,ronneberger2015u}.
    Another approach is to learn a regularizer, and then use it in a classical variational reconstruction scheme according to~\eqref{eq: variational-problem}.
    Other works investigate using dictionary learning~\cite{xu2012low}, variational auto-encoders~\cite{meinhardt2017learning}, and wavelet transforms~\cite{dokmanic2016inverse} for these learned regularizers.
    Perhaps the most popular schemes are learned iterative algorithms such as gradient descent~\cite{adler2017solving,kobler2017variational,hammernik2018learning}, proximal gradient descent or primal-dual algorithms~\cite{adler2018learned,sun2016deep}. These iterative schemes are typically unrolled, and an ``adaptive" iteration-dependent regularizer is learned. 
    One key difference between the approach of WPs and the aforementioned data-driven approaches is that our approach is \emph{unsupervised}. That is, we do not need a correspondence between the measurement $b$ and the true underlying signal $u^\star$.
    Generating approximate WPs simply requires a batch of true signals and a batch of measurements, regardless of whether these directly correspond to each other (\ie an injective map between the two might not be available); this is especially useful in some applications (\eg medical imaging) where the true image corresponding to the measurement is often not available. 
    Another set of work uses deep image priors (DIP)~\cite{ulyanov2018deep,baguer2020computed}, which attempt to parameterize the \emph{signal} by a neural network. The weights are optimized by a gradient descent method that minimizes the data discrepancy of the output of the network. The authors in~\cite{baguer2020computed} show that combining DIPs with classical regularization techniques are effective in limited-data regimes.
    
    \subsection{Wasserstein GANs and Optimal Transport} \label{subsec:WGANandOT}
    Our work bears connections with GANs \cite{goodfellow2014generative,arjovsky2017wasserstein}, and their applications to inverse problems~\cite{shah2018solving}.
    In GANs \cite{goodfellow2014generative,arjovsky2017wasserstein}, access is given to a discriminator and generator, and the goal is to train the generator to produce samples from a desired distribution. 
    The generator does this by taking samples from a known distribution $\mathcal{N}$ and transforming them into samples from the desired distribution $\PU$. 
    Meanwhile, the purpose of the discriminator is to guide the optimization of the generator. 
    Given a generator network $G_\theta$ and a discriminator network $D_\omega$, the goal in Wasserstein GANs is to find a saddle point solution to the minimax problem
        \begin{equation}
            \inf_{G_\theta} \sup_{D_\omega}\;  \mathbb{E}_{u\sim \PU} \left[D_\omega(u)\right] - \mathbb{E}_{z\sim \mathcal{N}} \left[ D_\omega(G_\theta(z))\right], \quad \text{s.t. }\quad  \|\nabla D_\omega\| \le 1,
            \label{eq:WGAN-problem}
        \end{equation}
    Here, the discriminator attempts to distinguish real images from fake/generated images, and the   generator aims to produce samples that ``fool" the discriminator by appearing real.
    The supremum expression in~\eqref{eq:WGAN-problem} is the Kantorovich-Rubenstein dual formulation~\cite{villani2008optimal} of the Wasserstein-1 distance, and the discriminator is required to be 1-Lipschitz. 
    Thus, the   discriminator   computes the Wasserstein-1 distance between the true distribution $\sD_{\rm true}$ and the fake image distribution  generated by $G_\theta(z)$.    
    Common methods to enforce the Lipschitz condition on the discriminator include weight-clipping \cite{arjovsky2017wasserstein} and gradient penalties in the loss function \cite{gulrajani2017improved}.\newparagraph
    
    Our approach can be viewed as training a special case of Wasserstein GANs, except that rather than solving a minimax problem, we solve a sequence of minimization problems. 
    In this case, $J$ is the discriminator network that distinguishes between signals coming from the ``fake" distribution (\ie our approximate distribution) and the true distribution, and $g_\eta$ is the generator which tries to generate signals that resemble those from the true distribution.\newparagraph

    Under certain assumptions (see Section~\ref{sec:convergence-analysis}),  Wasserstein-based projections can be interpreted as a subgradient flow that minimizes the Wasserstein-1 distance, where the function $J$ corresponds to the Kantorovich potential~\cite{arjovsky2017wasserstein,lin2020apac,tanaka2019discriminator,onken2020ot,lin2019fluid}, or in the context of mean field games and optimal control, the value function~\cite{ruthotto2020machine,lin2020apac}.
    Analogous to classical physics, the signals flow in a manner that minimize their potential energy. Our approach learns a sequence of these potential functions that project (or ``flow") the distribution of estimates toward the true distribution.

    \subsection{Adversarial Regularizers}
    \label{subsec:adversarialRegularizers}
    Our work is closely related to adversarial regularizers~\cite{lunz2019adversarial}.
    A good regularizer $J \colon \sX \to \overline{\bbR}$ is able to distinguish between signals drawn from the true distribution $\PU$ and drawn from an approximate distribution $\widetilde{\bbP}$ -- taking low values on signals from $\PU$ and high values otherwise~\cite{benning2017learning}. 
    Such  a regularizer plays a similar role as the discriminator described in Section~\ref{subsec:WGANandOT}; however, this setting is different in that $D_\omega$ assigns high values to true signals instead. Mathematically, $J = -D_\omega$.
    These regularizers are called adversarial regularizers~\cite{lunz2019adversarial}. They are trained a priori in a GAN-like fashion and then used to solve a classical inverse problem via the variational model \eqref{eq: variational-problem} (see Algorithms 1 and 2 in~\cite{lunz2019adversarial}).
    The adversarial regularizers act quite similarly to expert regularizers \cite{gilboa2013expert}, which attain small values at signals similar to the distribution of true signals and larger values at signals drawn elsewhere.  
    The key modeling difference between Wasserstein-based projections and adversarial regularizers is in how the manifold is used to construct a variational model.
    The latter essentially uses the distance function $d_{\sM}$ as a regularizer while the former uses the indicator function $\delta_{\sM}$ (see \eqref{eq: VP-manifold}). 
    Using the distance function as a regularizer can encourage nice behavior, but  requires choosing a weighting parameter and this approach allows noise to bias the reconstructed signal so that it is not necessarily on the manifold $\sM$.

    \subsection{Manifolds and Dimension Reduction}  
    The current era of big data has given rise to many problems that suffer from the curse of dimensionality \cite{donoho2000high}.
    In order to translate the high dimensional signals found in practice into interpretable visualizations, dimensionality reduction techniques have been introduced (\eg PCA~\cite{pearson1901liii}, Isomap \cite{tenenbaum2000global}, Laplacian eigenmaps \cite{belkin2003laplacian}, and  t-SNE \cite{maaten2008visualizing}). We refer the reader to \cite{maaten2009dimensionality,talwalkar2008large,lin2008riemannian,yan2006graph,cayton2005algorithms} for summaries and further sources on dimension reduction and manifold learning. Beyond visualization, some efforts seek to exploit low dimensional representations to better solve inverse problems. For example, a related work \cite{osher2017low} introduced a patch-based low dimensional manifold model (LDMM) for image processing. This built upon previous patch-based image processing works \cite{carlsson2008local,lee2003nonlinear,peyre2008image,peyre2009manifold}.
    Perhaps, the closest work to ours is \cite{rick2017one}, which attempts to learn a projection as well. However, that work used a different training loss function and provides limited theoretical analysis of the performance of their convolutional autoencoder approach.
        
\section{Numerical Experiments} \label{sec:numExperiments} 
    In this section, we demonstrate the potential of Wasserstein-based projections. 
    We begin with a toy example in 2D to provide intuition for the training and online processes.  
    We then test our approach on computed tomography (CT) image reconstruction problems  using two standard datasets: a synthetic dataset comprised of randomly generated ellipses and the Low-Dose Parallel Beam (LoDoPaB) dataset~\cite{leuschner2019lodopab}.
    As mentioned previously, all experiments aim to solve the problem (\ref{eq: experiments-problem}) using a sequence $\{z^t\}$ generated via (\ref{eq: relaxed-projected-gradient-update}).
    We approximate each projection in (\ref{eq: relaxed-projected-gradient-update}) using 20 steps of Algorithm \ref{alg: adversarial-projection}.
    
  \subsection{Deployment of the Projection}
    Once trained, the approximate WP operator can be incorporated into optimization algorithms in the same manner as Plug-and-Play methods. 
    To clarify this via illustration, we preview here how we apply the projection in our experiments. For the experiments, we solve the   special case of (\ref{eq: VP-manifold}) with least squares fidelity term, \ie 
    \begin{equation}
        \min_{z\in \sM} \dfrac{1}{2}\|Az-d\|_2^2 
        =\min_{z \in \bbR^n} \dfrac{1}{2}\|Az-d\|_2^2 + \delta_{\sM}(z),
        \label{eq: experiments-problem}
    \end{equation}
    where $A \in \bbR^{m\times n}$, and $b\in \bbR^m$.
    In our experiments,   (\ref{eq: experiments-problem}) is solved with a relaxed form of projected gradient. This consists of generating a sequence $\{z^t\}$ with updates of the form
    \begin{equation}
        z^{t+1} = (1- \kappa ) z^t + \kappa \cdot P_{\sM}(z^t- \xi A^T (Az^t-d)) ,
        \label{eq: relaxed-projected-gradient-update}
    \end{equation}    
    where  $\kappa \in (0,1)$ and $\xi \in (0, 2 / \|A^TA\|_2)$.     
    Application of the projection operator $P_{\sM}$ is required in the update (\ref{eq: relaxed-projected-gradient-update}), but in many applications we do not have an explicit expression for this.
    Thus, at each iteration $t$, we use Algorithm \ref{alg: adversarial-projection} to approximate projections, \ie 
    \begin{equation}
        z^{t+1} = (1- \kappa ) z^t + \kappa \cdot \mathrm{WP}\left(z^t- \xi A^T (Az^t-d)\right) ,
        \label{eq: relaxed-projected-gradient-update-WP}
    \end{equation}       
    where $\mathrm{WP}(z)$ is the output from Algorithm \ref{alg: adversarial-projection}.
    
    \begin{remarkx}
        \label{remark: PM_Approximation}
        In this subsection, we use $z$ and $z^t$ to denote signals.
        This notation is used to avoid confusion between the sequence $\{z^t\}$ in (\ref{eq: relaxed-projected-gradient-update}) and the sequence $\{u^k\}$ in Algorithm \ref{alg: adversarial-projection}. 
        The connection between these is that, setting $u^1 =  z^t - \alpha A^T (Az^t-d)$,  Algorithm \ref{alg: adversarial-projection}  computes the projection operation in (\ref{eq: relaxed-projected-gradient-update}), \ie
        \begin{equation}
           P_{\sM}(z^t-\xi A^T (Az^t-d)) = \limk u^k.
           \label{eq: projection-limit-approximation}
        \end{equation} 
        In our experiments we approximate the above limit using 20 iterations (\ie $u^{20}$).
    \end{remarkx}
    
    \begin{remarkx}
        Although for practical reasons we consider linear inverse problems in our experiments, we emphasize that our presented methodology applies even when $u^\star$ is recovered from {\it nonlinear} measurements (\ie when $A$ is a nonlinear operator). 
    \end{remarkx}

    \fboxsep=0.1in
    \fboxrule=0mm

    \subsection{Step Size Illustration}
  \begin{remarkx} \label{remark: step-size}
        The peculiar choice of step size in  (\ref{eq: gk-abusive}) requires explanation.
        As mentioned, an ideal setting would use step size $J_{\theta^k}(u^k)$. However, the problem in \eqref{eq:distance-inclusion-sup} often cannot be solved exactly.
        Even knowing that $J_{\theta^k}$ should equal zero on the manifold, we cannot simply perturb $J_{\theta^k}$ by adding a constant so that the average value on the manifold is zero. This would cause (with high probability) there to be points at which the step size $J_{\theta^k}$ would evaluate to negative values. And, using a negative step size would yield gradient \textit{ascent} and potential divergence.
        On the other hand, our choice of $\beta_k$ mitigates this offsetting issue by using the difference of the average values of $J_{\theta^k}$ for each distribution ($\bbP^k$ and $\PU$).
        Moreover, because the $\beta_k$ term provides an average step size contribution common to all signals in the distribution $\bbP^k$, it yields a more uniform flow of signals that is insensitive to errors in our approximation $J_{\theta^k}$ of the distance $d_{\sM}$.
        A drawback of using $\beta_k$ is that, since the step sizes are diminishing and dependent on the average distance, some ``stragglers'' (\ie signals that are left behind from the majority of the distribution) take a long time to reach the manifold when only using $\beta_k$. 
        We illustrate the underlying phenomena in Figure \ref{fig: stragglers}. 
        Figure \ref{fig: stragglers}a shows two initial distributions of signals, generated samples in blue and red signals in the true manifold.
        Figures \ref{fig: stragglers}b, \ref{fig: stragglers}c, and \ref{fig: stragglers}d show the straggler phenomenon, which is reduced by including a contribution of $J_{\theta^k}$ in the step size $\lambda_k$ (\ie $\mu_2 > 0$).
        In summary, there is a balance to be played in practice for how much to weigh each term to obtain the best results for a particular application, depending on how well can be $d_{\sM}$ approximated.
    \end{remarkx}

    \begin{figure}[t]
        \centering
        \subfloat[Original]{\includegraphics[width=1.53in]{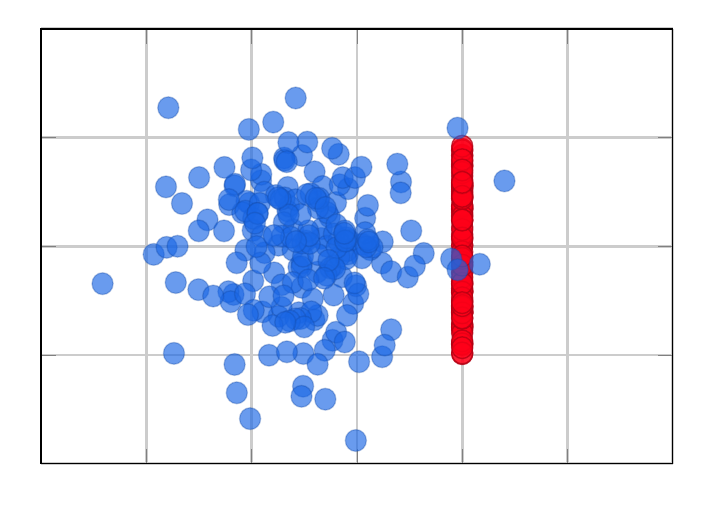}}        
        \subfloat[$\lambda_k(u) = \frac{1}{2}  \beta_k$]{\includegraphics[width=1.53in]{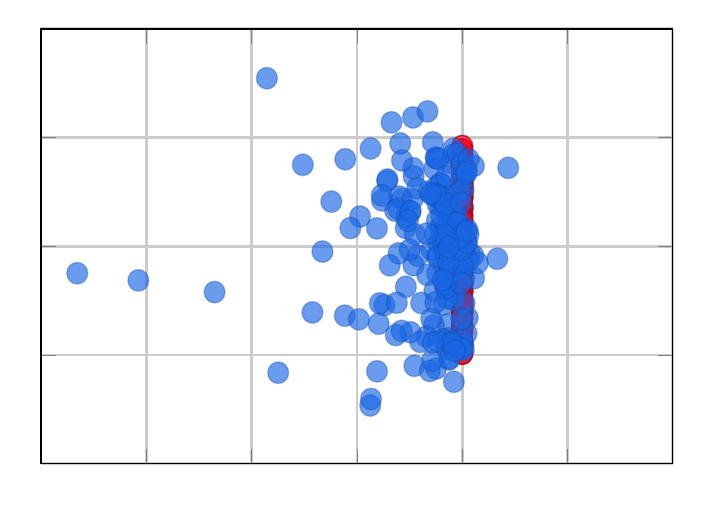}}
        \subfloat[$\lambda_k(u) = \frac{1}{4}  \left( \beta_k + J_{\theta^k}(u)\right)$]{\includegraphics[width=1.53in]{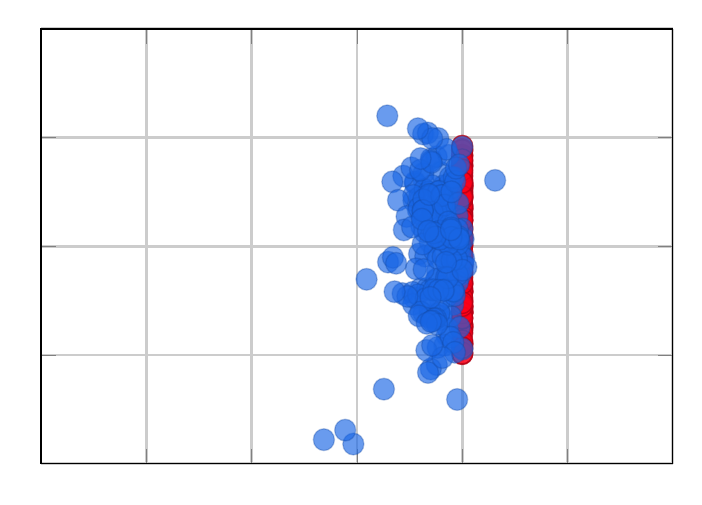}}
        \subfloat[$\lambda_k(u) = \frac{1}{2}  J_{\theta^k}(u)$]{\includegraphics[width=1.53in]{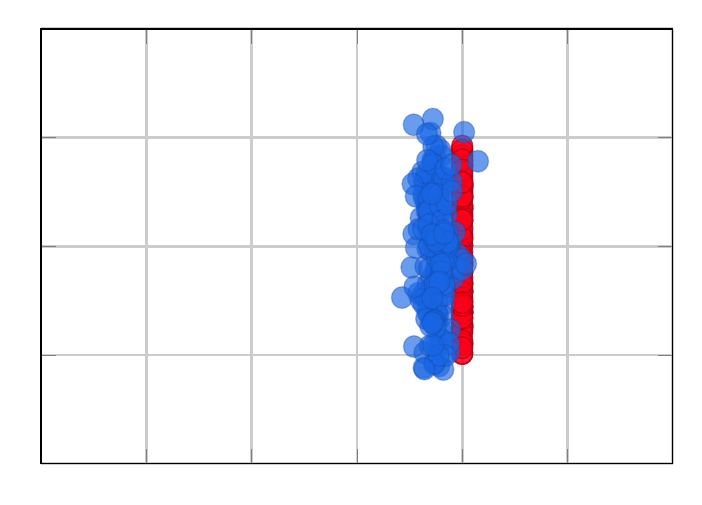}} 
        
        \caption{Illustration of the effects of different step size rules. A distribution $\bbP^1$ of blue generated samples and red manifold samples $\sM$ are shown in (a). The remaining figures show the distribution $\bbP^{12}$, after 11 updates to $\bbP^1$, for different step size $\lambda_k(u)$ rules.
        The code for generating this data is in this \href{https://colab.research.google.com/drive/1hhMmAr1MuBm9LOe29v8-cE88UdawUeRw?usp=sharing}{notebook}.}
        \label{fig: stragglers}
    \end{figure}    
    
    \begin{figure} 
        \centering
        \subfloat[$\bbP^1$ (blue) and $\PU$ (red)]{\includegraphics[height=1.95in]{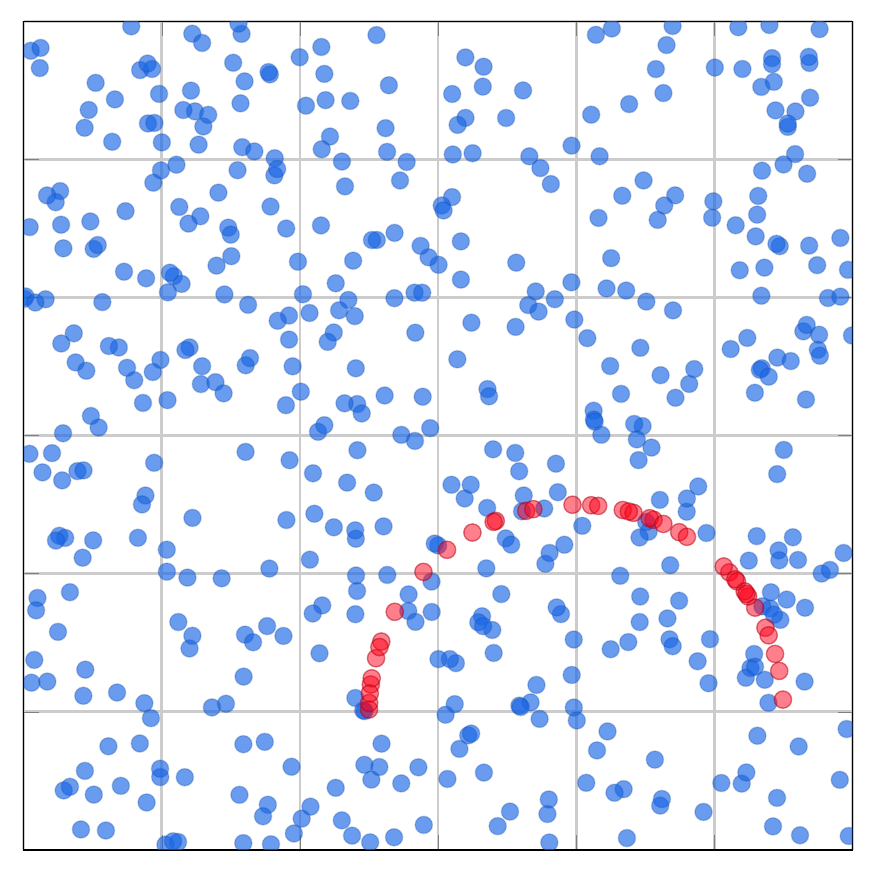}}
        %
        \subfloat[Convergence Trajectories]{\includegraphics[height=1.95in]{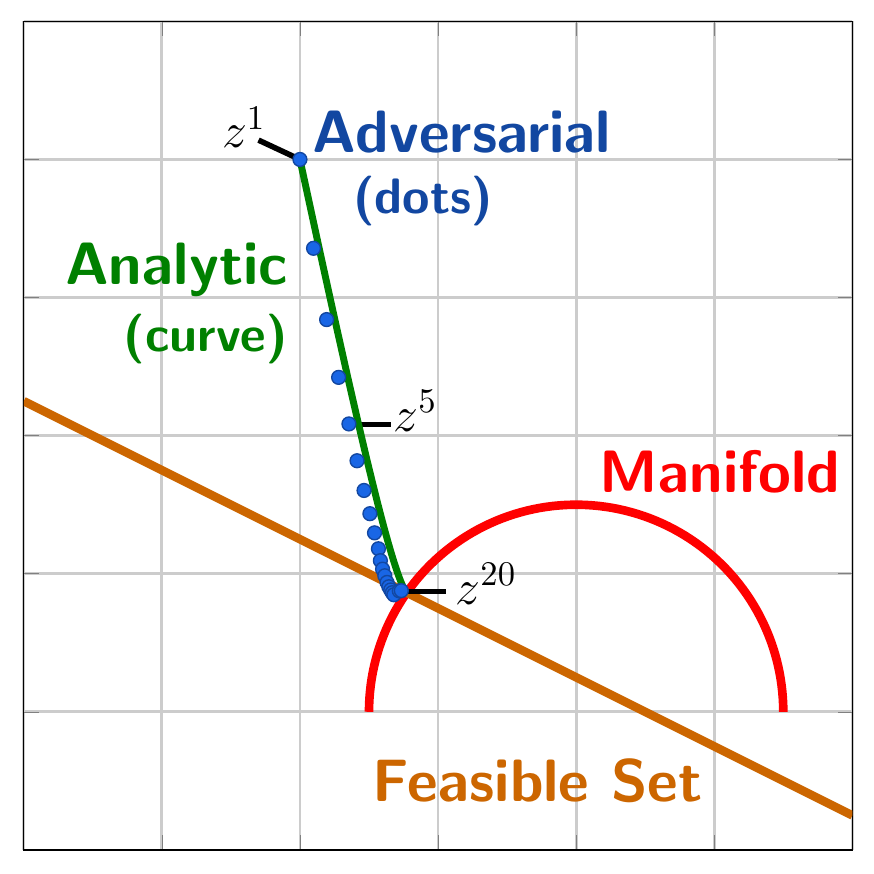}
        }
        %
        \subfloat[Landscape Plot of $J_{\theta^1} \approx d_{\sM}$ ]{
        \begin{minipage}{1.9in}
        \vspace{-1.885in}
        \includegraphics[height=1.86in]{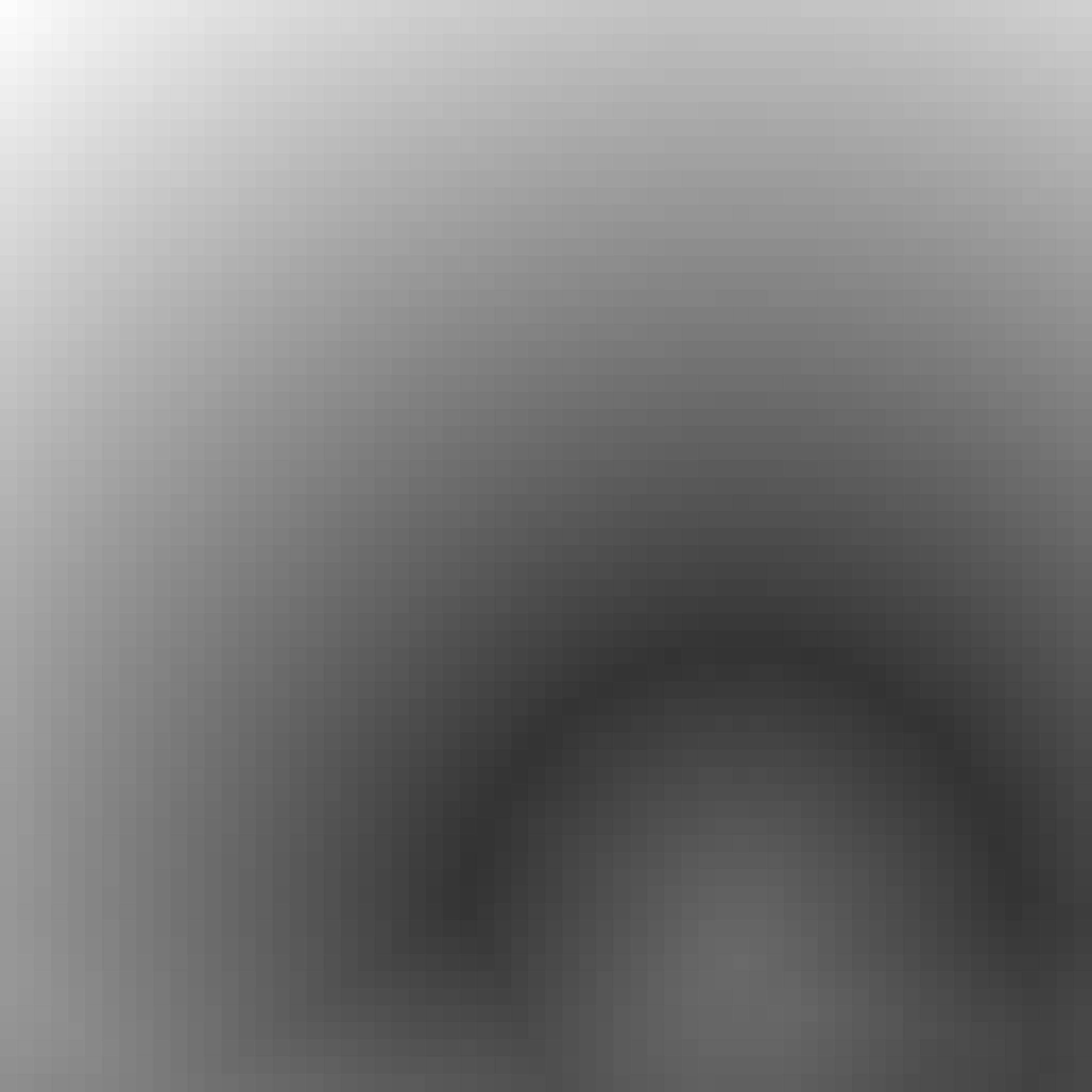} 
        \end{minipage} 
        } 
        \caption{2D Toy Problem training setup and inference example. 
        a) Training (Algorithm \ref{alg: training}) uses a uniform distribution $\bbP^1$ and manifold sampling distribution $\PU$.
        b) Trajectories (\ref{eq: relaxed-projected-gradient-update}) for solving (\ref{eq: experiments-problem}) are shown with analytic projections (green) and Wasserstein-based projections (blue). 
        Wasserstein-based projections are implemented using Algorithm \ref{alg: adversarial-projection}, only utilizing knowledge of $\PU$. The green curve uses an analytic expression for each projection. (The analytic method yields discrete points, but these are joined via a curve for illustrative purposes.)
        The feasible set is all $z$ such that $Az = d$.
        c) Estimate of the distance function $d_{\sM}$ from training. 
        All plots are over the region ${[ 0.0,3.0] \times [-0.5,2.5]}$. }
        \label{fig: toy-illustration-setup}
    \end{figure}
        
    \subsection{Toy Illustration}   
    We compare the performance of our Wasserstein-based method to an analytic method in a 2D example.\footnote{All codes for this example can be run online \href{https://colab.research.google.com/drive/1tO8T5E_Jycke9qV0s3uPsNqYybeDO0ue?usp=sharing}{here} via Jupyter notebooks in Google Colab.} 
    The fundamental difference between the methods is that the Wasserstein-based method only has has access to a sparse sampling of a manifold while the analytic method has complete knowledge of the manifold.
    Here the manifold $\sM$ is the upper half circle of radius 0.75 centered at $(2,0)$. A sparse sampling of 50 points from $\sM$ forms our estimate of $\PU$. We also sample a uniform distribution  on the rectangle $[0.0,3.0]\times[-0.5,1.5]$ to form $\bbP^1$. These distributions may be visualized in Figure \ref{fig: toy-illustration-setup}. 
    The estimate $J_{\theta^1}$ of the distance function $d_{\sM}$ is plotted in Figure \ref{fig: toy-illustration-setup}c. The neural network takes $u\in\bbR^2$ as input and outputs a scalar $J(u)$.
    Its structure uses GroupSort activation functions and orthonormal weight matrices in similar fashion to the proposed networks in \cite{anil2019sorting}, which possess the property of being universal 1-Lipschitz function approximators (as the number of parameters/layers increase). We used 6 hidden layers with $10\times 10$ weight matrices. 
    
    \begin{remarkx}
        This example uses a uniform distribution over a region that contains the true manifold. In practice, this cannot be typically done. Instead, we are often given a collection of samples from a noisy distribution.
        Additionally, during training in all of our numerical examples, we used signals that are drawn from $\bbP^k$ and perturbed by small amounts of random Gaussian noise to ``smooth'' $\bbP^k$, thereby improving generalization.
    \end{remarkx}
    
    Figure \ref{fig: toy-illustration-setup}b illustrates trajectories $\{z^t\}_{t=1}^{20}$ for solving the problem in \eqref{eq: experiments-problem} using \eqref{eq: relaxed-projected-gradient-update}, where $A = [1\  2]$, $b = 2$, and $z^1 = (0.5,1)$. The green curve uses an analytic formula for the projection $P_{\sM}$ while the blue dots use the Wasserstein-based projections algorithm (Algorithm \ref{alg: adversarial-projection}) to approximate $P_{\sM}$ from the provided manifold samples (see Figure \ref{fig: toy-illustration-setup}a). The two methods deviate from each other during an early portion of the trajectories, but ultimately both converge to the same limit.
    This shows the proposed algorithm was able to leverage sparse samples of the manifold $\sM$ to obtain a decent estimate of the projection operation $P_\sM$. 
    
     \subsection{Low-Dose Computed Tomography}
    We now perform Wasserstein-based projections on two low-dose CT examples.
    We focus on the unsupervised learning setting, where we do not have a correspondence between the distribution of true signals and approximate signals.
    Therefore, we set adversarial regularizers (a state-of-the-art unsupervised learning approach) as our benchmark for the CT problems.
    The quality of the image reconstructions are determined using the Peak Signal-To-Noise Ratio (PSNR) and structural similarity index measure (SSIM).
    As stated in Remark~\ref{remark: PM_Approximation} we use 20 iterations in Algorithm~\ref{alg: training} to approximate the projection operator for all experiments. We use the PyTorch deep learning framework~\cite{paszke2019pytorch} and the ADAM~\cite{kingma2014adam} optimizer.
    We also use the Operator Discretization Library (ODL) python library~\cite{jonas_adler_2017_249479} to compute the TV and filtered backprojection (FBP) solutions.
    The CT experiments are run on a single NVIDIA TITAN X GPU with 12GB RAM.
    \newparagraph
        
    \paragraph{Ellipse Phantoms}
    We use a synthetic dataset consisting of random phantoms of combined ellipses as in~\cite{adler2017solving}.
    The images have a resolution of $128 \times 128$ pixels. Measurements are simulated with a parallel beam geometry with a sparse-angle setup of only $30$ angles and $183$ projection beams. Moreover, we add Gaussian noise with a standard deviation of $2.5\%$ of the mean absolute value of the projection data to the projection data. In total, the training set contains 10,000 pairs, while the validation and test set consist of 1,000 pairs each.
    \newparagraph
    
    \begin{table}[t]
        \centering
        \caption*{CT Results on Ellipses Dataset}
        \begin{tabular}{@{}ccc@{}}
        \toprule
        Method & Avg. PSNR (dB) & Avg. SSIM
        \\
        \midrule
        Filtered Backprojection & 16.53 & 0.179
        \\
        Total Variation & 26.46 & 0.625 
        \\
        Adversarial Regularizers & 26.95 & 0.680
        \\
        Wasserstein-based Projections (ours) & {\bf 28.09} & {\bf 0.764}
        \\
        \vspace{.0mm}
        \end{tabular}
        \caption{Average PSNR and SSIM on a validation dataset with 1,000 images of random ellipses.}
        \label{tab:ellipse_results}
    \end{table}
    
    \begin{figure}
      \centering
      \small
        \setlength{\tabcolsep}{0.1pt}
        \begin{tabular}{ccccc}
            ground truth & FBP & TV & Adv. Reg. & Adv. Proj.
            \\

        \begin{tikzpicture} [spy using outlines={rectangle, magnification=3, size=1cm, connect spies}, rounded corners]
                \node[anchor=south west,inner sep=0] (image) at (0,0) {\adjincludegraphics[width=0.185\textwidth]{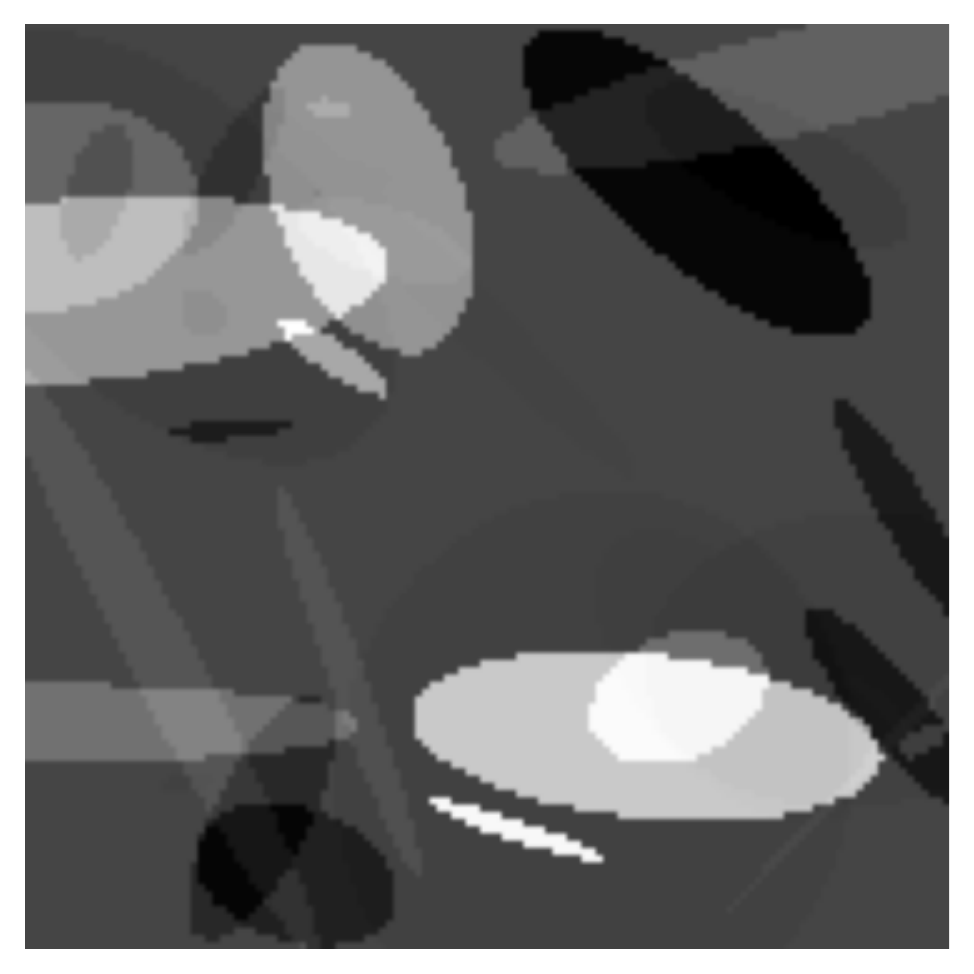}}; 
                \begin{scope}
                \spy[\spycolor,size=0.175\textwidth, every spy on node/.append style={line width = \W}] on (1.0,1.7) in node at (1.44, -1.5);  
                \end{scope}
        \end{tikzpicture}
        &
        \begin{tikzpicture} [spy using outlines={rectangle, magnification=3, size=1cm, connect spies}, rounded corners]
                \node[anchor=south west,inner sep=0] (image) at (0,0) {\adjincludegraphics[width=0.185\textwidth]{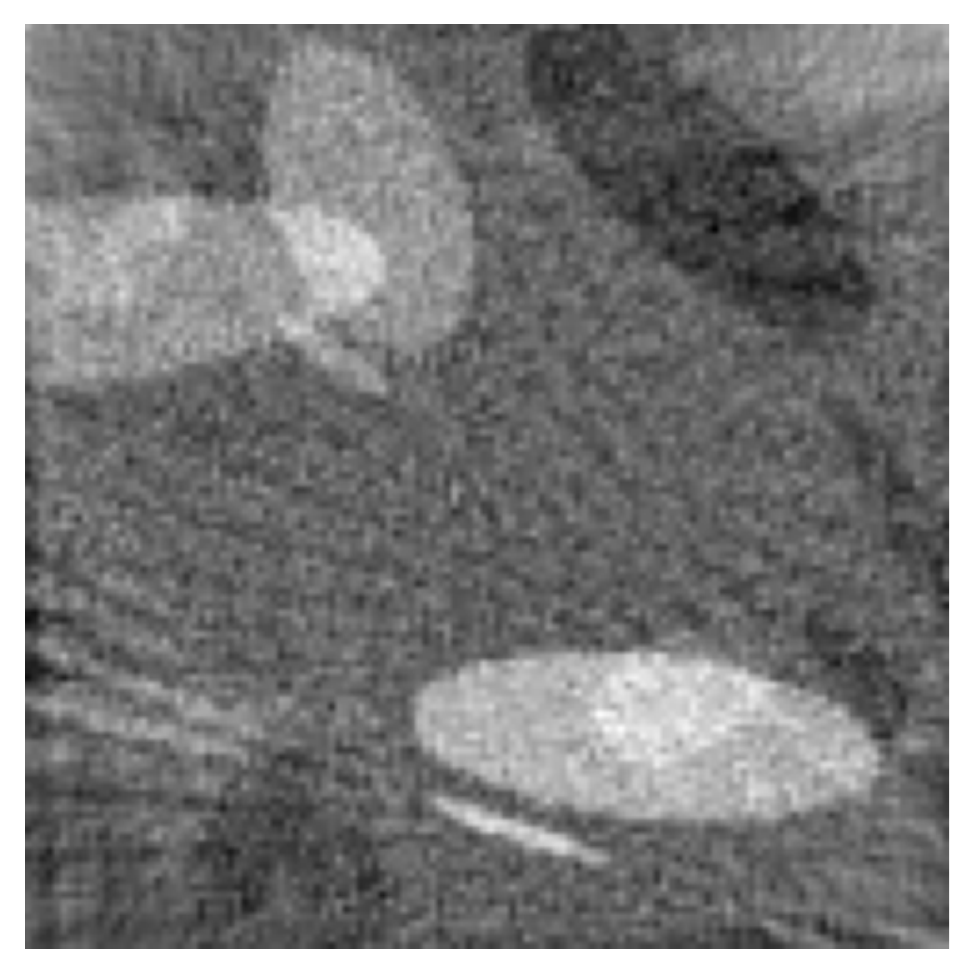}};
                \begin{scope}
                \spy[\spycolor,size=0.175\textwidth, every spy on node/.append style={line width = \W}] on (1.0,1.7) in node at (1.44, -1.5);  
                \end{scope}
        \end{tikzpicture}       
        &
        \begin{tikzpicture} [spy using outlines={rectangle, magnification=3, size=1cm, connect spies}, rounded corners]
                \node[anchor=south west,inner sep=0] (image) at (0,0) {\adjincludegraphics[width=0.185\textwidth]{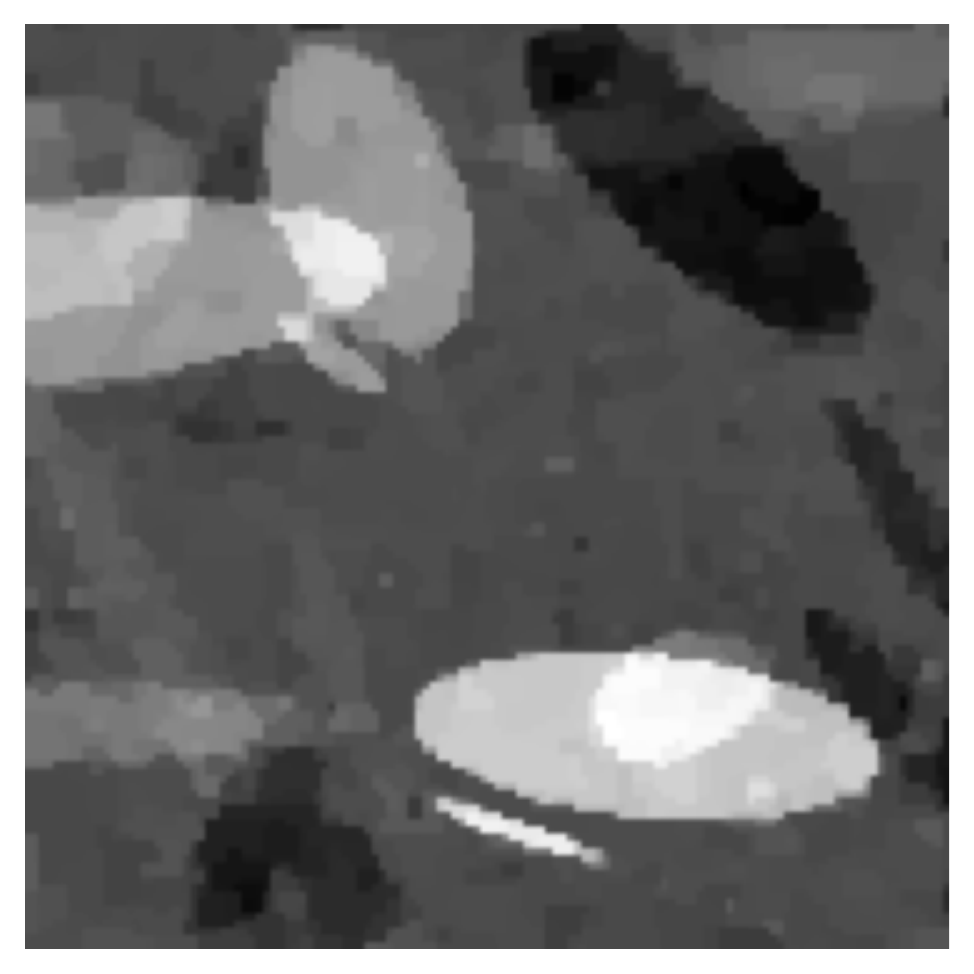}};
                \begin{scope}
                \spy[\spycolor,size=0.175\textwidth, every spy on node/.append style={line width = \W}] on (1.0,1.7) in node at (1.44, -1.5);  
                \end{scope}
        \end{tikzpicture}   
        &
        \begin{tikzpicture} [spy using outlines={rectangle, magnification=3, size=1cm, connect spies}, rounded corners]
                \node[anchor=south west,inner sep=0] (image) at (0,0) {\adjincludegraphics[width=0.185\textwidth]{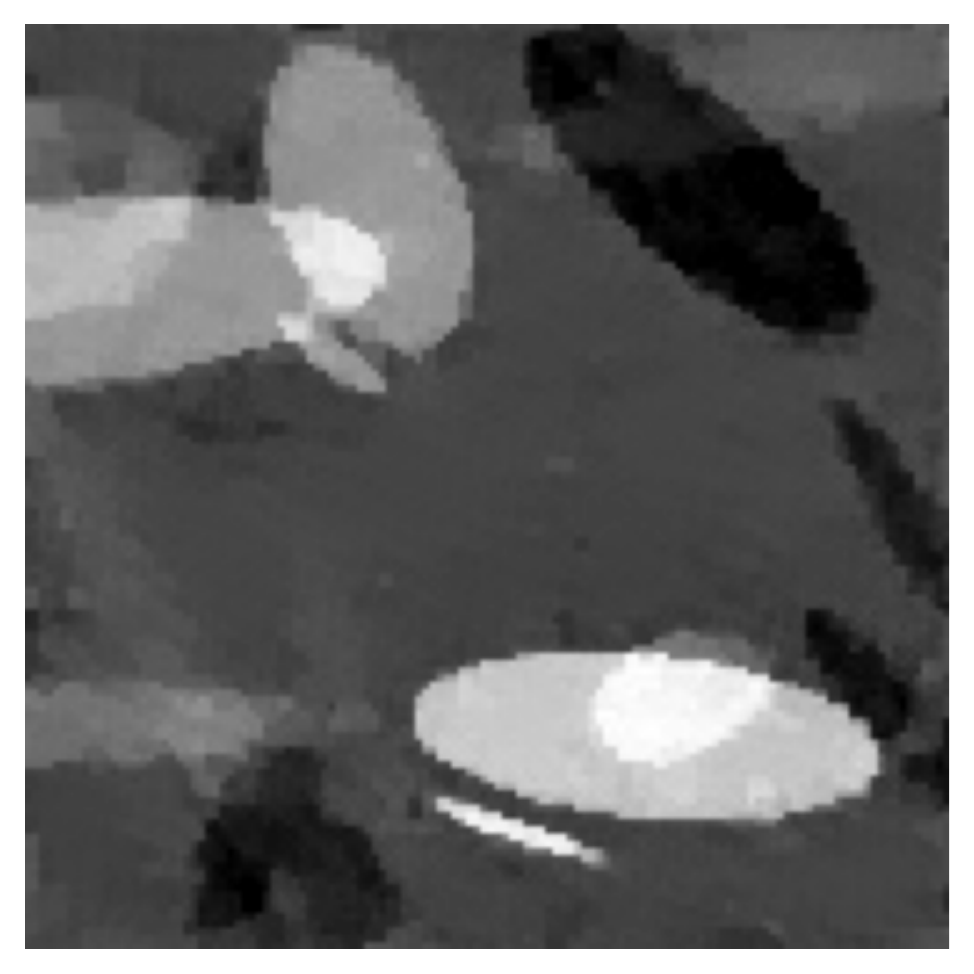}};
                \begin{scope}
                \spy[\spycolor,size=0.175\textwidth, every spy on node/.append style={line width = \W}] on (1.0,1.7) in node at (1.44, -1.5);  
                \end{scope}
        \end{tikzpicture}       
        &
        \begin{tikzpicture} [spy using outlines={rectangle, magnification=3, size=1cm, connect spies}, rounded corners]
                \node[anchor=south west,inner sep=0] (image) at (0,0) {\adjincludegraphics[width=0.185\textwidth]{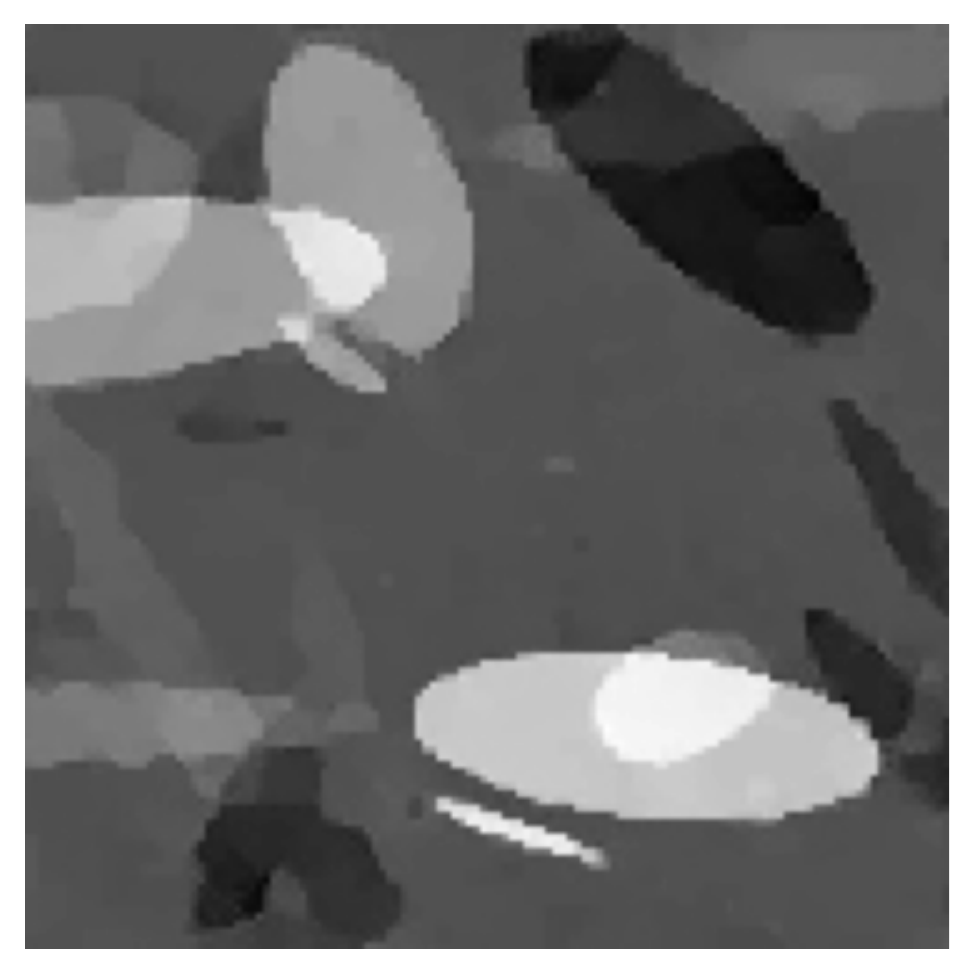}};
                \begin{scope}
                \spy[\spycolor,size=0.175\textwidth, every spy on node/.append style={line width = \W}] on (1.0,1.7) in node at (1.44, -1.5);  
                \end{scope}
        \end{tikzpicture}   
        \\
        & SSIM: 0.301 & SSIM: 0.812 & SSIM: 0.822 & SSIM: 0.870
        \\
        & PSNR: 18.53 & PSNR: 28.26 & PSNR: 28.30 & PSNR: 29.38
        \end{tabular}
        \caption{Reconstruction on a validation sample obtained with Filtered Back Projection (FBP) method, TV regularization, Adversarial Regularizer, and Wasserstein-based Projections (left to right). Bottom row shows expanded version of corresponding cropped region indicated by red box.}
        \label{fig:ellipses_reconstructions_3}
    \end{figure}

    \paragraph{Human Phantoms}
    As a more realistic dataset, we use human phantoms consisting of chest CT scans from the Low-Dose Parallel Beam dataset (LoDoPaB)~\cite{leuschner2019lodopab}. 
    In our setup, we use 20,000 training images and 2,000 validation images of size $128 \times 128$.
    Similar to the ellipse phantoms, we simulate the data using only 30 angles and 183 projection beams. 
    As a more realistic setting, we add $1.5\%$ Gaussian noise corresponding to the individual beams.
    Note this is different from adding the noise corresponding to the mean of all 183 beams for a particular angle as was done in the ellipse dataset.
    \newparagraph
    
    \paragraph{Network Structure}
    We use a simple 5 layer neural network containing 38,534 trainable parameters. The first three being convolution layers with kernel size $4$ and stride $2$, with output channels $32, 64,$ and $1$ for layers one, two, and three, respectively. 
    For the last two layers, we use fully connected layers to bring the dimensions back to a scalar.
    As nonlinear activation function,we choose the Parametric Rectified Linear Units (PReLU) functions
    \begin{equation*}
		\sigma_c(x) = 
		\begin{cases}
		x & \text{if } x \geq 0 \\
		- c x & \text{else}
		\end{cases}
    \end{equation*}
    between layers, which was shown to be effective in other applications such as classification~\cite{he2015delving}. The last activation function is chosen to be the Huber function to ensure positivity of $J$.
    
    \begin{remarkx}
        In our experience, training was more effective when using the Huber function instead of the absolute value function on the CT problems.
        Since we have a limited amount of high-dimensional manifold data, we suspect that a smoother landscape (with respect to the network weights) yields better generalization for this particular setup~\cite{chaudhari2019entropy,keskar2016large}.
    \end{remarkx}

    For adversarial regularizers, we use the network structure described in~\cite{lunz2019adversarial}, which consists of an 8-layer CNN with Leaky-Relu activation. The network contains 2,495,201 parameters. More details can be found in~\cite[Appendix B]{lunz2019adversarial}.\\

    \paragraph{Wasserstein-based Projection Training Setup}
    To train the Wasserstein-based projections, we begin with an initial distribution obtained from the TV reconstructions (\ie $\bbP^1$ is set to be the distribution of TV images). Our approach can therefore be considered as a post-processing scheme, depending on the mapping from measurement space to signal space.
    We update the distribution whenever 200 epochs have passed since the last update for both ellipses and human phantoms datasets. 
    In our setting, we choose the stepsizes only according to the mean distance by setting the relaxation parameter in line 1 of Algorithm~\ref{alg: training} as $\mu = (0.5, 0)$.
    We use the ADAM optimizer with a learning rate of $10^{-5}$ and batch size of 16 samples.
    To ensure Assumption~\ref{ass: gammak-props} is satisfied, we choose $\gamma_k = 10^{-1}/k$.
    Finally, to approximately satisfy Assumption~\ref{ass: lip}, we enforce $J$ to be 1-Lipschitz by adding a gradient penalty~\cite{gulrajani2017improved}.     
    As stopping criterion, we set a maximum of 20 iterations, \ie generator updates, in Algorithm~\ref{alg: training}.
    We note that, in practice, the number of epochs used for $\theta^k$ are hyperparameters that need to be tuned. 
    These are important since they determine how well we approximate solutions to Line 6 in Algorithm~\ref{alg: training}. 
    To evaluate our trained projection operator on a new signal, we run ten iterations of~\eqref{eq: relaxed-projected-gradient-update}.
    For the ellipse phantom dataset, we choose $\kappa = 10^{-1}$ and $\xi = 8 \times 10^{-2}$.
    For the human phantom dataset, we choose $\kappa = 8 \times 10^{-2}$ and $\xi = 5 \times 10^{-1}$.
    Like any optimization algorithm, these hyperparameters are application dependent and require tuning.
    
    \begin{table}[t]
        \centering
        \caption*{CT Results on Human Phantoms}
        \begin{tabular}{@{}ccc@{}}
        \toprule
        Method & Avg. PSNR (dB) & Avg. SSIM
        \\
        \midrule
        Filtered Backprojection & 15.90 & 0.467
        \\
        Total Variation & 21.55 & 0.728 
        \\     
        Adversarial Regularizers & 24.86 & 0.747
        \\
        Wasserstein-based Projections (ours) & {\bf 26.66} & {\bf 0.782} 
        \\
        \bottomrule
        \vspace{.0mm}
        \end{tabular}
        \caption{Average PSNR and SSIM on validation dataset with 2,000 images of human phantoms.}
        \label{tab:lodopab_results}
    \end{table}
    
    \begin{figure}[t]
      \centering
      \small
        \setlength{\tabcolsep}{0.1pt}
        \begin{tabular}{ccccc}
            ground truth & FBP & TV & Adv. Reg. & Adv. Proj.
            \\

        \begin{tikzpicture} [spy using outlines={rectangle, magnification=3, size=1cm, connect spies}, rounded corners]
                \node[anchor=south west,inner sep=0] (image) at (0,0) {\adjincludegraphics[angle=180,width=0.185\textwidth]{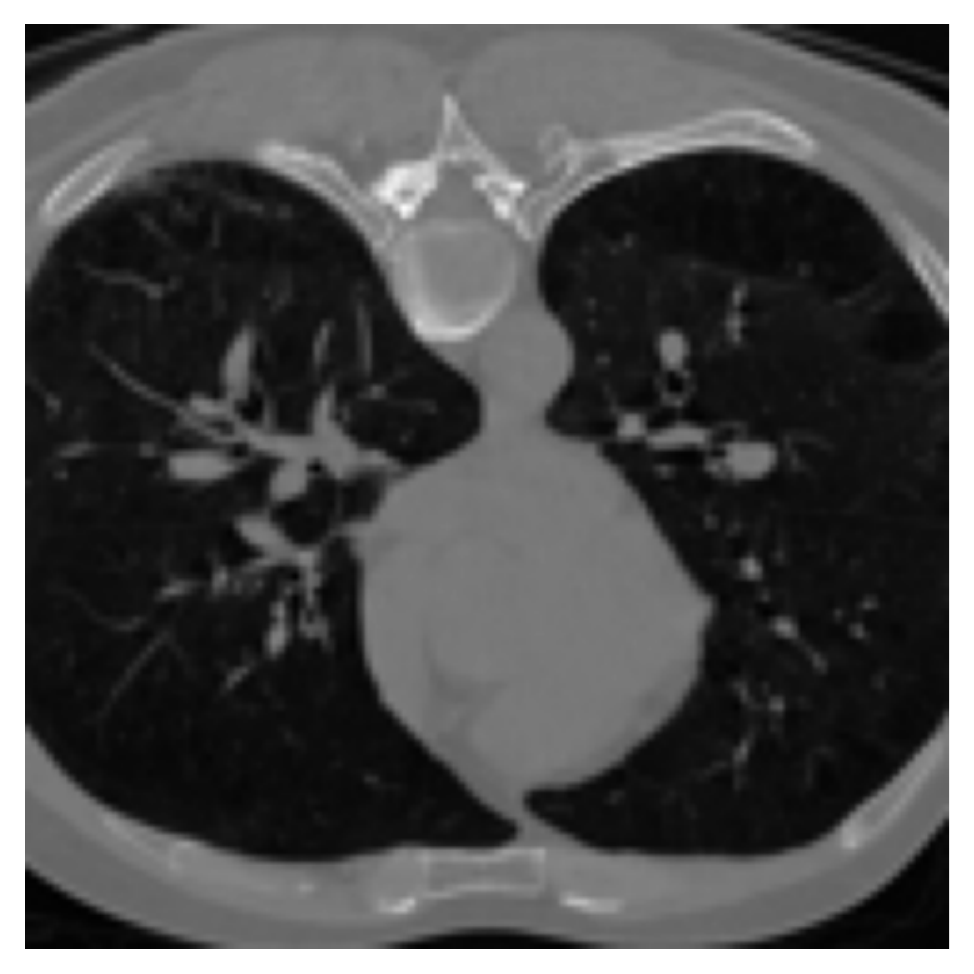}}; 
                \begin{scope}
                \spy[\spycolor,size=0.175\textwidth, every spy on node/.append style={line width = \W}] on (1.6,0.6) in node at (1.44, -1.5);  
                \end{scope}
        \end{tikzpicture}
        &
        \begin{tikzpicture} [spy using outlines={rectangle, magnification=3, size=1cm, connect spies}, rounded corners]
                \node[anchor=south west,inner sep=0] (image) at (0,0) {\adjincludegraphics[angle=180,width=0.185\textwidth]{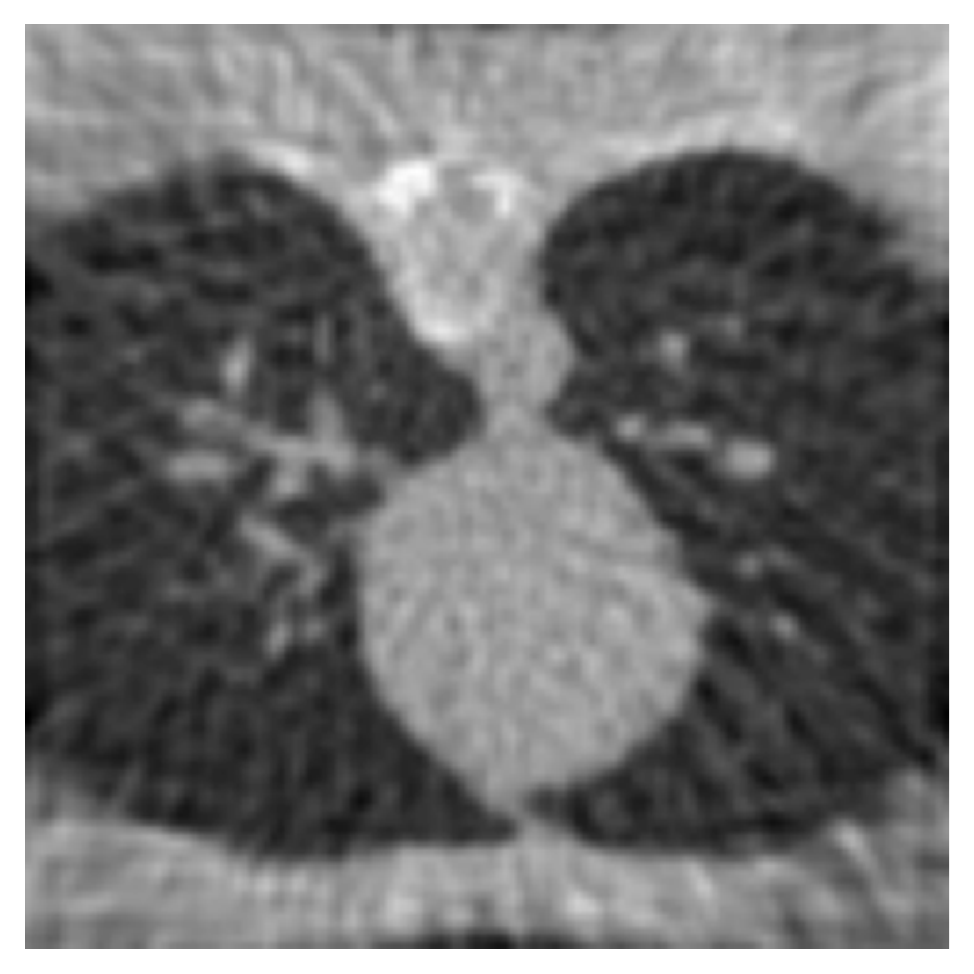}};
                \begin{scope}
                \spy[\spycolor,size=0.175\textwidth, every spy on node/.append style={line width = \W}] on (1.6,0.6) in node at (1.44, -1.5);  
                \end{scope}
        \end{tikzpicture}       
        &
        \begin{tikzpicture} [spy using outlines={rectangle, magnification=3, size=1cm, connect spies}, rounded corners]
                \node[anchor=south west,inner sep=0] (image) at (0,0) {\adjincludegraphics[angle=180,width=0.185\textwidth]{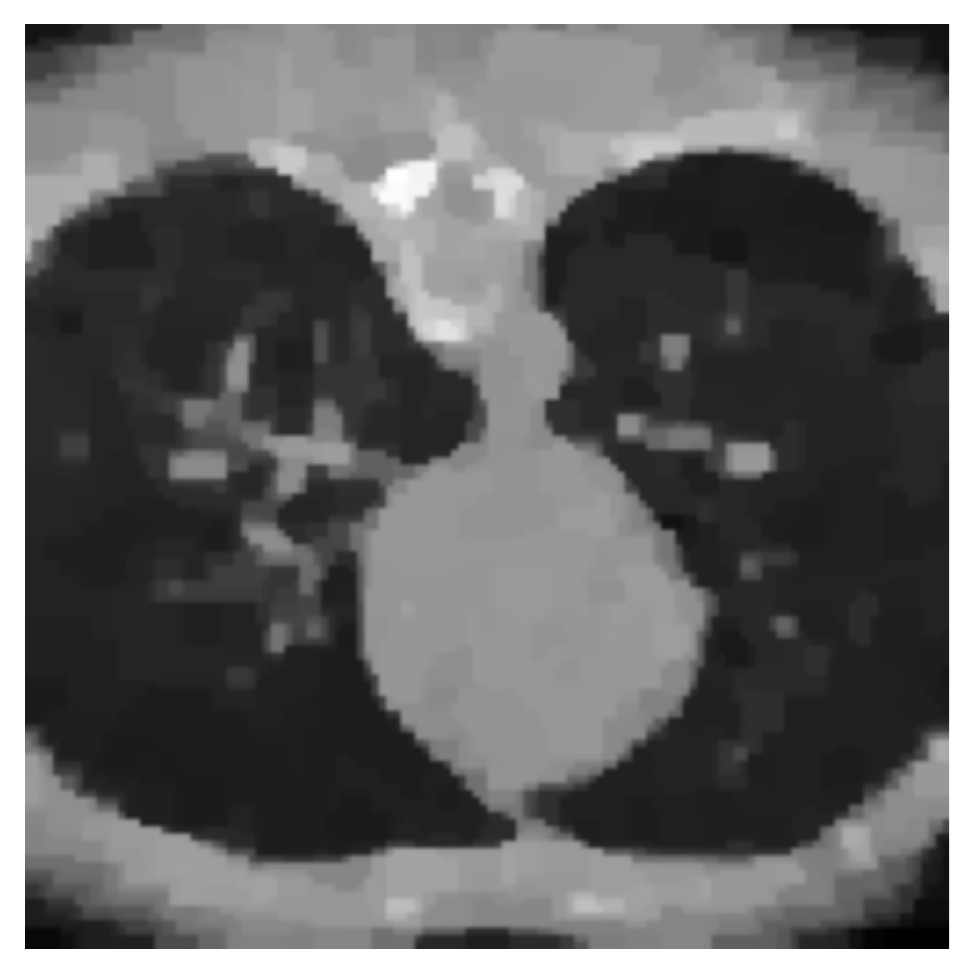}}; 
                \begin{scope}
                \spy[\spycolor,size=0.175\textwidth, every spy on node/.append style={line width = \W}] on (1.6,0.6) in node at (1.44, -1.5);  
                \end{scope}
        \end{tikzpicture}   
        &
        \begin{tikzpicture} [spy using outlines={rectangle, magnification=3, size=1cm, connect spies}, rounded corners]
                \node[anchor=south west,inner sep=0] (image) at (0,0) {\adjincludegraphics[angle=180,width=0.185\textwidth]{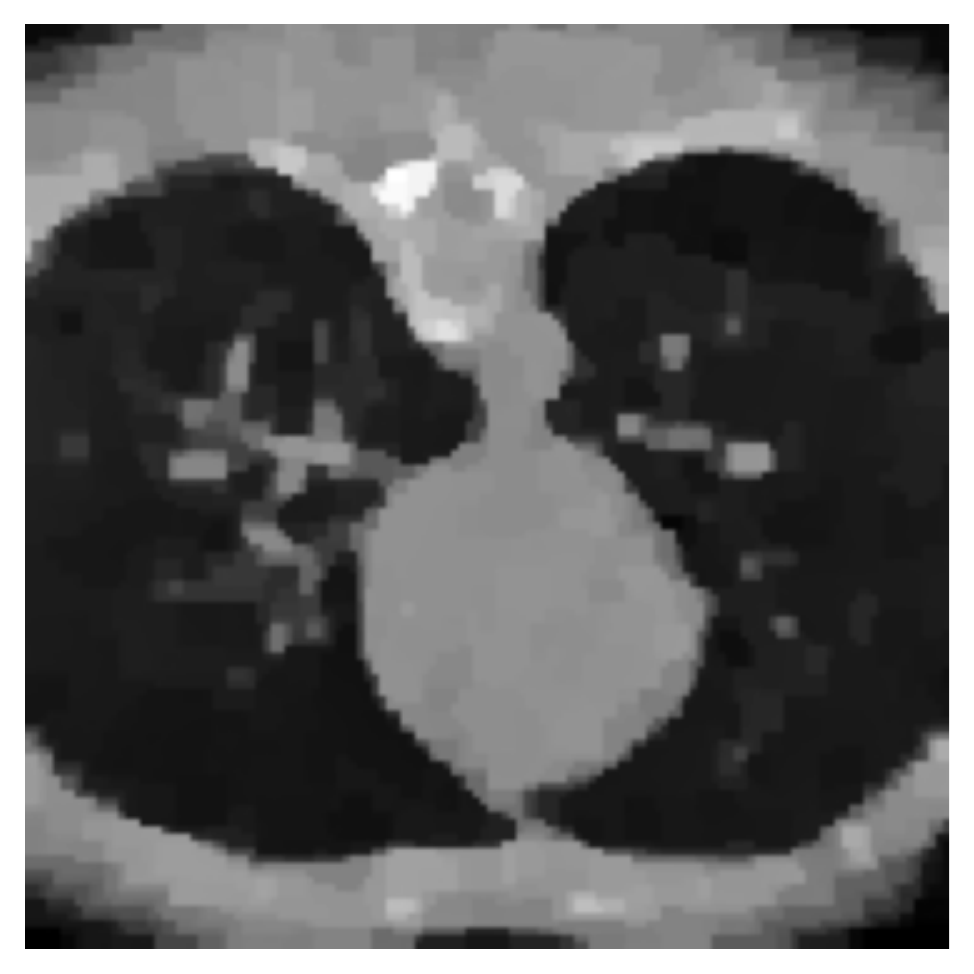}}; 
                \begin{scope}
                \spy[\spycolor,size=0.175\textwidth, every spy on node/.append style={line width = \W}] on (1.6,0.6) in node at (1.44, -1.5);  
                \end{scope}
        \end{tikzpicture}       
        &
        \begin{tikzpicture} [spy using outlines={rectangle, magnification=3, size=1cm, connect spies}, rounded corners]
                \node[anchor=south west,inner sep=0] (image) at (0,0) {\adjincludegraphics[angle=180,width=0.185\textwidth]{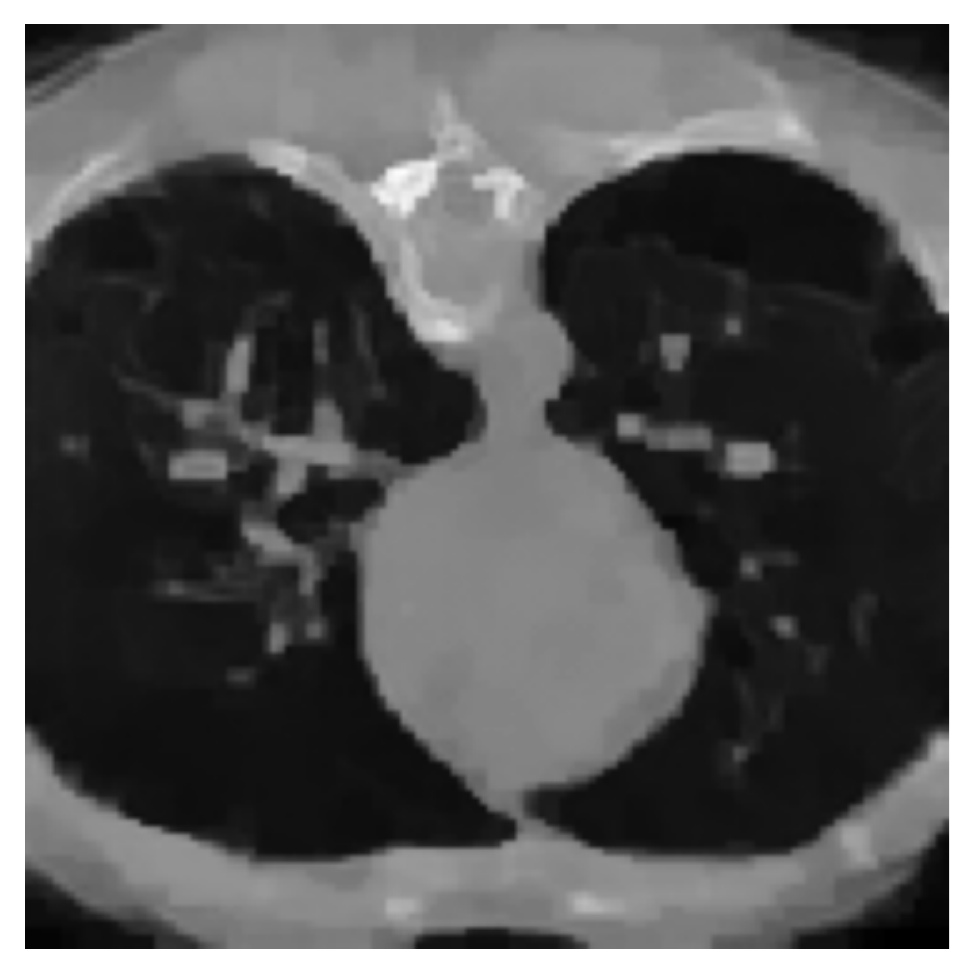}}; 
                \begin{scope}
                \spy[\spycolor,size=0.175\textwidth, every spy on node/.append style={line width = \W}] on (1.6,0.6) in node at (1.44, -1.5);  
                \end{scope}
        \end{tikzpicture}   
        \\
        & SSIM: 0.467 & SSIM: 0.712 & SSIM: 0.763 & SSIM: 0.804
        \\
        & PSNR: 15.00 & PSNR: 19.63 & PSNR: 23.14  & PSNR: 26.12
        \end{tabular}
        \caption{Reconstruction on a validation sample obtained with Filtered Back Projection (FBP) method, TV regularization, Adversarial Regularizer, and Wasserstein-based Projections (left to right). Bottom row shows expanded version of corresponding cropped region indicated by red box.}
        \label{fig:lodopab_reconstructions_1}
    \end{figure}

    \paragraph{Adversarial Regularizer Training Setup}
    To train the adversarial regularizers, we use the code provided in~\cite{advregGithub}. 
    Here, for the ellipses we use a learning rate of $10^{-4}$, a batchsize of 16, and a gradient-norm-weight of 20, and for the LoDoPaB dataset we use a learning rate of $10^{-4}$, a batchsize of 32, and a gradient-norm-weight of $20$, and we also did not use the ``unregularised minimization" option.
    We note that the setup for adversarial regularizers in~\cite{lunz2019adversarial} adds white Gaussian noise {independent} of the data, and is therefore different from our setup 
    As a result, we re-train the adversarial regularizers to match our setup and tune the reconstruction parameters to the best of our ability.
    In particular, after the regularizer is trained, we tune the regularization parameter, stepsize, and number of gradient steps (see Algorithm 2 in~\cite{lunz2019adversarial}) for the highest PSNR. 
    For a fair comparison, the adversarial regularizer is also trained on TV reconstructions as the initial distribution $\bbP^1$.
    \newparagraph
    
    \begin{remarkx}
        When the supports of $\bbP^k$ and $\PU$ do not overlap, it can be shown that distance function $d_{\sM}$ satisfies (\eg see~\cite[Cor. 1]{arjovsky2017wasserstein})
        \begin{equation}
            \bbE_{u\sim\bbP^k}\left[ \|\nabla d_{\sM}(u) \|^2\right] = 1.
            \label{eq: dm-gradient}
        \end{equation}
        This equality can be used to serve as a stopping criterion when tuning the weights $\theta$ to find $J_{\theta^k} = d_{\sM}$. That is, one can compute the expectation on the left hand side of (\ref{eq: dm-gradient}) with the estimate $J_\theta$ in place of $d_{\sM}$.
        During training, this expectation will gradually increase, being bounded above. A stopping criterion can be when this expectation stops increasing or is close to unity. 
    \end{remarkx}
    
    \paragraph{Experimental Results}
    In Tables~\ref{tab:ellipse_results} and~\ref{tab:lodopab_results}, we compare the average PSNR and SSIM on the validation datasets (1,000 images) for the ellipse dataset and LoDoPaB dataset (2,000 images), respectively. 
    These results compare Wasserstein-based projections with FBP, TV, and adversarial regularizers.
    We also show an ellipses image in Figure~\ref{fig:ellipses_reconstructions_3} and a LoDoPab image in Figure~\ref{fig:lodopab_reconstructions_1}.
    For the adversarial regularizers, we find that using 25 steps with a stepsize of 0.05 and a regularization parameter of 2 leads to the highest PSNR on the ellipse dataset. Similarly, we find that using 25 steps with a stepsize of 0.01 and a regularization parameter of 2 leads to the highest PSNR on the LoDoPaB dataset.
    While Wasserstein-based projections performs the best, we note that, \eg some ellipses are not reconstructed fully reconstructed (n.b. this is also the case for adversarial regularizers). 
    This is due to the fact that the initial TV reconstruction completely erases some ellipses due to the sparse angle setup.
    In this case, we have that some modes collapse, and Assumption~\ref{ass: projection-manifold-recovery} is not entirely satisfied. In particular, we obtain that the pushforward is simply a subset of the true manifold $\mathcal{M}$.
    Finally, it is worth noting that projection is expected to do well on signals that are close or on the manifold. That is, the data that ``looks like'' the training data. Nevertheless, one reason the Wasserstein-based projections does not overfit by introducing artifacts (as is common in standard deep learning for imaging) is that our approach is unsupervised. 
    More image reconstructions can be found in Appendix~\ref{subsec:more_reconstructions}.

\section{Conclusion} \label{sec:conclusion}
    We present a new mechanism for solving inverse problems. Our primary result provides a way to solve the variational problem (\ref{eq: VP-manifold}), which requires recovered signal estimates to lay on the underlying low dimensional manifold of true data.
    Our contribution that makes this possible is to show, by solving \textit{unsupervised} learning problems, we can project signals directly onto the manifold.
    That is, given an algorithm within the standard optimization framework for solving the constrained problem (\ref{eq: VP-manifold}) (\eg ADMM or projected gradient), we provide a Wasserstein-based method for solving each projection subproblem defined in an algorithm's update formula.    
    Our method for solving these subproblems is iterative and provably converges to the desired projections (in probability). The presented numerical experiments illustrate these results with projected gradient methods. These examples show Wasserstein-based projections outperform adversarial regularizers, a state-of-the-art unsupervised learning method, on CT image reconstruction problems.
    \newparagraph

    This work may inspire several extensions.
    At the aggregate level of distributions, Wasserstein-based projections may be viewed as a subgradient method for minimizing the Wasserstein-1 distance between the distribution of initial estimates and the true distribution. 
    Indeed, the training process consists of solving a sequence of minimization problems that may be interpreted as training a discriminator (similar to WGANs).    Future work will further investigate this connection. 
    Another extension to our work we intend to investigate is the semi-supervised regime, where we have labels for some of the data, and to investigate inclusion of the measurement data into the projection scheme. Additionally, we intend to investigate guidelines on the design of more effective network architectures such as PDE-based neural networks~\cite{haber2017stable,ruthotto2019deep}.

\section*{Acknowledgments}
 We thank the reviewers for their thoughtful efforts in providing feedback that enabled us to improve the writing of this paper.  
 
\bibliographystyle{siamplain}
\bibliography{WP_references}

\begin{thebibliography}{100}

\bibitem{jonas_adler_2017_249479}
{\sc J.~Adler, H.~Kohr, and O.~Öktem}, {\em Operator discretization library
  (odl)}, Jan. 2017, \url{https://doi.org/10.5281/zenodo.249479},
  \url{https://doi.org/10.5281/zenodo.249479}.

\bibitem{adler2017solving}
{\sc J.~Adler and O.~{\"O}ktem}, {\em Solving ill-posed inverse problems using
  iterative deep neural networks}, Inverse Problems, 33 (2017), p.~124007.

\bibitem{adler2018learned}
{\sc J.~Adler and O.~{\"O}ktem}, {\em Learned primal-dual reconstruction}, IEEE
  transactions on medical imaging, 37 (2018), pp.~1322--1332.

\bibitem{anil2019sorting}
{\sc C.~Anil, J.~Lucas, and R.~Grosse}, {\em Sorting out lipschitz function
  approximation}, in International Conference on Machine Learning, 2019,
  pp.~291--301.

\bibitem{arjovsky2017wasserstein}
{\sc M.~Arjovsky, S.~Chintala, and L.~Bottou}, {\em Wasserstein generative
  adversarial networks}, in International Conference on Machine Learning, 2017,
  pp.~214--223.

\bibitem{arridge1999optical}
{\sc S.~R. Arridge}, {\em Optical tomography in medical imaging}, Inverse
  problems, 15 (1999), p.~R41.

\bibitem{arridge2009optical}
{\sc S.~R. Arridge and J.~C. Schotland}, {\em Optical tomography: forward and
  inverse problems}, Inverse problems, 25 (2009), p.~123010.

\bibitem{baguer2020computed}
{\sc D.~O. Baguer, J.~Leuschner, and M.~Schmidt}, {\em Computed tomography
  reconstruction using deep image prior and learned reconstruction methods},
  arXiv preprint arXiv:2003.04989,  (2020).

\bibitem{bauschke2017convex}
{\sc H.~H. Bauschke and P.~L. Combettes}, {\em Convex Analysis and Monotone
  Operator Theory in Hilbert Spaces}, Springer International Publishing, 2017,
  \url{https://doi.org/10.1007/978-3-319-48311-5},
  \url{https://doi.org/10.1007%2F978-3-319-48311-5}.

\bibitem{bauschke2002phase}
{\sc H.~H. Bauschke, P.~L. Combettes, and D.~R. Luke}, {\em Phase retrieval,
  error reduction algorithm, and fienup variants: a view from convex
  optimization}, JOSA A, 19 (2002), pp.~1334--1345.

\bibitem{beck2009fast}
{\sc A.~Beck and M.~Teboulle}, {\em A fast iterative shrinkage-thresholding
  algorithm for linear inverse problems}, SIAM journal on imaging sciences, 2
  (2009), pp.~183--202.

\bibitem{belkin2003laplacian}
{\sc M.~Belkin and P.~Niyogi}, {\em Laplacian eigenmaps for dimensionality
  reduction and data representation}, Neural computation, 15 (2003),
  pp.~1373--1396.

\bibitem{benning2017learning}
{\sc M.~Benning, G.~Gilboa, J.~S. Grah, and C.-B. Sch{\"o}nlieb}, {\em Learning
  filter functions in regularisers by minimising quotients}, in International
  Conference on Scale Space and Variational Methods in Computer Vision,
  Springer, 2017, pp.~511--523.

\bibitem{bottou2010large}
{\sc L.~Bottou}, {\em Large-scale machine learning with stochastic gradient
  descent}, in Proceedings of COMPSTAT'2010, Springer, 2010, pp.~177--186.

\bibitem{bui2013computational}
{\sc T.~Bui-Thanh, O.~Ghattas, J.~Martin, and G.~Stadler}, {\em A computational
  framework for infinite-dimensional bayesian inverse problems part i: The
  linearized case, with application to global seismic inversion}, SIAM Journal
  on Scientific Computing, 35 (2013), pp.~A2494--A2523.

\bibitem{calvetti2003tikhonov}
{\sc D.~Calvetti and L.~Reichel}, {\em Tikhonov regularization of large linear
  problems}, BIT Numerical Mathematics, 43 (2003), pp.~263--283.

\bibitem{candes2015phase}
{\sc E.~J. Candes, Y.~C. Eldar, T.~Strohmer, and V.~Voroninski}, {\em Phase
  retrieval via matrix completion}, SIAM review, 57 (2015), pp.~225--251.

\bibitem{candes2006quantitative}
{\sc E.~J. Candes and J.~Romberg}, {\em Quantitative robust uncertainty
  principles and optimally sparse decompositions}, Foundations of Computational
  Mathematics, 6 (2006), pp.~227--254.

\bibitem{candes2006robust}
{\sc E.~J. Cand{\`e}s, J.~Romberg, and T.~Tao}, {\em Robust uncertainty
  principles: Exact signal reconstruction from highly incomplete frequency
  information}, IEEE Transactions on information theory, 52 (2006),
  pp.~489--509.

\bibitem{carlsson2008local}
{\sc G.~Carlsson, T.~Ishkhanov, V.~De~Silva, and A.~Zomorodian}, {\em On the
  local behavior of spaces of natural images}, International journal of
  computer vision, 76 (2008), pp.~1--12.

\bibitem{cayton2005algorithms}
{\sc L.~Cayton}, {\em Algorithms for manifold learning},  (2005).

\bibitem{cegielski2012iterative}
{\sc A.~Cegielski}, {\em Iterative methods for fixed point problems in Hilbert
  spaces}, vol.~2057, Springer, 2012.

\bibitem{chambolle2011first}
{\sc A.~Chambolle and T.~Pock}, {\em A first-order primal-dual algorithm for
  convex problems with applications to imaging}, Journal of mathematical
  imaging and vision, 40 (2011), pp.~120--145.

\bibitem{chan2020two}
{\sc R.~H. Chan, K.~K. Kan, M.~Nikolova, and R.~J. Plemmons}, {\em A two-stage
  method for spectral--spatial classification of hyperspectral images}, Journal
  of Mathematical Imaging and Vision,  (2020), pp.~1--18.

\bibitem{chan2016plug}
{\sc S.~H. Chan, X.~Wang, and O.~A. Elgendy}, {\em Plug-and-play admm for image
  restoration: Fixed-point convergence and applications}, IEEE Transactions on
  Computational Imaging, 3 (2016), pp.~84--98.

\bibitem{chaudhari2019entropy}
{\sc P.~Chaudhari, A.~Choromanska, S.~Soatto, Y.~LeCun, C.~Baldassi, C.~Borgs,
  J.~Chayes, L.~Sagun, and R.~Zecchina}, {\em Entropy-sgd: Biasing gradient
  descent into wide valleys}, Journal of Statistical Mechanics: Theory and
  Experiment, 2019 (2019), p.~124018.

\bibitem{chen2017low}
{\sc H.~Chen, Y.~Zhang, M.~K. Kalra, F.~Lin, Y.~Chen, P.~Liao, J.~Zhou, and
  G.~Wang}, {\em Low-dose ct with a residual encoder-decoder convolutional
  neural network}, IEEE transactions on medical imaging, 36 (2017),
  pp.~2524--2535.

\bibitem{cohen2020regularization}
{\sc R.~Cohen, M.~Elad, and P.~Milanfar}, {\em Regularization by denoising via
  fixed-point projection (red-pro)}, arXiv preprint arXiv:2008.00226,  (2020).

\bibitem{combettes2020deep}
{\sc P.~L. Combettes and J.-C. Pesquet}, {\em Deep neural network structures
  solving variational inequalities}, Set-Valued and Variational Analysis,
  (2020), pp.~1--28.

\bibitem{combettes2020lipschitz}
{\sc P.~L. Combettes and J.-C. Pesquet}, {\em Lipschitz certificates for
  layered network structures driven by averaged activation operators}, SIAM
  Journal on Mathematics of Data Science, 2 (2020), pp.~529--557.

\bibitem{cucker2002best}
{\sc F.~Cucker and S.~Smale}, {\em Best choices for regularization parameters
  in learning theory: on the bias-variance problem}, Foundations of
  computational Mathematics, 2 (2002), pp.~413--428.

\bibitem{deutsch2001best}
{\sc F.~Deutsch}, {\em Best Approximation in Inner Product Spaces}, vol.~7,
  Springer Science \& Business Media, 2001.

\bibitem{dokmanic2016inverse}
{\sc I.~Dokmani{\'c}, J.~Bruna, S.~Mallat, and M.~de~Hoop}, {\em Inverse
  problems with invariant multiscale statistics}, arXiv preprint
  arXiv:1609.05502,  (2016).

\bibitem{donoho2006compressed}
{\sc D.~L. Donoho}, {\em Compressed sensing}, IEEE Transactions on information
  theory, 52 (2006), pp.~1289--1306.

\bibitem{donoho2000high}
{\sc D.~L. Donoho et~al.}, {\em High-dimensional data analysis: The curses and
  blessings of dimensionality}, AMS math challenges lecture, 1 (2000), p.~32.

\bibitem{esser2010general}
{\sc E.~Esser, X.~Zhang, and T.~F. Chan}, {\em A general framework for a class
  of first order primal-dual algorithms for convex optimization in imaging
  science}, SIAM Journal on Imaging Sciences, 3 (2010), pp.~1015--1046.

\bibitem{fung2019large}
{\sc S.~W. Fung}, {\em Large-Scale Parameter Estimation in Geophysics and
  Machine Learning}, PhD thesis, Emory University, 2019.

\bibitem{fung2020multigrid}
{\sc S.~W. Fung and Z.~W. Di}, {\em Multigrid optimization for large-scale
  ptychographic phase retrieval}, SIAM Journal on Imaging Sciences, 13 (2020),
  pp.~214--233.

\bibitem{fung2019multiscale}
{\sc S.~W. Fung and L.~Ruthotto}, {\em A multiscale method for model order
  reduction in {PDE} parameter estimation}, Journal of Computational and
  Applied Mathematics, 350 (2019), pp.~19--34.

\bibitem{fung2019uncertainty}
{\sc S.~W. Fung and L.~Ruthotto}, {\em An uncertainty-weighted asynchronous
  {ADMM} method for parallel {PDE} parameter estimation}, SIAM Journal on
  Scientific Computing, 41 (2019), pp.~S129--S148.

\bibitem{galantai2003projectors}
{\sc A.~Gal{\'a}ntai}, {\em Projectors and projection methods}, vol.~6,
  Springer Science \& Business Media, 2003.

\bibitem{gao2017properties}
{\sc B.~Gao and L.~Pavel}, {\em On the properties of the softmax function with
  application in game theory and reinforcement learning}, arXiv preprint
  arXiv:1704.00805,  (2017).

\bibitem{gilboa2013expert}
{\sc G.~Gilboa}, {\em Expert regularizers for task specific processing}, in
  International Conference on Scale Space and Variational Methods in Computer
  Vision, Springer, 2013, pp.~24--35.

\bibitem{golub1999tikhonov}
{\sc G.~H. Golub, P.~C. Hansen, and D.~P. O'Leary}, {\em Tikhonov
  regularization and total least squares}, SIAM journal on matrix analysis and
  applications, 21 (1999), pp.~185--194.

\bibitem{goodfellow2014generative}
{\sc I.~Goodfellow, J.~Pouget-Abadie, M.~Mirza, B.~Xu, D.~Warde-Farley,
  S.~Ozair, A.~Courville, and Y.~Bengio}, {\em Generative adversarial nets}, in
  Advances in neural information processing systems, 2014, pp.~2672--2680.

\bibitem{gulrajani2017improved}
{\sc I.~Gulrajani, F.~Ahmed, M.~Arjovsky, V.~Dumoulin, and A.~C. Courville},
  {\em Improved training of wasserstein gans}, in Advances in neural
  information processing systems, 2017, pp.~5767--5777.

\bibitem{haber2000fast}
{\sc E.~Haber, U.~Ascher, D.~Aruliah, and D.~Oldenburg}, {\em Fast simulation
  of 3d electromagnetic problems using potentials}, Journal of Computational
  Physics, 163 (2000), pp.~150--171.

\bibitem{haber2004inversion}
{\sc E.~Haber, U.~M. Ascher, and D.~W. Oldenburg}, {\em Inversion of 3d
  electromagnetic data in frequency and time domain using an inexact
  all-at-once approach}, Geophysics, 69 (2004), pp.~1216--1228.

\bibitem{haber2017stable}
{\sc E.~Haber and L.~Ruthotto}, {\em Stable architectures for deep neural
  networks}, Inverse Problems, 34 (2017), p.~014004.

\bibitem{halpern1967fixed}
{\sc B.~Halpern}, {\em Fixed points of nonexpanding maps}, Bulletin of the
  American Mathematical Society, 73 (1967), pp.~957--961.

\bibitem{hammernik2018learning}
{\sc K.~Hammernik, T.~Klatzer, E.~Kobler, M.~P. Recht, D.~K. Sodickson,
  T.~Pock, and F.~Knoll}, {\em Learning a variational network for
  reconstruction of accelerated {MRI} data}, Magnetic resonance in medicine, 79
  (2018), pp.~3055--3071.

\bibitem{hansen2006deblurring}
{\sc P.~C. Hansen, J.~G. Nagy, and D.~P. O'leary}, {\em Deblurring images:
  matrices, spectra, and filtering}, SIAM, 2006.

\bibitem{he2015delving}
{\sc K.~He, X.~Zhang, S.~Ren, and J.~Sun}, {\em Delving deep into rectifiers:
  Surpassing human-level performance on imagenet classification}, in
  Proceedings of the IEEE international conference on computer vision, 2015,
  pp.~1026--1034.

\bibitem{jin2017deep}
{\sc K.~H. Jin, M.~T. McCann, E.~Froustey, and M.~Unser}, {\em Deep
  convolutional neural network for inverse problems in imaging}, IEEE
  Transactions on Image Processing, 26 (2017), pp.~4509--4522.

\bibitem{kan2020pnkh}
{\sc K.~Kan, S.~W. Fung, and L.~Ruthotto}, {\em {PNKH-B}: A projected
  newton-krylov method for large-scale bound-constrained optimization}, arXiv
  preprint arXiv:2005.13639,  (2020).

\bibitem{keskar2016large}
{\sc N.~S. Keskar, D.~Mudigere, J.~Nocedal, M.~Smelyanskiy, and P.~T.~P. Tang},
  {\em On large-batch training for deep learning: Generalization gap and sharp
  minima}, arXiv preprint arXiv:1609.04836,  (2016).

\bibitem{kingma2014adam}
{\sc D.~P. Kingma and J.~Ba}, {\em Adam: A method for stochastic optimization},
  arXiv preprint arXiv:1412.6980,  (2014).

\bibitem{kobler2017variational}
{\sc E.~Kobler, T.~Klatzer, K.~Hammernik, and T.~Pock}, {\em Variational
  networks: connecting variational methods and deep learning}, in German
  conference on pattern recognition, Springer, 2017, pp.~281--293.

\bibitem{lee2003nonlinear}
{\sc A.~B. Lee, K.~S. Pedersen, and D.~Mumford}, {\em The nonlinear statistics
  of high-contrast patches in natural images}, International Journal of
  Computer Vision, 54 (2003), pp.~83--103.

\bibitem{leuschner2019lodopab}
{\sc J.~Leuschner, M.~Schmidt, D.~O. Baguer, and P.~Maa{\ss}}, {\em {The
  LoDoPaB-CT dataset}: A benchmark dataset for low-dose {CT} reconstruction
  methods}, arXiv preprint arXiv:1910.01113,  (2019).

\bibitem{lin2020apac}
{\sc A.~T. Lin, S.~W. Fung, W.~Li, L.~Nurbekyan, and S.~J. Osher}, {\em
  {APAC-Net}: Alternating the population and agent control via two neural
  networks to solve high-dimensional stochastic mean field games}, arXiv
  preprint arXiv:2002.10113,  (2020).

\bibitem{lin2019fluid}
{\sc J.~Lin, K.~Lensink, and E.~Haber}, {\em Fluid flow mass transport for
  generative networks}, arXiv preprint arXiv:1910.01694,  (2019).

\bibitem{lin2008riemannian}
{\sc T.~Lin and H.~Zha}, {\em Riemannian manifold learning}, IEEE Transactions
  on Pattern Analysis and Machine Intelligence, 30 (2008), pp.~796--809.

\bibitem{liu1995nonlinear}
{\sc L.-S. Liu}, {\em for nonlinear strongly accretive mappings in banach
  spaces}, Journal of Mathematical Analysis and Applications, 194 (1995),
  pp.~114--125.

\bibitem{advregGithub}
{\sc S.~Lunz}.
\newblock \url{https://github.com/lunz-s/DeepAdverserialRegulariser}, 2018.

\bibitem{lunz2019adversarial}
{\sc S.~Lunz, O.~\"{O}ktem, and C.-B. Sch\"{o}nlieb}, {\em Adversarial
  regularizers in inverse problems}, in Advances in Neural Information
  Processing Systems 31, S.~Bengio, H.~Wallach, H.~Larochelle, K.~Grauman,
  N.~Cesa-Bianchi, and R.~Garnett, eds., Curran Associates, Inc., 2018,
  pp.~8507--8516,
  \url{http://papers.nips.cc/paper/8070-adversarial-regularizers-in-inverse-problems.pdf}.

\bibitem{maaten2008visualizing}
{\sc L.~v.~d. Maaten and G.~Hinton}, {\em Visualizing data using t-sne},
  Journal of machine learning research, 9 (2008), pp.~2579--2605.

\bibitem{meinhardt2017learning}
{\sc T.~Meinhardt, M.~Moller, C.~Hazirbas, and D.~Cremers}, {\em Learning
  proximal operators: Using denoising networks for regularizing inverse imaging
  problems}, in Proceedings of the IEEE International Conference on Computer
  Vision, 2017, pp.~1781--1790.

\bibitem{miyato2018spectral}
{\sc T.~Miyato, T.~Kataoka, M.~Koyama, and Y.~Yoshida}, {\em Spectral
  normalization for generative adversarial networks}, arXiv preprint
  arXiv:1802.05957,  (2018).

\bibitem{moreau1965proximite}
{\sc J.-J. Moreau}, {\em Proximit{\'e} et dualit{\'e} dans un espace
  hilbertien}, Bulletin de la Soci{\'e}t{\'e} math{\'e}matique de France, 93
  (1965), pp.~273--299.

\bibitem{onken2020ot}
{\sc D.~Onken, S.~W. Fung, X.~Li, and L.~Ruthotto}, {\em {OT-Flow}: Fast and
  accurate continuous normalizing flows via optimal transport}, arXiv preprint
  arXiv:2006.00104,  (2020).

\bibitem{osher2005iterative}
{\sc S.~Osher, M.~Burger, D.~Goldfarb, J.~Xu, and W.~Yin}, {\em An iterative
  regularization method for total variation-based image restoration},
  Multiscale Modeling \& Simulation, 4 (2005), pp.~460--489.

\bibitem{osher2017low}
{\sc S.~Osher, Z.~Shi, and W.~Zhu}, {\em Low dimensional manifold model for
  image processing}, SIAM Journal on Imaging Sciences, 10 (2017),
  pp.~1669--1690.

\bibitem{ouyang2015accelerated}
{\sc Y.~Ouyang, Y.~Chen, G.~Lan, and E.~Pasiliao~Jr}, {\em An accelerated
  linearized alternating direction method of multipliers}, SIAM Journal on
  Imaging Sciences, 8 (2015), pp.~644--681.

\bibitem{paszke2019pytorch}
{\sc A.~Paszke, S.~Gross, F.~Massa, A.~Lerer, J.~Bradbury, G.~Chanan,
  T.~Killeen, Z.~Lin, N.~Gimelshein, L.~Antiga, et~al.}, {\em Pytorch: An
  imperative style, high-performance deep learning library}, in Advances in
  neural information processing systems, 2019, pp.~8026--8037.

\bibitem{pearson1901liii}
{\sc K.~Pearson}, {\em Liii. on lines and planes of closest fit to systems of
  points in space}, The London, Edinburgh, and Dublin Philosophical Magazine
  and Journal of Science, 2 (1901), pp.~559--572.

\bibitem{peyre2008image}
{\sc G.~Peyr{\'e}}, {\em Image processing with nonlocal spectral bases},
  Multiscale Modeling \& Simulation, 7 (2008), pp.~703--730.

\bibitem{peyre2009manifold}
{\sc G.~Peyr{\'e}}, {\em Manifold models for signals and images}, Computer
  vision and image understanding, 113 (2009), pp.~249--260.

\bibitem{reich1979constructive}
{\sc S.~Reich}, {\em Constructive techniques for accretive and monotone
  operators}, in Applied nonlinear analysis, Elsevier, 1979, pp.~335--345.

\bibitem{rick2017one}
{\sc J.~Rick~Chang, C.-L. Li, B.~Poczos, B.~Vijaya~Kumar, and A.~C.
  Sankaranarayanan}, {\em One network to solve them all--solving linear inverse
  problems using deep projection models}, in Proceedings of the IEEE
  International Conference on Computer Vision, 2017, pp.~5888--5897.

\bibitem{robbins1951stochastic}
{\sc H.~Robbins and S.~Monro}, {\em A stochastic approximation method}, The
  annals of mathematical statistics,  (1951), pp.~400--407.

\bibitem{ronneberger2015u}
{\sc O.~Ronneberger, P.~Fischer, and T.~Brox}, {\em {U-Net}: Convolutional
  networks for biomedical image segmentation}, in International Conference on
  Medical image computing and computer-assisted intervention, Springer, 2015,
  pp.~234--241.

\bibitem{rudin1992nonlinear}
{\sc L.~I. Rudin, S.~Osher, and E.~Fatemi}, {\em Nonlinear total variation
  based noise removal algorithms}, Physica D: nonlinear phenomena, 60 (1992),
  pp.~259--268.

\bibitem{ruthotto2019deep}
{\sc L.~Ruthotto and E.~Haber}, {\em Deep neural networks motivated by partial
  differential equations}, Journal of Mathematical Imaging and Vision,  (2019),
  pp.~1--13.

\bibitem{ruthotto2020machine}
{\sc L.~Ruthotto, S.~J. Osher, W.~Li, L.~Nurbekyan, and S.~W. Fung}, {\em A
  machine learning framework for solving high-dimensional mean field game and
  mean field control problems}, Proceedings of the National Academy of
  Sciences, 117 (2020), pp.~9183--9193.

\bibitem{shah2018solving}
{\sc V.~Shah and C.~Hegde}, {\em Solving linear inverse problems using gan
  priors: An algorithm with provable guarantees}, in 2018 IEEE international
  conference on acoustics, speech and signal processing (ICASSP), IEEE, 2018,
  pp.~4609--4613.

\bibitem{sun2016deep}
{\sc J.~Sun, H.~Li, Z.~Xu, et~al.}, {\em {Deep ADMM-Net} for compressive
  sensing {MRI}}, in Advances in neural information processing systems, 2016,
  pp.~10--18.

\bibitem{talwalkar2008large}
{\sc A.~Talwalkar, S.~Kumar, and H.~Rowley}, {\em Large-scale manifold
  learning}, in 2008 IEEE Conference on Computer Vision and Pattern
  Recognition, IEEE, 2008, pp.~1--8.

\bibitem{tanaka2019discriminator}
{\sc A.~Tanaka}, {\em Discriminator optimal transport}, in Advances in Neural
  Information Processing Systems (NeurIPS, 2019, pp.~6813--6823.

\bibitem{tenenbaum2000global}
{\sc J.~B. Tenenbaum, V.~De~Silva, and J.~C. Langford}, {\em A global geometric
  framework for nonlinear dimensionality reduction}, Science, 290 (2000),
  pp.~2319--2323.

\bibitem{thao2020phase}
{\sc N.~H. Thao, D.~R. Luke, O.~Soloviev, and M.~Verhaegen}, {\em Phase
  retrieval with sparse phase constraint}, SIAM Journal on Mathematics of Data
  Science, 2 (2020), pp.~246--263.

\bibitem{ulyanov2018deep}
{\sc D.~Ulyanov, A.~Vedaldi, and V.~Lempitsky}, {\em Deep image prior}, in
  Proceedings of the IEEE Conference on Computer Vision and Pattern
  Recognition, 2018, pp.~9446--9454.

\bibitem{maaten2009dimensionality}
{\sc L.~Van Der~Maaten, E.~Postma, and J.~Van~den Herik}, {\em Dimensionality
  reduction: a comparative}, J Mach Learn Res, 10 (2009), p.~13.

\bibitem{venkatakrishnan2013Plug}
{\sc S.~V. {Venkatakrishnan}, C.~A. {Bouman}, and B.~{Wohlberg}}, {\em
  Plug-and-play priors for model based reconstruction}, in 2013 IEEE Global
  Conference on Signal and Information Processing, 2013, pp.~945--948.

\bibitem{villani2008optimal}
{\sc C.~Villani}, {\em Optimal transport: old and new}, vol.~338, Springer
  Science \& Business Media, 2008.

\bibitem{vito2005learning}
{\sc E.~D. Vito, L.~Rosasco, A.~Caponnetto, U.~D. Giovannini, and F.~Odone},
  {\em Learning from examples as an inverse problem}, Journal of Machine
  Learning Research, 6 (2005), pp.~883--904.

\bibitem{wang2016perspective}
{\sc G.~Wang}, {\em A perspective on deep imaging}, IEEE access, 4 (2016),
  pp.~8914--8924.

\bibitem{fung2020admm}
{\sc S.~Wu~Fung, S.~Tyrv{\"a}inen, L.~Ruthotto, and E.~Haber}, {\em
  {ADMM}-{S}oftmax: An {ADMM} approach for multinomial logistic regression},
  Electronic Transactions on Numerical Analysis, 52 (2020), pp.~214--229.

\bibitem{xu2002iterative}
{\sc H.-K. Xu}, {\em Iterative algorithms for nonlinear operators}, Journal of
  the London Mathematical Society, 66 (2002), pp.~240--256.

\bibitem{xu2012low}
{\sc Q.~Xu, H.~Yu, X.~Mou, L.~Zhang, J.~Hsieh, and G.~Wang}, {\em Low-dose
  x-ray {CT} reconstruction via dictionary learning}, IEEE transactions on
  medical imaging, 31 (2012), pp.~1682--1697.

\bibitem{yan2006graph}
{\sc S.~Yan, D.~Xu, B.~Zhang, H.-J. Zhang, Q.~Yang, and S.~Lin}, {\em Graph
  embedding and extensions: A general framework for dimensionality reduction},
  IEEE transactions on pattern analysis and machine intelligence, 29 (2006),
  pp.~40--51.

\bibitem{yao2017strong}
{\sc Y.~Yao, M.~Postolache, and S.~Naseer}, {\em Strong convergence of halpern
  method for firmly type nonexpansive mappings}, Journal of Nonlinear Science
  and Applications, 10 (2017), pp.~5932--5938.

\bibitem{yin2008bregman}
{\sc W.~Yin, S.~Osher, D.~Goldfarb, and J.~Darbon}, {\em Bregman iterative
  algorithms for $\ell_1$-minimization with applications to compressed
  sensing}, SIAM Journal on Imaging sciences, 1 (2008), pp.~143--168.

\bibitem{zaeemzadeh2020normpreservation}
{\sc A.~{Zaeemzadeh}, N.~{Rahnavard}, and M.~{Shah}}, {\em Norm-preservation:
  Why residual networks can become extremely deep?}, IEEE Transactions on
  Pattern Analysis and Machine Intelligence,  (2020), pp.~1--1.

\end{thebibliography}

\newpage 

\appendix 
\section{Proofs} \label{sec:appendix-proofs}
 
    Herein we prove all of our results. We restate each result before its proof for the reader's convenience. 
    
    {\bf Theorem \ref{thm:distance-attains-sup}.}
    \textit{Under Assumptions \ref{ass: manifold-support} and \ref{ass: projection-manifold-recovery}, for all $k\in \bbN$, $\tau\in [0,\infty)$, and $p\in[1,\infty)$, the pointwise distance function $d_{\sM}$ is a solution to (\ref{eq:distance-sup-problem}). 
        where $\Gamma$ is the set of nonnegative 1-Lipschitz functions  mapping $\sX$ to $\bbR$. 
        Moreover, when $\tau > 0$, the restriction of each minimizer $f^\star$ of (\ref{eq:distance-sup-problem}) to the support of $\bbP^k$ is unique, \ie
        (\ref{eq: distance-support-equality}) holds.}
    \begin{proof}
        This proof is an extension of the proof of Theorem 2 in \cite{lunz2019adversarial}. 
        We proceed by first showing the pointwise distance function $d_{\sM}$ is 1-Lipschitz (Step 1). Then we verify $d_{\sM}$ is a solution to \eqref{eq:distance-sup-problem} (Step 2). Lastly, we verify that $d_{\sM}$ is the unique solution when $\tau > 0$ (Step 3).\\
                
        {\bf Step 1.} For any $u,v\in\sX$, observe that
        \begin{align}
            d_{\sM}(u) - d_{\sM}(v)
            &= \min_{x\in\sM} \|u-x\| - \min_{x\in\sM} \|v-x\| \label{eq:dm-1-lipschitz}\\
            &= \min_{x\in\sM} \|u-x\| - \|v - P_{\sM}(v)\|  \\
            & \leq  \| u - P_{\sM}(v)\| - \|v - P_{\sM}(v)\|  \\
            & \leq \| u-v\|,
        \end{align}
        where the second equality holds since Assumption \ref{ass: manifold-support} implies the projection $P_{\sM}(v)$ of $v$ is unique and the final inequality is a rearrangement of the triangle inequality.        
        An analogous sequence of inequalities verifies that the inequality holds with $u$ and $v$ swapped in the left hand side of (\ref{eq:dm-1-lipschitz}).
        Thus, the pointwise distance function $d_{\sM}$ is 1-Lipschitz.\\
        
        {\bf Step 2.}
        The objective function $\phi: \Gamma \rightarrow \bbR$ of interest is defined by
        \begin{equation}
            \phi(f) \triangleq \bbE_{u\sim \PU}\left[ f(u) + \tau f(u)^p\right] -  \bbE_{u^k\sim\bbP^k}\left[ f(u^k) \right].
        \end{equation}   
        Fix any $f\in\Gamma$. By Assumption \ref{ass: projection-manifold-recovery}, 
        \begin{equation}
        \phi(f)
        =   \bbE_{u\sim\bbP^k} \left[  f(P_{\sM}(u^k)) - f(u^k)   + \tau f(P_{\sM}(u^k))^p   \right]. 
        \end{equation}
        For all $u^k\in\sX$,
        {\small 
        \begin{equation}
            f(u^k) - f(P_{\sM}(u^k))
            \leq \| u^k - P_{\sM}(u^k) \| 
            = d_{\sM}(u^k)
            \ \ \Longrightarrow \ \
            f(P_{\sM}(u^k)) - f(u^k) \geq - d_{\sM}(u^k), 
        \end{equation}}  
        and so
        \begin{align}
             \underbrace{f(P_{\sM}(u^k))  -f(u^k)}_{\geq -d_{\sM}(u^k)} + \underbrace{\tau f(P_{\sM}(u^k))^p}_{\geq 0}
            \geq -  d_{\sM}(u^k).
        \end{align}
        Therefore, taking expectations reveals the inequality
        \begin{align}
            \phi(f)
            \geq \bbE_{u^k\sim\bbP^k}\left[ -d_{\sM}(u^k) \right] 
            = \phi(d_{\sM}) \label{eq:objective-minimum},
        \end{align}        
        where the final equality holds since $d_{\sM}(u) = 0$ for all $u \in \sM$.
        Because (\ref{eq:objective-minimum}) holds for arbitrarily chosen $f\in\Gamma$, we conclude $d_{\sM}$ is a minimizer of $\phi$ over $\Gamma$.\\
        
        {\bf Step 3.} All that remains is to verify the restriction of minimizers of $\phi$ to $\mbox{supp}(\bbP^k)$ is unique. Let $f$ be a minimizer of $\phi$ over $\Gamma$. By way of contradiction, suppose there exists a set $V\subseteq  \sX$ with positive measure (\ie $\PU[V] > 0$) for which $f > 0$ on $V$. Choose a compact subset $K\subseteq V$ with positive measure and set
        \begin{equation}
            \xi \triangleq \min_{u\in K} f(u).
            \label{eq:xi-K}
        \end{equation}
        By the continuity of $f$ and the compactness of $K$, it follows that the minimum in (\ref{eq:xi-K}) exists. Moreover, by the choice of $K$, we deduce $\xi > 0$.
        Then observe
        \begin{equation}
            \bbE_{u^k\sim\bbP^k}\left[ f(P_{\sM}(u^k))^p\right]
            = \int_{\sX} f(P_{\sM}(u^k))^p \ \mbox{d}\bbP^k(u^k)
            \geq \int_{K}  \underbrace{f(u)^p}_{\geq \xi^p} \ \mbox{d}\PU(u) 
            \geq \xi^p\cdot \PU[ K].
        \end{equation}
        Then applying the above results yields
        \begin{align}
            \phi(f)
            & = \bbE_{u\sim\bbP^k}\Big[  \underbrace{f(P_{\sM}(u))-f(u)}_{\geq -d_{\sM}(u)} + \underbrace{ \tau f(P_{\sM}(u))^p}_{\geq \tau \xi^p\cdot \PU[K]} \Big]  \\
            &\geq \underbrace{\bbE_{u\sim\bbP^k}[-d_{\sM}(u)]}_{=\phi(d_{\sM})} + \underbrace{\tau \xi^p\cdot \PU[K]}_{>0}\\
            & > \phi(d_{\sM}),
        \end{align}
        from which we deduce $f$ is not a minimizer of $\phi$ over $\Gamma$, a contradiction. This contradiction proves the earlier assumption was false, and so $f= 0$ on $\sM$.\\
        
        Using the fact that $f = 0$ on $\sM$, observe that
        \begin{equation}
            f(u)
            = f(u) - f(P_{\sM}(u))
            \leq \|u - P_{\sM}(u)\|
            = d_{\sM}(u),
            \ \ \ \mbox{for all $u\in\sX$,}
        \end{equation}
        \ie $f \leq d_{\sM}$.        
        By way of contradiction, suppose there is a set $U$ of positive measure (\ie $\bbP^k[U] > 0$) such that $f < d_{\sM}$ in $U$. 
        Let $\tilde{K}$ be a compact subset of $U$ (\ie $\tilde{K} \subseteq U$) with positive measure and set
        \begin{equation}
            \zeta \triangleq \min_{u \in \tilde{K}} d_{\sM}(u) - f(u) > 0.
        \end{equation}
        This implies 
        \begin{align}
            \varepsilon \triangleq \int_{\tilde{K}} d_{\sM}(u) - f(u) \ \mbox{d}\bbP^k(u)           
            \geq \zeta \cdot \bbP^k[\tilde{K}]
            > 0. 
            \label{eq:diff-pos-measure}
        \end{align}
        Then
        \begin{equation}
            \bbE_{u^k\sim \bbP^k}\left[ d_{\sM}(u^k) - f(u^k) \right]
            = \int_{\sX} d_{\sM}(u^k) - f(u^k) \ \mbox{d}\bbP^k(u^k)
            \geq \varepsilon,
        \end{equation}
        where the final inequality holds by (\ref{eq:diff-pos-measure}) and the fact $f\leq d_{\sM}$.
        Rearranging reveals
        \begin{equation}
            \phi(d_{\sM})
            < \phi(d_{\sM}) + \varepsilon
            = - \bbE_{u\sim \bbP^k}\left[ d_{\sM}(u) \right] + \varepsilon
            \leq -  \bbE_{u\sim \bbP^k}\left[ f(u) \right]
            = \phi(f),
        \end{equation}
        and so $f$ does not minimize $\phi$.
        This contradiction proves the assumption that $U$ has positive measure was false, from which we deduce $f \geq d_{\sM}$ in $\mbox{supp}(\bbP^k)$.
        Therefore, combining our results, we conclude  $f = d_{\sM}$   in $\mbox{supp}(\bbP^k)$.
    \end{proof}

The following lemma can be found in various forms in the literature (\eg see \cite{liu1995nonlinear,reich1979constructive,xu2002iterative}) and is used below in the proof of our main result.
    
    \begin{lemma} \label{lemma: delta-k-inequality-conv} 
    If $\{\delta_n\}$ is a sequence of nonnegative real numbers such that
    \begin{equation}
        \delta_{k+1} \leq (1-\gamma_k) \delta_k + \gamma_k \sigma_k, \ \ \ \mbox{for all $k\in\bbN$,}
        \label{eq: lemma-delta-k-inequality}
    \end{equation}
    where $\{\gamma_k\}$ is a sequence in $(0,1]$ and $\{\sigma_k\}$ is a sequence in $\bbR$ such that
    \begin{equation}
        \sum_{k\in\bbN} \gamma_k = \infty
    \end{equation}
    and
    \begin{equation}
        \limsup_{k\rightarrow\infty} \sigma_k \leq 0,
    \end{equation}
    then 
    \begin{equation}
        \limk \delta_k = 0.
    \end{equation}
    \end{lemma}
    
    \ \\
    
    The next lemma identifies the relationship between the step size $\beta_k$ and the average distance to the manifold. 
    
    \begin{lemma}
        \label{lemma:betak}
        Under the assumptions of Theorem \ref{thm: main-result}, 
        \begin{equation}
            \beta_k = \bbE_{u^k\sim\bbP^k}\left[ d_{\sM}(u^k)\right].
        \end{equation}
    \end{lemma}
    \begin{proof}
        This follows from the fact that $u\sim\PU$ implies $d_{\sM}(u) = 0$. Substituting yields
        \begin{align}
            \beta_k
            & = \bbE_{u^k\sim\bbP^k}\left[ J_{\theta^k}(u^k)\right]
            - \bbE_{u\sim\PU }\left[ J_{\theta^k}(u)\right]\\
            &= \bbE_{u^k\sim\bbP^k}\left[ d_{\sM}(u^k)\right]
            - \underbrace{\bbE_{u\sim\PU }\left[ d_{\sM}(u)\right]}_{=0}\\
            &= \bbE_{u^k\sim\bbP^k}\left[ d_{\sM}(u^k)\right].
        \end{align}
    \end{proof}
    
    \ \\ 
    
    Below is a proof of the main result, Theorem \ref{thm: main-result}.
    The analysis for the Halpern iteration closely follows the approach in \cite{yao2017strong}. For completeness, we first restate the theorem. \\
    
    {\bf Theorem} \ref{thm: main-result}. 
        \textit{$\mathrm{(Convergence\ of\ Adversarial\ Projections)}$
        Suppose Assumptions \ref{ass: manifold-support}, \ref{ass: bounded-D1},  \ref{ass: projection-manifold-recovery}, \ref{ass: lip}, and \ref{ass: gammak-props} hold. 
        If the sequence $\{u^k\}$ is generated by Algorithm \ref{alg: adversarial-projection}, then the sequence $\{u^k\}$ converges  to $P_{\sM}(u^1)$ in mean square, and thus, in probability.}
        
    \begin{proof}  
        We proceed in the following manner. First, set $z \triangleq P_{\sM}(u^1)$.
        An inequality is derived bounding the expectation of $\|g_k(u^k)-z\|^2$ (Step 1).
        This is used to show the sequence $\{\delta_k\}$ is bounded (Step 2) and then obtain an inequality relating $\delta_{k+1}$, $\delta_k$, and $\sigma_k$ as in (\ref{eq: lemma-delta-k-inequality}) (Step 3).
        We next verify the limit supremum of the sequence $\{\sigma_k\}$ is finite (Step 4), which enables us to show a subsequence of $\{d_k\}$ converges to zero (Step 5). This enables us to prove $\delta_k \rightarrow 0$ (Step 6) through verifying $ \limsup_{k\rightarrow \infty} \sigma_k \leq 0$. This establishes convergence in mean square. By Markov's inequality, convergence in probability then follows.\\
    
        Throughout this proof we define, for all $k\in\bbN$, 
        \begin{align}
            \delta_k & \triangleq \bbE_{u^k\sim\bbP^k}\left[ \|u^k-z\|^2\right], \\
            \sigma_k & \triangleq \left[ \bbE_{u^k\sim\bbP^k}\left[ 2\braket{u^{k+1}-z,u^1-z} \right] -\dfrac{(1-\gamma_k)\beta_k^2(\mu_1+\mu_2)(2-\mu_1-\mu_2)}{\gamma_k} \right]. \label{eq: sigmak-def}
        \end{align}
        In addition, for notational brevity, we henceforth write $u^k$ in place of the random variable $\sA_k(u^1)$ (n.b. $u^1$ is the sample of the random variable, usually denoted by $\omega$).
        We also notationally suppress the dependence of $\lambda_k(u)$ on $u$, as defined in \eqref{eq:lambda-def}, by writing $\lambda_k = \lambda_k(u)$.
        \\
        
        {\bf Step 1.} We first  derive an inequality for relaxed projections.
        Define the residual operator
        \begin{equation}
            S(u) \triangleq u - P_{\sM}(u).
        \end{equation}
        Fix any $k\in\bbN$. Using \eqref{eq: gk-relaxed-projection} with $\alpha_k=\alpha_k(u)$, observe
        \begin{align}
            \|g_k(u^k)-z\|^2
            & = \|u^k + \alpha_k (P_{\sM}(u^k)-u^k) - z\|^2\\
            & = \|u^k-z\|^2 - 2\alpha_k \braket{u^k-z, u^k-P_{\sM}(u^k)} + \alpha_k^2 \|u^k-P_{\sM}(u^k)\|^2 \\
            & = \| u^k - z\|^2 - 2\alpha_k \braket{u^k - z, S(u^k) - S(z)} + \alpha_k^2 \|S(u^k)\|^2 ,
            \label{eq: gk-expansion}
        \end{align}
        where we note $S(z) = 0$ since $z = P_{\sM}(u^1) \in \sM$.
        Furthermore, $S$ is firmly nonexpansive (\eg see Prop. 4.16 in \cite{bauschke2017convex}), which implies
        \begin{equation}
            \braket{u^k - z, S(u^k) - S(z)}
            \geq \|S(u^k)-S(z)\|^2
            = \|S(u^k)\|^2.
            \label{eq: FNE-inequality}
        \end{equation}
        Combining (\ref{eq: gk-expansion}) and (\ref{eq: FNE-inequality}) yields
        \begin{align}
            \|g_k(u^k)-z\|^2
            & \leq \|u^k - z\|^2 - \alpha_k (2-\alpha_k) \|S(u^k)\|^2  \\
            & = \|u^k - z\|^2 - \alpha_k (2-\alpha_k)  d_{\sM}(u^k)^2 \\
            & = \|u^k-z\|^2 - \lambda_k (2d_{\sM}(u^k)-\lambda_k).
            \label{eq: gk-inequality}
        \end{align}
        Then observe
        \begin{align}
            & \ \ \ \ \bbE_{u^k\sim\bbP^k}\left[\lambda_k(2d_{\sM}(u^k)-\lambda_k)\right] \\
            & = \bbE_{u^k\sim\bbP^k}\left[ \lambda_k(2-\mu_2)d_{\sM}(u^k) - \mu_1\beta_k \lambda_k    \right]\\
            & = \left(\mu_1 \beta_k^2 + \mu_2  \cdot \bbE_{u^k\sim\bbP^k}\left[ d^2_{\sM}(u^k) \right]\right)(2-\mu_2) - \mu_1 \beta_k \left(\mu_1\beta_k + \mu_2 \beta_k\right)\\
            & \geq \beta_k^2 \left[ (\mu_1+\mu_2)(2-\mu_2) - \mu_1(\mu_1+\mu_2) \right] \\
            & = \beta_k^2 (\mu_1+\mu_2) (2-\mu_1-\mu_2), \label{eq: gk-descent-terms}
        \end{align}
        where we note $\beta_k$ is the expected value of the distance (see Lemma \ref{lemma:betak}) and the inequality is an application of Jensen's inequality.
        Thus, combining  (\ref{eq: gk-inequality}) and (\ref{eq: gk-descent-terms}) in expectation yields
        \begin{align}
            \bbE_{u^k \sim\bbP^k} \left[ \|g_k(u^k) -z\|^2  \right]
            & \leq \bbE_{u^k \sim\bbP^k} \left[ \|u^k-z\|^2\right]
            -  \beta_k^2 (\mu_1+\mu_2) (2-\mu_1-\mu_2) \\
            & = \delta_k - \beta_k^2 (\mu_1+\mu_2) (2-\mu_1-\mu_2),
            \label{eq: gk-inequality-expectation}
        \end{align}        
        where the choice of $\mu$ in Line 1 of Algorithm \ref{alg: training} ensures $\mu_1 + \mu_2 > 0$ and $(2-\mu_1-\mu_2) > 0$.\\

        {\bf Step 2.}  
        Expanding the expression for $\delta_{k+1}$, we deduce
        \begin{align}
            \delta_{k+1}
            & =  \bbE_{u^k\sim\bbP^k}\left[ \| u^{k+1} -z \|^2  \right]  \\
            & =  \bbE_{u^k\sim\bbP^k}\left[   \| \gamma_k u^1  + (1-\gamma_k) g_k(u^k) -  z \|^2 \right] \label{eq: uk-exp1}\\
            & \leq  \bbE_{u^k\sim\bbP^k}\left[ \gamma_k \|u^1 -z \|^2 + (1-\gamma_k)\| g_k(u^k) -z \|^2 \right] \label{eq: uk-exp2}\\
            & =   \gamma_k\cdot  \bbE_{u^k\sim\bbP^k}\left[ \|u^1 -z \|^2 \right] 
             + (1-\gamma_k)  \bbE_{u^k\sim\bbP^k}\left[\| g_k(u^k) -z \|^2 \right] \\
            & \leq \gamma_k \delta_1 + (1-\gamma_k)\delta_k \label{eq: step-2-gk-01}\\
            & \leq \max\left( \delta_1, \delta_k \right),
        \end{align}                        
        where (\ref{eq: uk-exp2}) follows from (\ref{eq: uk-exp1}) by Jensen's inequality, (\ref{eq: step-2-gk-01}) holds by applying (\ref{eq: gk-inequality-expectation}), and the final inequality holds by Assumption \ref{ass: gammak-props}i.
        Through induction, it follows that $\{\delta_k\}$ is bounded since
        \begin{equation}
            \delta_{k+1}
            \leq  \delta_1 < \infty,
            \ \ \ \mbox{for all $k\in\bbN$.}
            \label{eq: delta-k-bound}
        \end{equation}
        
        {\bf Step 3.}
        To establish a useful inequality bounding $\delta_{k+1}$, we expand this expression once again to obtain, for all $k\in\bbN$,
        {\small 
        \begin{align}
            \delta_{k+1}
            & = \bbE_{u^k\sim\bbP^k} \left[         \|u^{k+1}-z\|^2\right] \\
            & = \bbE_{u^k\sim\bbP^k}\left[ \|\gamma_k u^1 + (1-\gamma_k)  g_k(u^k) - z\|^2 \right]\\
            & = \bbE_{u^k\sim\bbP^k}\left[ \gamma_k^2 \|u^1-z\|^2 + (1-\gamma_k)^2  \| g_k(u^k) - z\|^2 +2 \gamma_k(1-\gamma_k) \braket{ u^1-z, g_k(u^k)-z} \right]\\
            & \leq (1-\gamma_k)\cdot \bbE_{u^k\sim\bbP^k}\left[ \| g_k(u^k)-z\|^2\right] + 2\gamma_k\cdot \bbE_{u^k\sim\bbP^k}\left[\braket{u^{k+1}-z, u^1-z}\right] \\
            & \leq (1-\gamma_k)\left( \delta_k- \beta_k^2(\mu_1+\mu_2)(2-\mu_1-\mu_2)\right) + 2 \gamma_k\cdot \bbE_{u^k\sim\bbP^k}\left[ \braket{u^{k+1}-z,u^1-z} \right]\\
            & = (1-\gamma_k)\delta_k + \gamma_k\left[ \bbE_{u^k\sim\bbP^k}\left[ 2\braket{u^{k+1}-z,u^1-z} \right] -\dfrac{(1-\gamma_k)\beta_k^2(\mu_1+\mu_2)(2-\mu_1-\mu_2)}{\gamma_k} \right],
            \label{eq: delta-k-halpern-bound}
        \end{align}
        }where we leverage the definition of $u^{k+1}$ and the inclusions $\gamma_k, (1-\gamma_k)\in [0,1]$.
        Substituting   the definition of $\sigma_k$ from (\ref{eq: sigmak-def})  into (\ref{eq: delta-k-halpern-bound}) yields the inequality
        \begin{equation}
            \delta_{k+1} \leq (1-\gamma_k) \delta_k + \gamma_k \sigma_k,
            \ \ \ \mbox{for all $k\in\bbN$.}
            \label{eq: delta-k-descent-inequality}
        \end{equation} 
        
        {\bf Step 4.} We now show the limit supremum of $\{\sigma_k\}$ is finite.
        Indeed, for all $k\in\bbN$,
        \begin{align}
            \sigma_k 
            &\leq \bbE_{u^k\sim\bbP^k}\left[ 2 \braket{u^{k+1}-z, u^1-z}\right]\\
            &\leq  \bbE_{u^k\sim\bbP^k}\left[ \|u^{k+1}-z\|^2 + \|u^1-z\|^2\right]\\
            &= \delta_{k+1} + \delta_1\\
            &\leq 2 \delta_1\\
            &< \infty,
        \end{align}
        and this implies
        \begin{equation}
            \limsup_{k\rightarrow\infty} \sigma_k < \infty.
        \end{equation}
        Next, by way of contradiction, suppose
        \begin{equation}
            \limsup_{k\rightarrow\infty} \sigma_k < -1.
            \label{eq: sigma-k-contradiction}
        \end{equation}
        This implies there exists $N_1 \in \bbN$ such that
        \begin{equation}
            \sigma_k \leq -1, \ \ \ \mbox{for all $k\geq N_1$,}
        \end{equation}
        and so
        \begin{equation}
            \delta_{k+1} \leq (1-\gamma_k) \delta_k  - \gamma_k
            \leq \delta_k - \gamma_k,
            \ \ \ \mbox{for all $k\geq N_1$.}
        \end{equation}
        By induction, it follows that
        \begin{equation}
            \delta_{k+1} \leq \delta_{N_1} - \sum_{\ell=N_1}^k \gamma_\ell.
        \end{equation}
        Applying Assumption \ref{ass: gammak-props}iii and letting $k\rightarrow\infty$ reveals
        \begin{equation}
            \limsup_{k\rightarrow\infty} \delta_k \leq \delta_{N_1} - \sum_{\ell=N_1}^\infty \gamma_\ell = -\infty,
        \end{equation}
        which induces a contradiction since the sequence $\{\delta_k\}$ is nonnegative. This proves (\ref{eq: sigma-k-contradiction}) is false, and so
        \begin{equation}
            -1 \leq \limsup_{k\rightarrow\infty} \sigma_k < \infty.
            \label{eq: sigmak-bounds}
        \end{equation}
        
        {\bf Step 5.}
        Because (\ref{eq: sigmak-bounds}) shows the limit supremum of the sequence $\{\sigma_k\}$ is finite, there is a convergent subsequence $\{\sigma_{n_k}\} \subseteq \{\sigma_k\} $ satisfying
        {\small 
        \begin{align}
            \limsup_{k\rightarrow\infty} \sigma_k
            &= \limk \sigma_{n_k} \\
            &= \limk \left[ \bbE_{u^{n_k}\sim\bbP^{n_k}}\left[ 2\braket{u^{k+1}-z,u^1-z} \right] -\dfrac{(1-\gamma_{n_k})\beta_{n_k}^2(\mu_1+\mu_2)(2-\mu_1-\mu_2)}{\gamma_{n_k}} \right].
            \label{eq: sigma-k-limit-exp}
        \end{align}}
        By the Cauchy-Schwarz inequality and the  result (\ref{eq: delta-k-bound}) in Step 1, 
        \begin{align}
            \bbE_{u^{n_k}\sim\bbP^{n_k}}\left[\left| \braket{u^{n_k+1}-z, u^1-z} \right| \right]
            & \leq \bbE_{u^{n_k}\sim\bbP^{n_k}}\left[ \dfrac{1}{2}\left( \|u^{n_k+1}-z\|^2 + \| u^1-z\|^2 \right)\right]\\
            & = \dfrac{1}{2}\left( \delta_{n_k+1} + \delta_1 \right) \\
            & \leq \delta_1,
        \end{align}
        and so $\{\bbE_{u^{n_k}\sim\bbP^{n_k}}[\braket{u^{n_k+1}-z, u^1-z}]\}$   is a bounded sequence of real numbers.
        Thus, it contains a convergent subsequence $\{ \braket{u^{m_k+1}-z, u^1-z}\}$ (\ie $\{m_k\}\subseteq\{n_k\}$).
        This implies, when combined with the convergence of $\{\sigma_{m_k}\}$ and (\ref{eq: sigma-k-limit-exp}),  existence of the limit
        \begin{equation}
            \limk\dfrac{(1-\gamma_{m_k})\beta_{m_k}^2(\mu_1+\mu_2)(2-\mu_1-\mu_2)}{\gamma_{m_k}}.
        \end{equation}
        Since Assumption \ref{ass: gammak-props}ii asserts $\gamma_{k} \rightarrow 0$, it follows that
        \begin{equation}
            \limk (1-\gamma_k)\beta_{m_k}^2(\mu_1+\mu_2)(2-\mu_1-\mu_2)= 0
            \ \ \ \Longrightarrow \ \ \ 
            \limk \beta_{m_k} = 0,
            \label{eq: dmk-limit}
        \end{equation}
        \ie a subsequence $\{\beta_{m_k}\}$ of $\{\beta_k\}$ converges to zero.\\

        {\bf Step 6.}
        Observe, for all $k\in\bbN$,
        \begin{align}
            \bbE_{u^k\sim\bbP^k}\left[ \|u^k-u^1\| \right]
            & \leq \bbE_{u^k\sim\bbP^k}\left[ \|u^k-z\| + \|u^1 - z\| \right] \\
            & \leq \sqrt{\bbE_{u^k\sim\bbP^k}\left[ \|u^k-z\|^2\right]} + \sqrt{\bbE_{u^k\sim\bbP^k}\left[ \|u^1-z\|^2\right]} \\
            & = \sqrt{\delta_k} + \sqrt{\delta_1} \\
            & \leq 2\sqrt{\delta_1},
            \label{eq: uk-diff-bound}
        \end{align}
        where the second equality is an application of Jensen's inequality and the final inequality follows from \eqref{eq: delta-k-bound}. Applying (\ref{eq: uk-diff-bound}) reveals
        \begin{align}
            \bbE_{u^k\sim\bbP^k} \left[\| u^{k+1} - u^k\|\right]
            & \leq \gamma_{k} \bbE_{u^k\sim\bbP^k}\left[ \|u^k - u^1\|\right] + (1-\gamma_k)\bbE_{u^k\sim\bbP^k}\left[ \| g_k(u^k) - u^k\|\right] \\
            & \leq 2\gamma_k \sqrt{\delta_1} + (1-\gamma_k) \   \bbE_{u^k\sim\bbP^k} \left[  \lambda_k \right] \label{eq: projection-unity-norm}\\
            & = 2\gamma_k\sqrt{\delta_1} + (1-\gamma_k)(\mu_1+\mu_2)\beta_k,
        \end{align}       
        where (\ref{eq: projection-unity-norm}) holds since, by the choice of $g_k$ in (\ref{eq: gk-relaxed-projection}),
        \begin{equation}
            g_k(u) - u
            \in  \lambda_k \partial d_{\sM}(u)
            \ \ \ \Longrightarrow \ \ \ 
            \|g_k(u) - u\| \leq \lambda_k ,
        \end{equation}
        with the implication following from the fact that $\partial d_{\sM}(u)$ is a subset of the unit ball centered at the origin (since $d_{\sM}$ is 1-Lipschitz).        
        Utilizing (\ref{eq: dmk-limit}) and the fact $\gamma_{m_k}\rightarrow 0$, we deduce
        \begin{equation}
            \limk \bbE_{u^{m_k}\sim\bbP^{m_k}}\left[ \|u^{m_k+1} - u^{m_k}\|\right]
            \leq \limk  2\gamma_{m_k} \sqrt{\delta_1} + (1-\gamma_{m_k}) (\mu_1+\mu_2) \beta_{m_k}
            = 0.
        \end{equation}
        Because the left hand side is nonnegative, the squeeze lemma implies
        \begin{equation}
            \limk \bbE_{u^{m_k}\sim\bbP^{m_k}}\left[ \|u^{m_k+1} - u^{m_k}\|\right] = 0.
            \label{eq: umk-diff-conv-to-0}
        \end{equation} 
        Also, the boundedness of $\mbox{supp}(\bbP^1)$ and $\sM$ implies there exists a constant $C >0$ such that
        \begin{equation}
            \bbP^1\big[\|u^1\| \leq C \big] = \bbP^1 \big[ \|z\| \leq C \big] = 1.
        \end{equation}
        Thus, 
        \begin{equation}
            \bbP^1 \big[  2\|u^1-z\|  \leq 4C \big] = 1,
        \end{equation}
        from which we deduce
        \begin{align}
            \limsup_{k\rightarrow\infty}\sigma_k
            &= \limk \sigma_{m_k} \\
            &= \limk \left[  \bbE_{u^{m_k}\sim\bbP^{m_k}} \left[2\braket{u^{m_k+1}-z, u^1-z}\right] - \dfrac{\mu (1-\gamma_{m_k})}{4\gamma_{m_k}}\cdot \beta_{m_k}^2 \right]\\
            &\leq \limk \bbE_{u^{m_k}\sim\bbP^{m_k}}\left[ 2\braket{u^{m_k+1}-z, u^1-z}\right]\\
            & = \limk \bbE_{u^{m_k}\sim\bbP^{m_k}}\left[ 2\braket{P_{\sM}(u^{m_k+1})-z, u^1-z} \right] \\
            & + \bbE_{u^{m_k}\sim\bbP^{m_k}}\left[ 2\braket{u^{m_k+1} - P_{\sM}(u^{m_k+1}), u^1-z}\right]\\
            & \leq \limk 4C \cdot \bbE_{u^{m_k}\sim\bbP^{m_k}}\left[ \|u^{m_k+1} - P_{\sM}(u^{m_k+1})\| \right], 
            \label{eq: sigma-k-bound}
        \end{align}
        where the final inequality holds by application of the Cauchy Schwarz inequality, and utilizing the fact that  $z = P_{\sM}(u^1)$ and, by the projection identity (\eg see Thm. 3.16 in \cite{bauschke2017convex}, Thm 4.1 \cite{deutsch2001best}, and Thm 7.45 in \cite{galantai2003projectors}),
        \begin{equation}
            \braket{ v - P_{\sM}(u^1), u^1 - P_{\sM}(u^1)} \leq 0,
            \ \ \ \mbox{for all $v \in \sM$.}
        \end{equation}
        Applying the triangle inequality with the 1-Lipschitz property of the projection $P_{\sM}$ yields
        \begin{align}
            \limsup_{k\rightarrow\infty}\sigma_k
            & \leq  \limk C \cdot \bbE_{u^{m_k}\sim\bbP^{m_k}}\left[ \|u^{m_k+1}- u^{m_k}\|  + \|u^{m_k} - P_{\sM}(u^{m_k})\| \right] \\
            & + \bbE_{u^{m_k}\sim\bbP^{m_k}}\left[ \|P_{\sM}(u^{m_k}) - P_{\sM}(u^{m_k+1}) \|  \right]  \\
            & \leq  \limk C \cdot \bbE_{u^{m_k}\sim\bbP^{m_k}}\left[ 2\|u^{m_k+1} - u^{m_k} \| + \|u^{m_k} - P_{\sM}(u^{m_k})\|  \right]   \\
            & =  \limk C \cdot  \left(\bbE_{u^{m_k}\sim\bbP^{m_k}}\left[ 2\|u^{m_k+1} - u^{m_k} \|  \right] + \beta_{m_k} \right) \label{eq: expectation-limit-sigmak-1} \\
            & = 0, \label{eq: expectation-limit-sigmak-2}
        \end{align} 
        where (\ref{eq: expectation-limit-sigmak-2}) follows from (\ref{eq: expectation-limit-sigmak-1}) by (\ref{eq: dmk-limit}) and (\ref{eq: umk-diff-conv-to-0}).       
        Now, since the limit supremum of $\{\sigma_k\}$ is nonpositive, we may apply Lemma \ref{lemma: delta-k-inequality-conv} to (\ref{eq: delta-k-descent-inequality}) to deduce
        \begin{equation}
            \delta_k \rightarrow 0
            \ \  \ \Longrightarrow \ \ \ 
            \limk \bbE_{u^k\sim\bbP^k}\left[ \|u^k -z \|^2 \right]
            = 0,
        \end{equation}
        completing the proof.         
    \end{proof}

\newpage
\section{More Reconstructions} 
\label{subsec:more_reconstructions}
This final subsection presents additional figures of the CT image reconstructions from our numerical examples.

    
    \begin{figure}[ht]
      \centering
      \small
        \setlength{\tabcolsep}{0.1pt}
        \begin{tabular}{ccccc}
            ground truth & FBP & TV & Adv. Reg. & Adv. Proj.
            \\

        \begin{tikzpicture} [spy using outlines={rectangle, magnification=3, size=1cm, connect spies}, rounded corners]
                \node[anchor=south west,inner sep=0] (image) at (0,0) {\adjincludegraphics[width=0.185\textwidth]{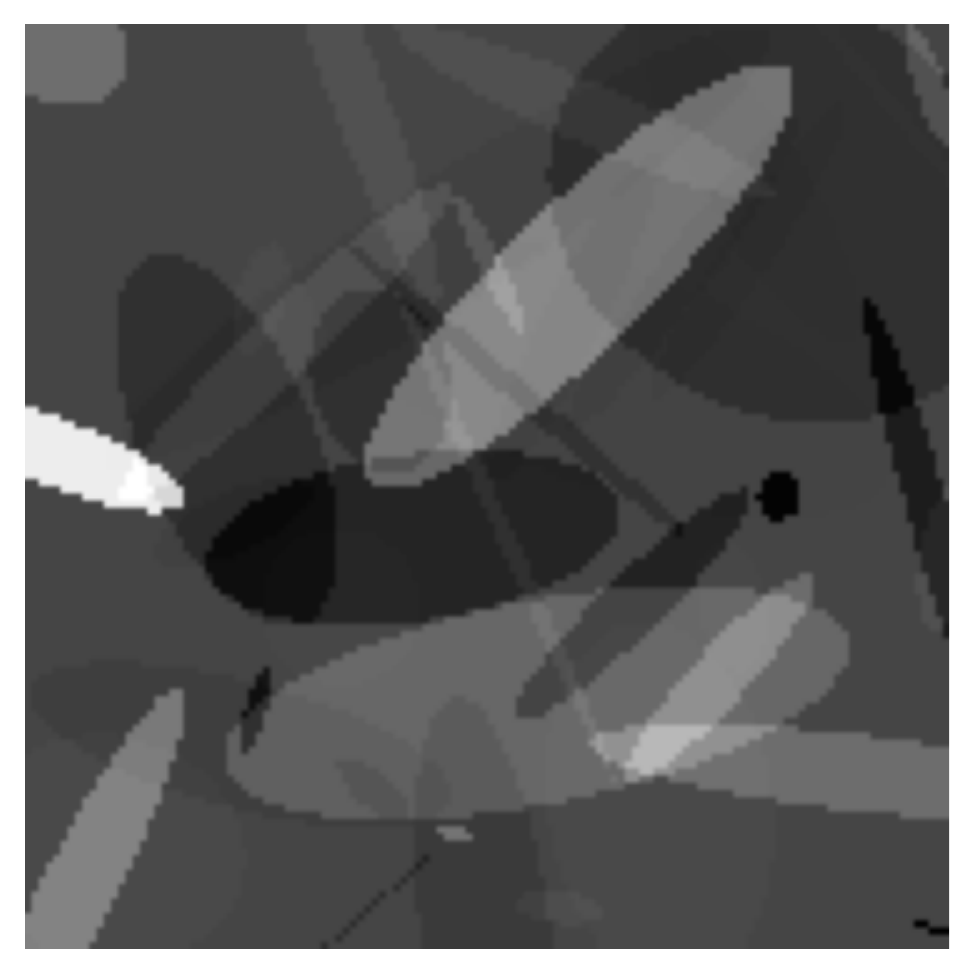}}; 
                \begin{scope}
                \spy[\spycolor,size=0.175\textwidth, every spy on node/.append style={line width = \W}] on (0.8,1.5) in node at (1.44, -1.5);  
                \end{scope}
        \end{tikzpicture}
        &
        \begin{tikzpicture} [spy using outlines={rectangle, magnification=3, size=1cm, connect spies}, rounded corners]
                \node[anchor=south west,inner sep=0] (image) at (0,0) {\adjincludegraphics[width=0.185\textwidth]{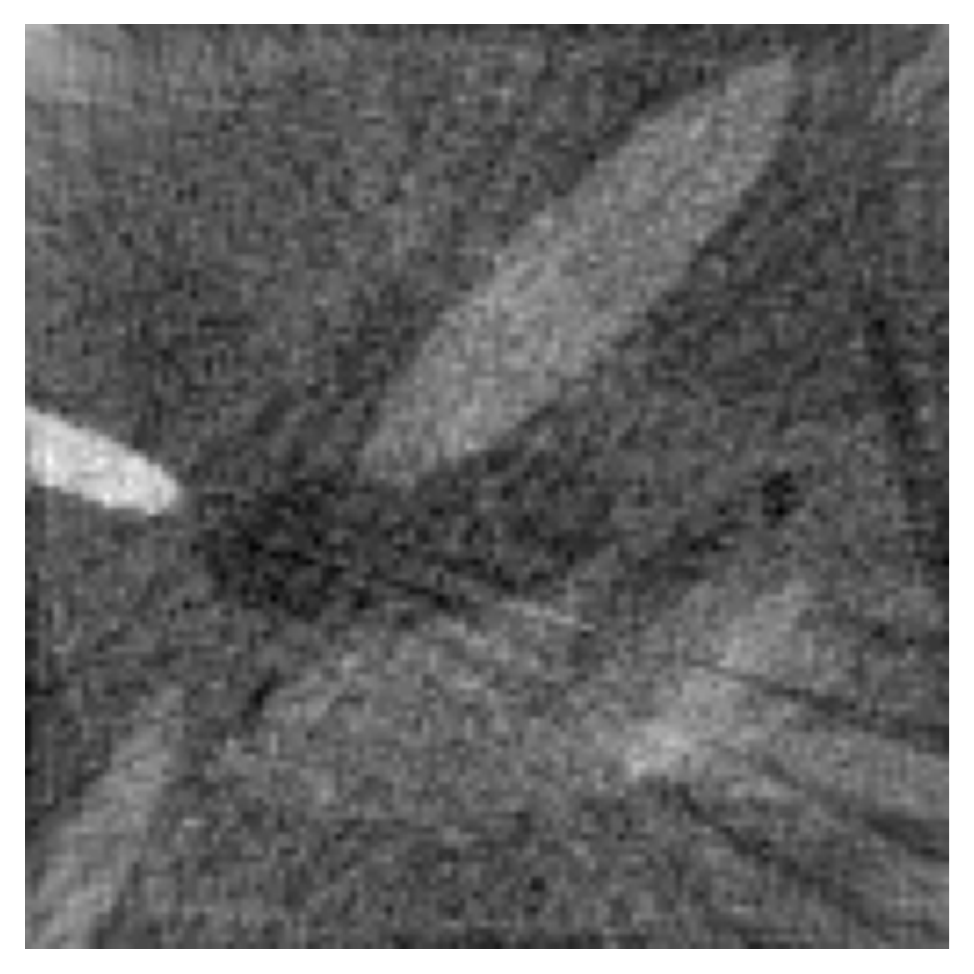}};
                \begin{scope}
                \spy[\spycolor,size=0.175\textwidth, every spy on node/.append style={line width = \W}] on (0.8,1.5) in node at (1.44, -1.5);  
                \end{scope}
        \end{tikzpicture}       
        &
        \begin{tikzpicture} [spy using outlines={rectangle, magnification=3, size=1cm, connect spies}, rounded corners]
                \node[anchor=south west,inner sep=0] (image) at (0,0) {\adjincludegraphics[width=0.185\textwidth]{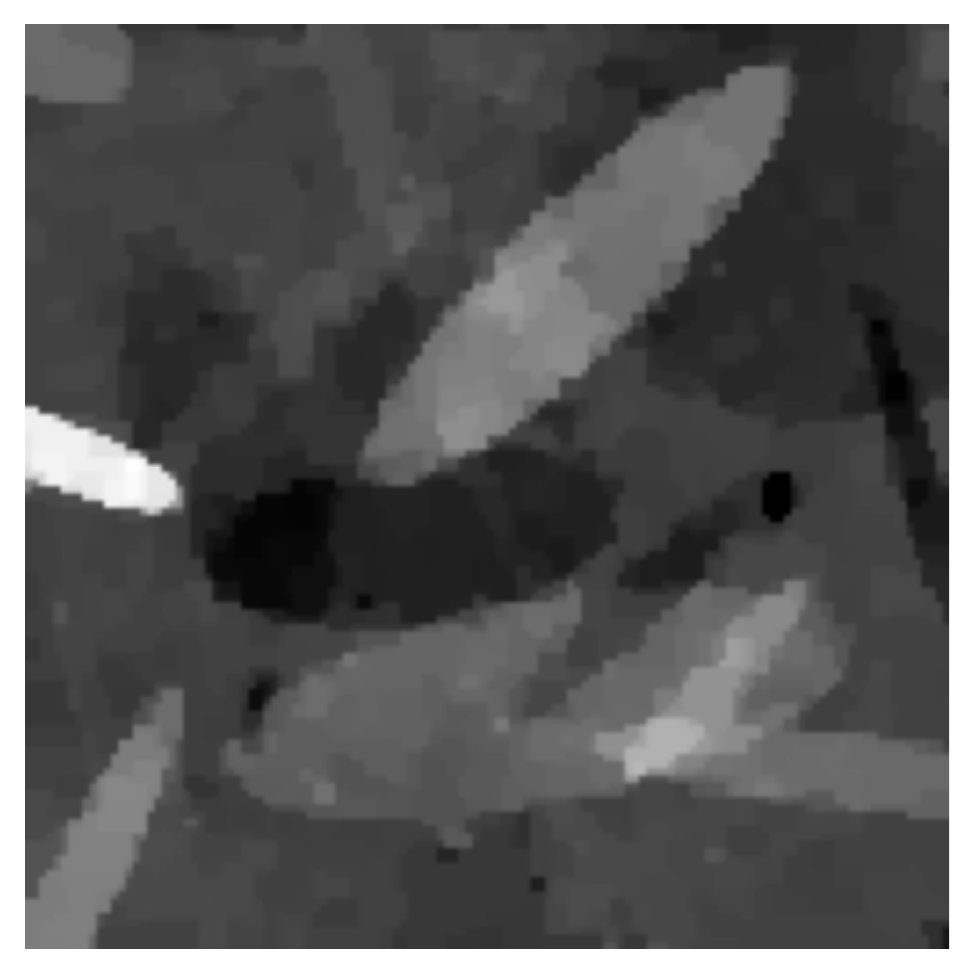}};
                \begin{scope}
                \spy[\spycolor,size=0.175\textwidth, every spy on node/.append style={line width = \W}] on (0.8,1.5) in node at (1.44, -1.5);  
                \end{scope}
        \end{tikzpicture}   
        &
        \begin{tikzpicture} [spy using outlines={rectangle, magnification=3, size=1cm, connect spies}, rounded corners]
                \node[anchor=south west,inner sep=0] (image) at (0,0) {\adjincludegraphics[width=0.185\textwidth]{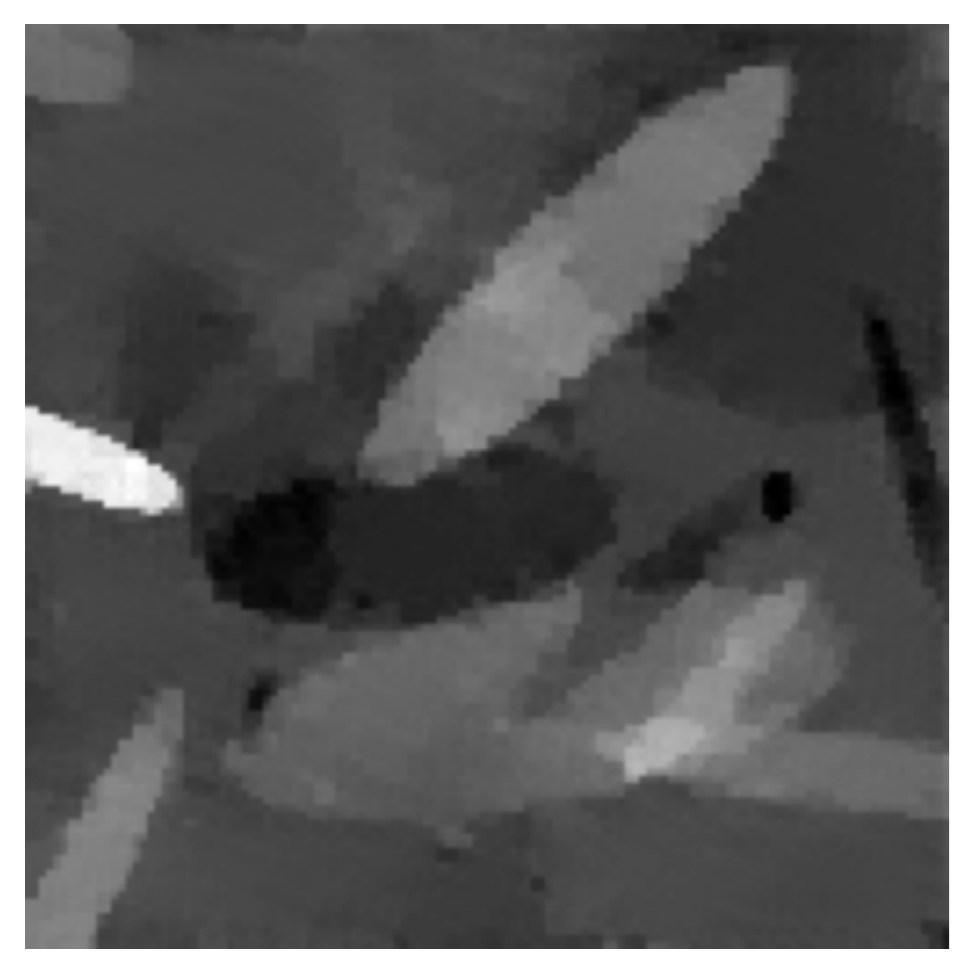}};
                \begin{scope}
                \spy[\spycolor,size=0.175\textwidth, every spy on node/.append style={line width = \W}] on (0.8,1.5) in node at (1.44, -1.5);  
                \end{scope}
        \end{tikzpicture}       
        &
        \begin{tikzpicture} [spy using outlines={rectangle, magnification=3, size=1cm, connect spies}, rounded corners]
                \node[anchor=south west,inner sep=0] (image) at (0,0) {\adjincludegraphics[width=0.185\textwidth]{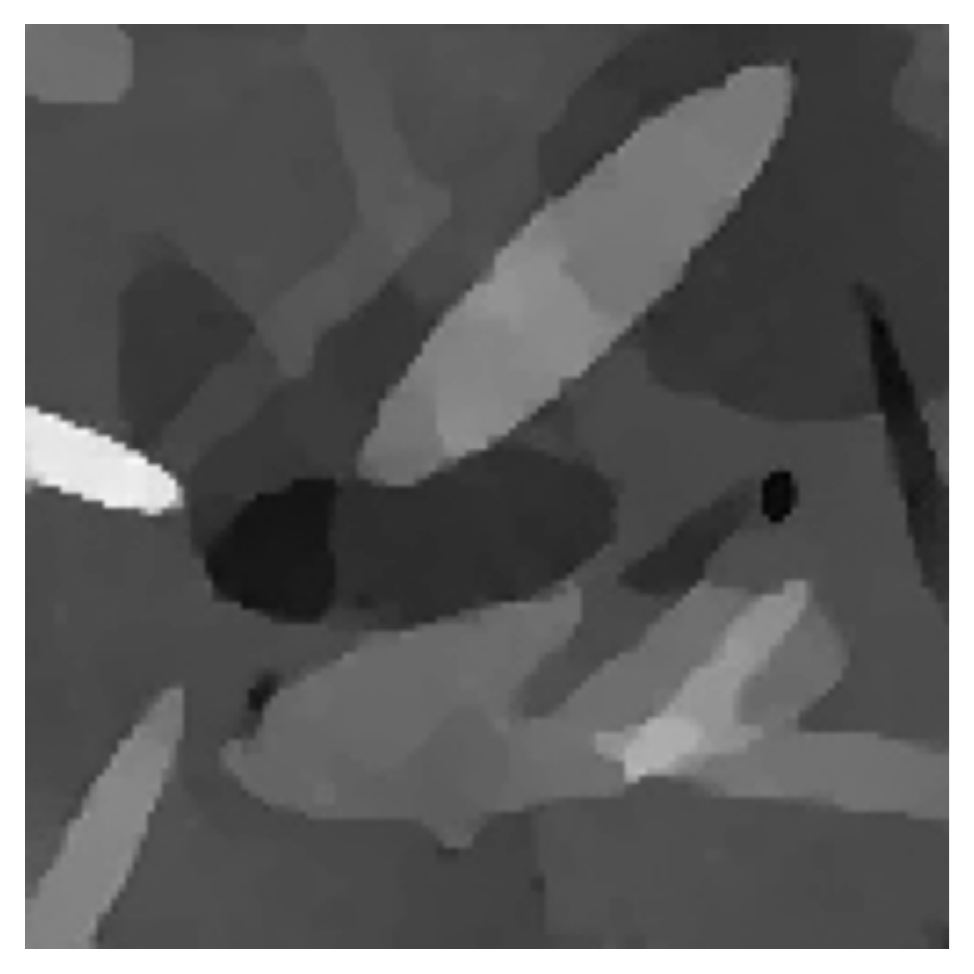}};
                \begin{scope}
                \spy[\spycolor,size=0.175\textwidth, every spy on node/.append style={line width = \W}] on (0.8,1.5) in node at (1.44, -1.5);  
                \end{scope}
        \end{tikzpicture}   
        \\
        & SSIM: 0.312 & SSIM: 0.834 & SSIM: 0.838 & SSIM: 0.880
        \\
        & PSNR: 20.20 & PSNR: 30.19 & PSNR: 30.10 & PSNR: 31.07
        \end{tabular}
        \caption{Reconstruction on a validation sample obtained with Filtered Back Projection (FBP) method, TV regularization, Adversarial Regularizer, and Wasserstein-based Projections (left to right). Bottom row shows expanded version of corresponding cropped region indicated by red box.}
        \label{fig:ellipses_reconstructions_2}
    \end{figure}


    \begin{figure}[t]
      \centering
      \small
        \setlength{\tabcolsep}{0.1pt}
        \begin{tabular}{ccccc}
            ground truth & FBP & TV & Adv. Reg. & Adv. Proj.
            \\

        \begin{tikzpicture} [spy using outlines={rectangle, magnification=3, size=1cm, connect spies}, rounded corners]
                \node[anchor=south west,inner sep=0] (image) at (0,0) {\adjincludegraphics[width=0.185\textwidth]{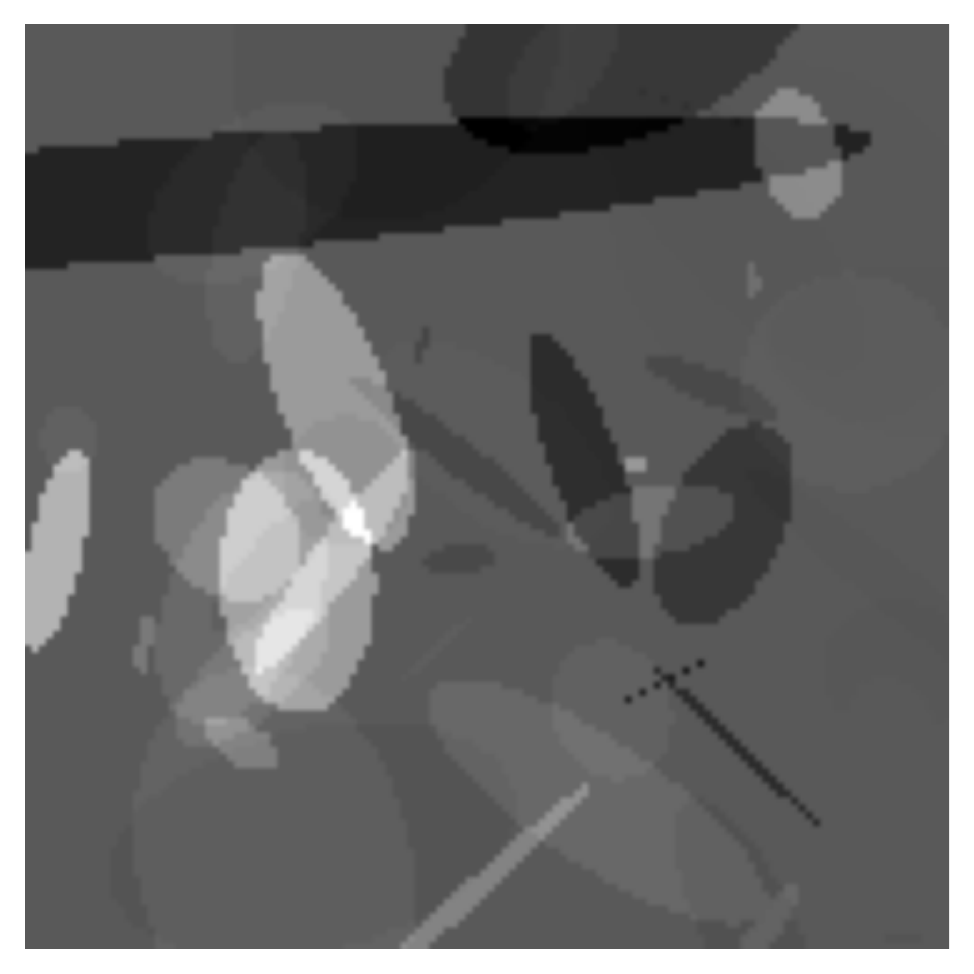}}; 
                \begin{scope}
                \spy[\spycolor,size=0.175\textwidth, every spy on node/.append style={line width = \W}] on (1.8,1.5) in node at (1.44, -1.5);  
                \end{scope}
        \end{tikzpicture}
        &
        \begin{tikzpicture} [spy using outlines={rectangle, magnification=3, size=1cm, connect spies}, rounded corners]
                \node[anchor=south west,inner sep=0] (image) at (0,0) {\adjincludegraphics[width=0.185\textwidth]{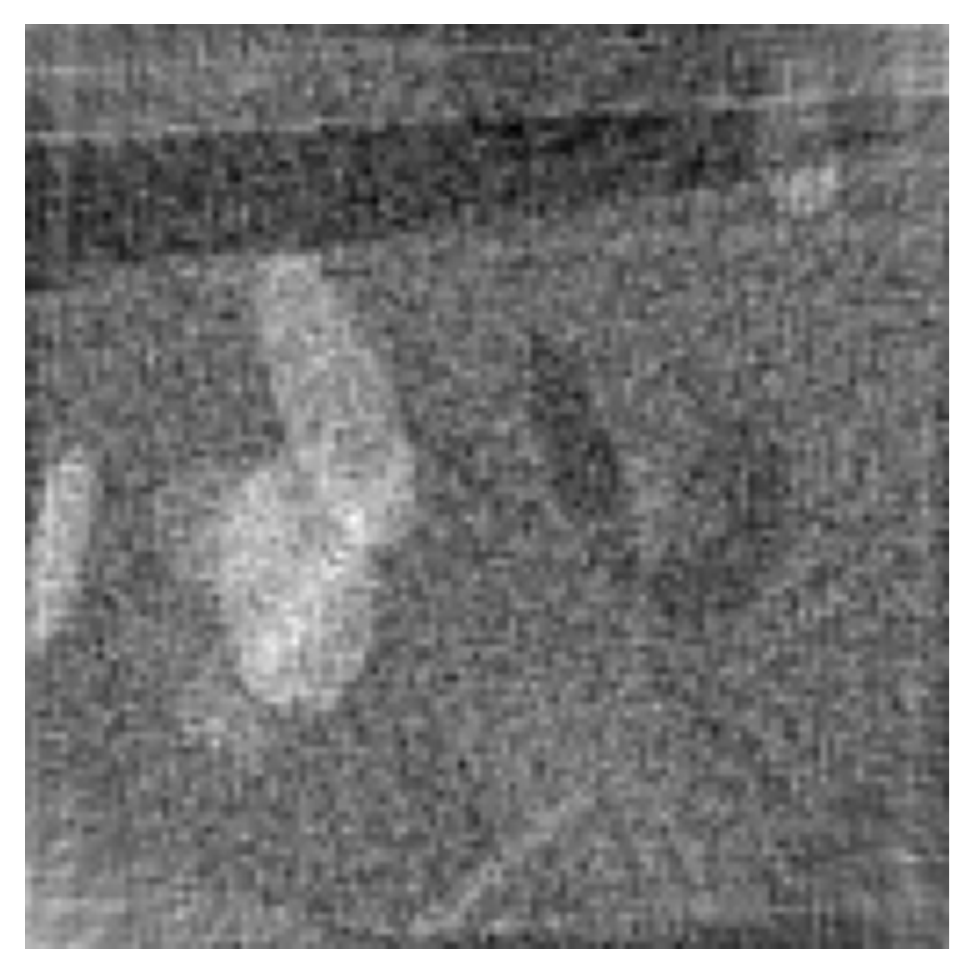}};
                \begin{scope}
                \spy[\spycolor,size=0.175\textwidth, every spy on node/.append style={line width = \W}] on (1.8, 1.5) in node at (1.44, -1.5);  
                \end{scope}
        \end{tikzpicture}       
        &
        \begin{tikzpicture} [spy using outlines={rectangle, magnification=3, size=1cm, connect spies}, rounded corners]
                \node[anchor=south west,inner sep=0] (image) at (0,0) {\adjincludegraphics[width=0.185\textwidth]{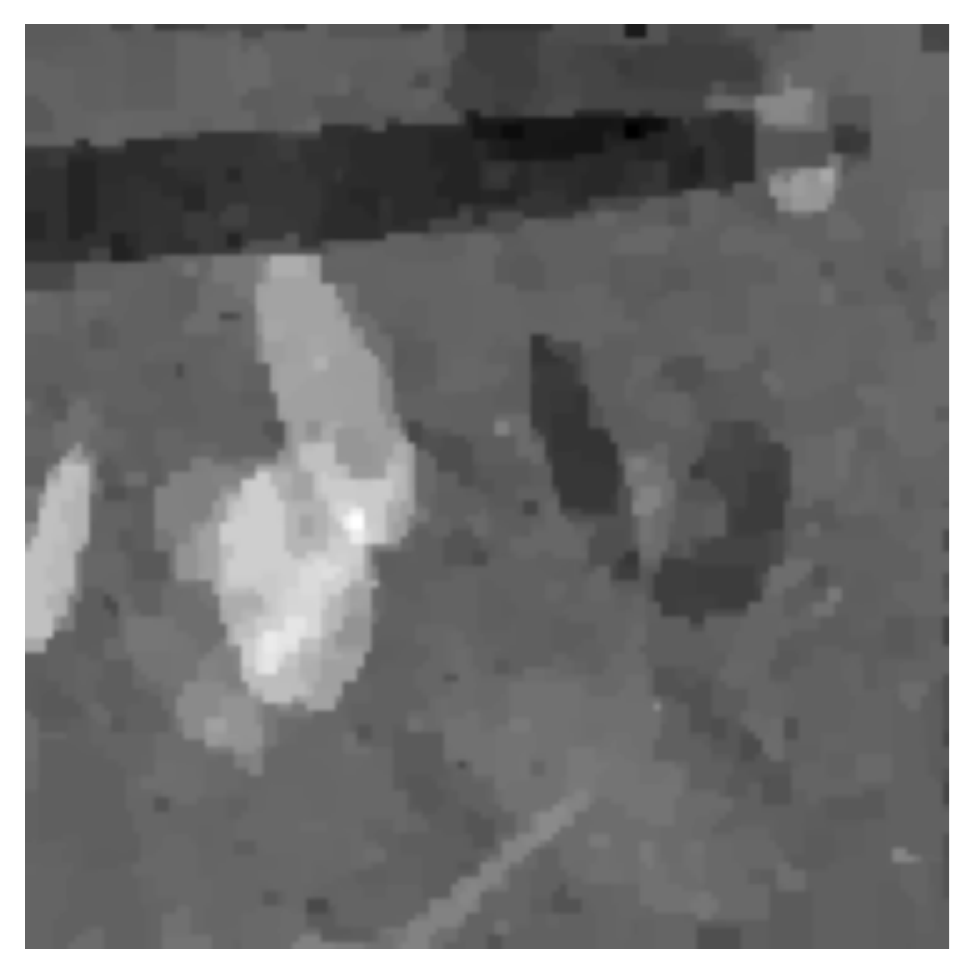}};
                \begin{scope}
                \spy[\spycolor,size=0.175\textwidth, every spy on node/.append style={line width = \W}] on (1.8,1.5) in node at (1.44, -1.5);  
                \end{scope}
        \end{tikzpicture}   
        &
        \begin{tikzpicture} [spy using outlines={rectangle, magnification=3, size=1cm, connect spies}, rounded corners]
                \node[anchor=south west,inner sep=0] (image) at (0,0) {\adjincludegraphics[width=0.185\textwidth]{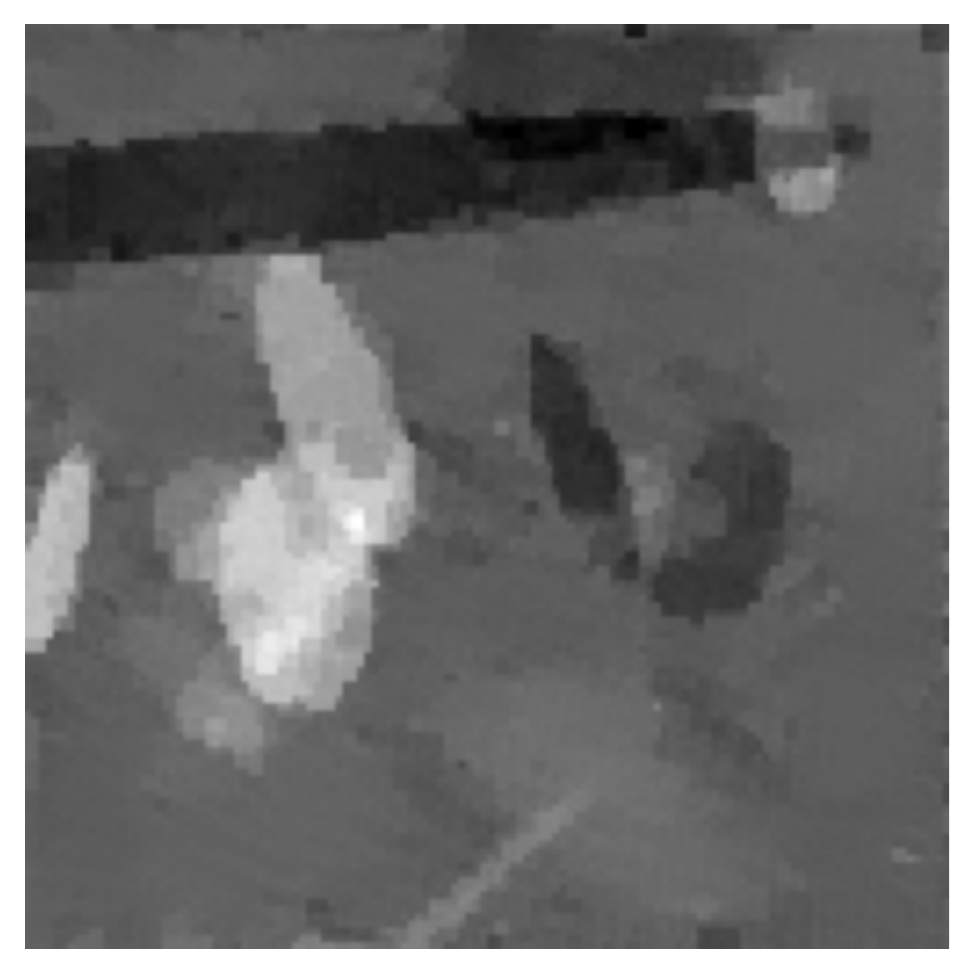}};
                \begin{scope}
                \spy[\spycolor,size=0.175\textwidth, every spy on node/.append style={line width = \W}] on (1.8,1.5) in node at (1.44, -1.5);  
                \end{scope}
        \end{tikzpicture}       
        &
        \begin{tikzpicture} [spy using outlines={rectangle, magnification=3, size=1cm, connect spies}, rounded corners]
                \node[anchor=south west,inner sep=0] (image) at (0,0) {\adjincludegraphics[width=0.185\textwidth]{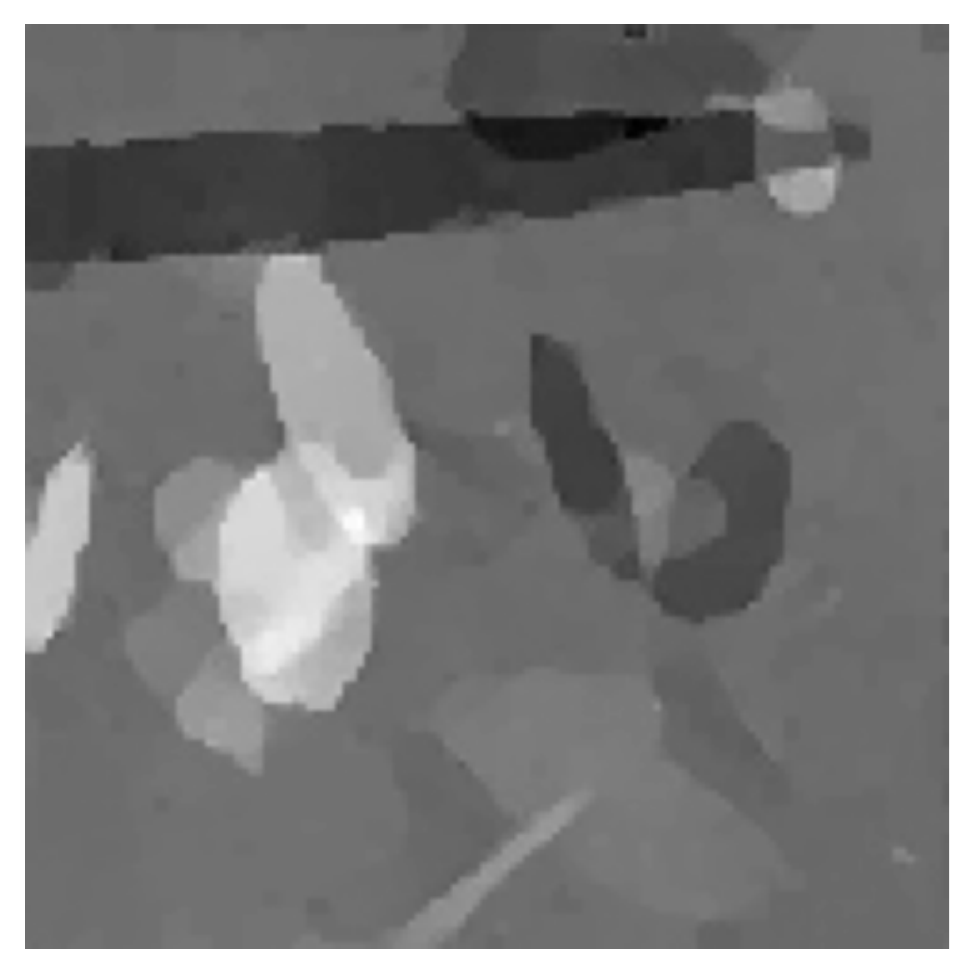}};
                \begin{scope}
                \spy[\spycolor,size=0.175\textwidth, every spy on node/.append style={line width = \W}] on (1.8,1.5) in node at (1.44, -1.5);  
                \end{scope}
        \end{tikzpicture}   
        \\
        & SSIM: 0.253 & SSIM: 0.769 & SSIM: 0.794 & SSIM: 0.841
        \\
        & PSNR: 19.19 & PSNR: 28.85 & PSNR: 29.04 & PSNR: 29.74
        \end{tabular}
        \caption{Reconstruction on a validation sample obtained with Filtered Back Projection (FBP) method, TV regularization, Adversarial Regularizer, and Wasserstein-based Projections (left to right). Bottom row shows expanded version of corresponding cropped region indicated by red box.}
        \label{fig:ellipses_reconstructions_1}
    \end{figure}

    \begin{figure}[t]
      \centering
      \small
        \setlength{\tabcolsep}{0.1pt}
        \begin{tabular}{ccccc}
            ground truth & FBP & TV & Adv. Reg. & Adv. Proj.
            \\

        \begin{tikzpicture} [spy using outlines={rectangle, magnification=3, size=1cm, connect spies}, rounded corners]
                \node[anchor=south west,inner sep=0] (image) at (0,0) {\adjincludegraphics[angle=180,width=0.185\textwidth]{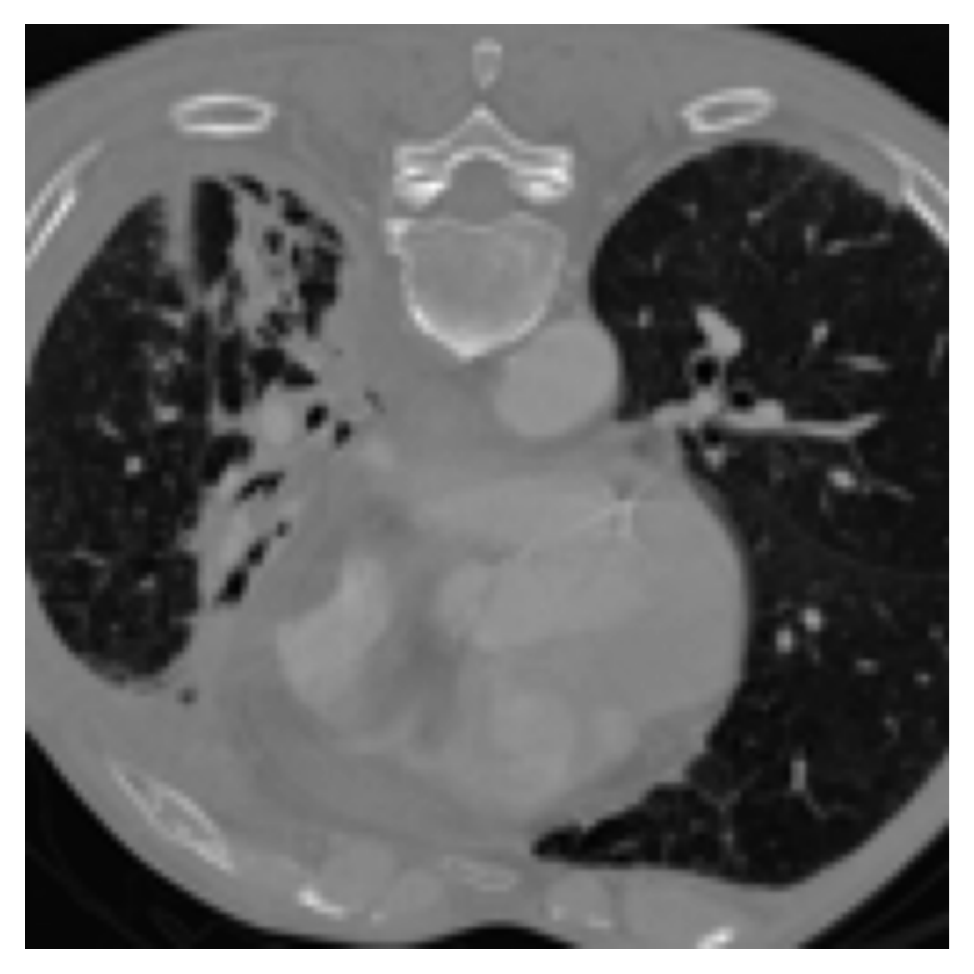}}; 
                \begin{scope}
                \spy[\spycolor,size=0.175\textwidth, every spy on node/.append style={line width = \W}] on (1.5,0.6) in node at (1.44, -1.5);  
                \end{scope}
        \end{tikzpicture}
        &
        \begin{tikzpicture} [spy using outlines={rectangle, magnification=3, size=1cm, connect spies}, rounded corners]
                \node[anchor=south west,inner sep=0] (image) at (0,0) {\adjincludegraphics[angle=180,width=0.185\textwidth]{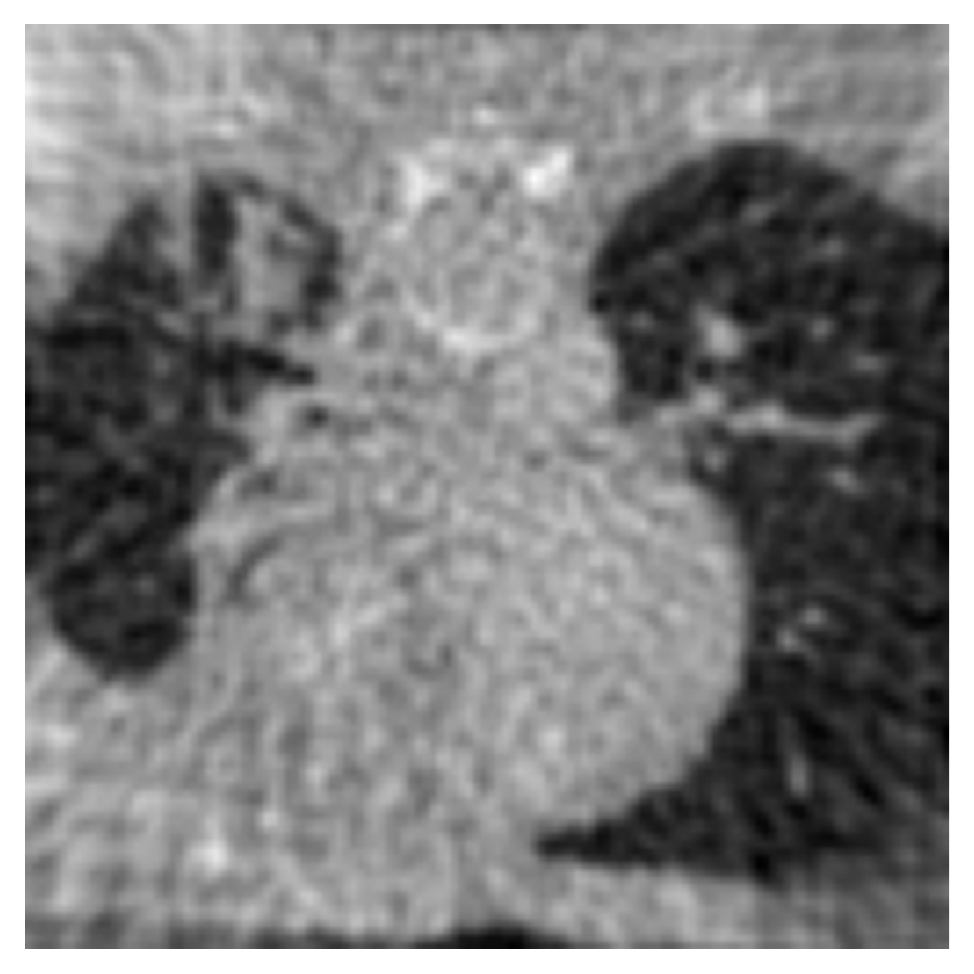}};
                \begin{scope}
                \spy[\spycolor,size=0.175\textwidth, every spy on node/.append style={line width = \W}] on (1.5,0.6) in node at (1.44, -1.5);  
                \end{scope}
        \end{tikzpicture}       
        &
        \begin{tikzpicture} [spy using outlines={rectangle, magnification=3, size=1cm, connect spies}, rounded corners]
                \node[anchor=south west,inner sep=0] (image) at (0,0) {\adjincludegraphics[angle=180,width=0.185\textwidth]{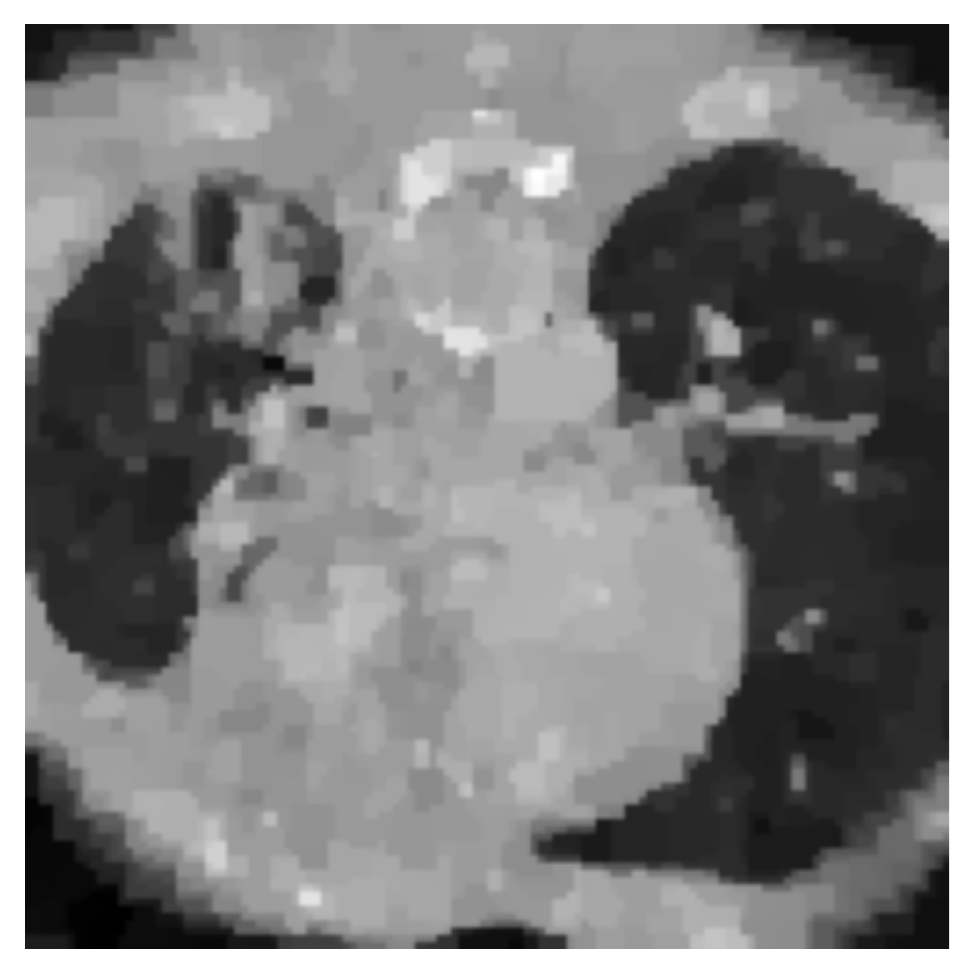}}; 
                \begin{scope}
                \spy[\spycolor,size=0.175\textwidth, every spy on node/.append style={line width = \W}] on (1.5,0.6) in node at (1.44, -1.5);  
                \end{scope}
        \end{tikzpicture}   
        &
        \begin{tikzpicture} [spy using outlines={rectangle, magnification=3, size=1cm, connect spies}, rounded corners]
                \node[anchor=south west,inner sep=0] (image) at (0,0) {\adjincludegraphics[angle=180,width=0.185\textwidth]{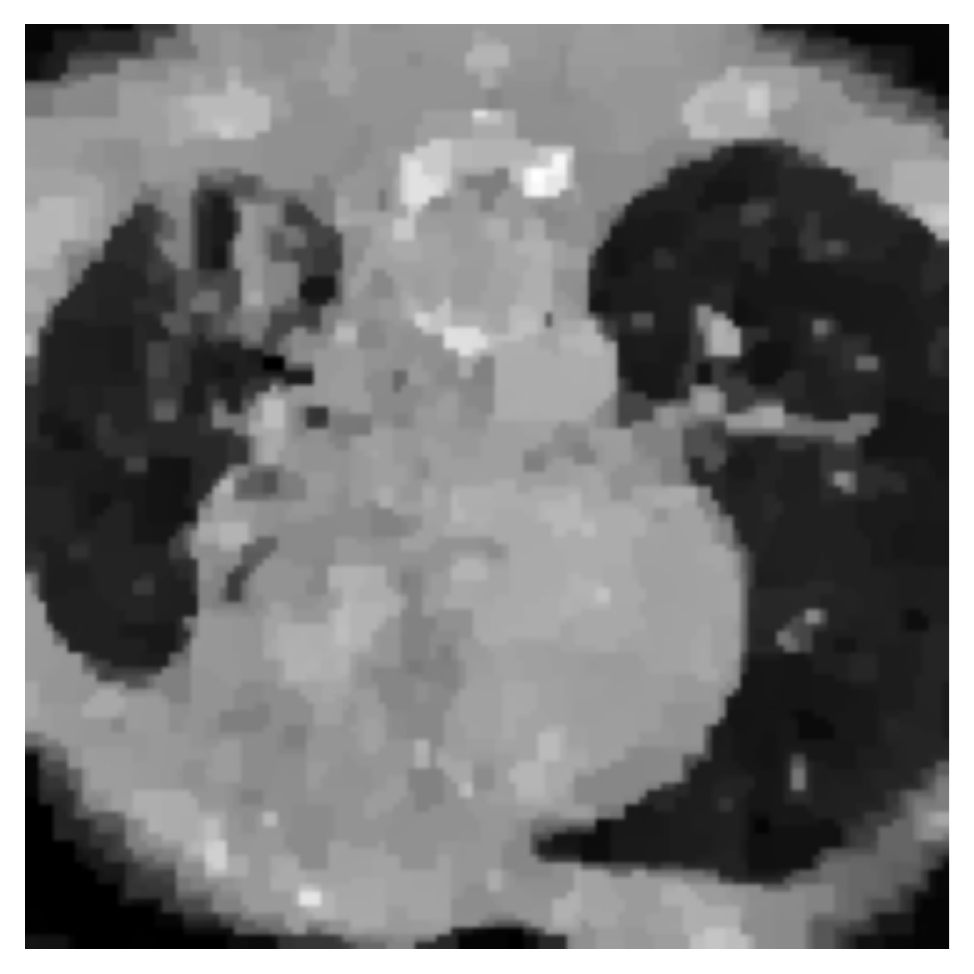}}; 
                \begin{scope}
                \spy[\spycolor,size=0.175\textwidth, every spy on node/.append style={line width = \W}] on (1.5,0.6) in node at (1.44, -1.5);  
                \end{scope}
        \end{tikzpicture}       
        &
        \begin{tikzpicture} [spy using outlines={rectangle, magnification=3, size=1cm, connect spies}, rounded corners]
                \node[anchor=south west,inner sep=0] (image) at (0,0) {\adjincludegraphics[angle=180,width=0.185\textwidth]{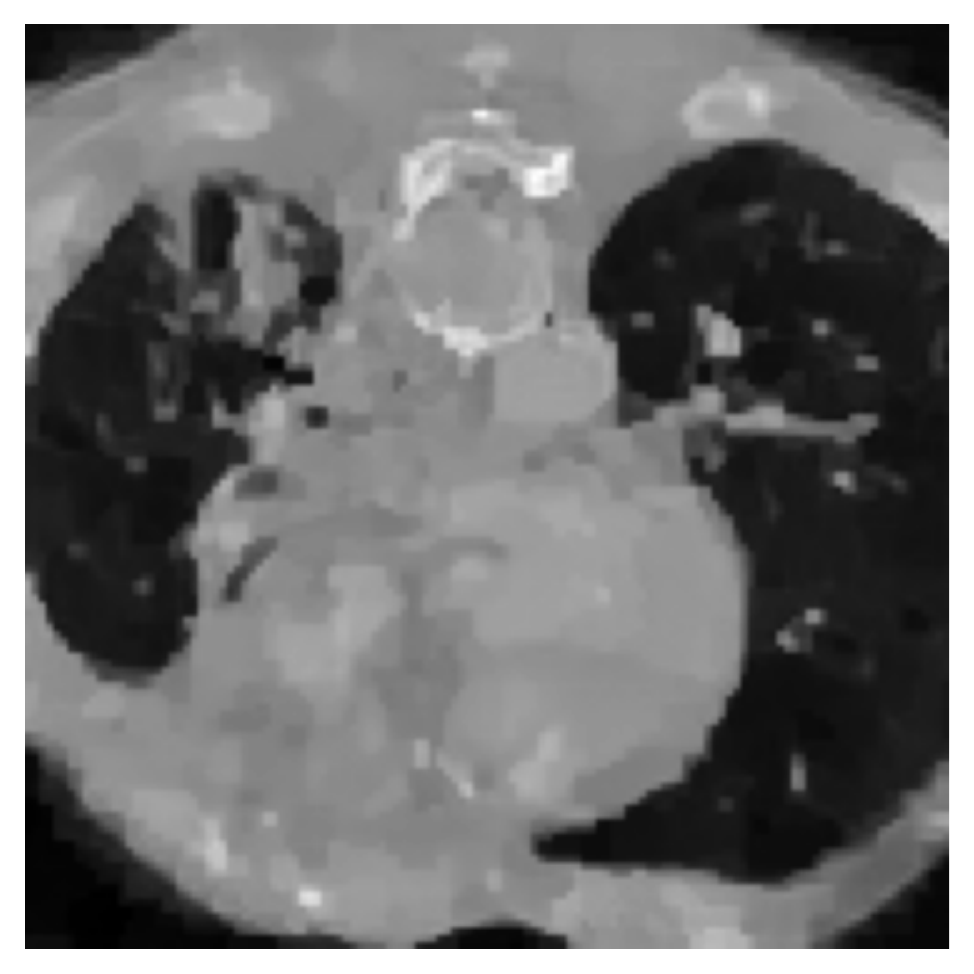}}; 
                \begin{scope}
                \spy[\spycolor,size=0.175\textwidth, every spy on node/.append style={line width = \W}] on (1.5,0.6) in node at (1.44, -1.5);  
                \end{scope}
        \end{tikzpicture}   
        \\
        & SSIM: 0.396 & SSIM: 0.679 & SSIM: 0.703 & SSIM: 0.740
        \\
        & PSNR: 15.89 & PSNR: 18.54 & PSNR: 22.51 & PSNR: 25.35
        \end{tabular}
        \caption{Reconstruction on a validation sample obtained with Filtered Back Projection (FBP) method, TV regularization, Adversarial Regularizer, and Wasserstein-based Projections (left to right). Bottom row shows expanded version of corresponding cropped region indicated by red box.}
        \label{fig:lodopab_reconstructions_2}
    \end{figure}
    
    \begin{figure}
      \centering
      \small
        \setlength{\tabcolsep}{0.1pt}
        \begin{tabular}{ccccc}
            ground truth & FBP & TV & Adv. Reg. & Adv. Proj.
            \\

        \begin{tikzpicture} [spy using outlines={rectangle, magnification=3, size=1cm, connect spies}, rounded corners]
                \node[anchor=south west,inner sep=0] (image) at (0,0) {\adjincludegraphics[angle=180,width=0.185\textwidth]{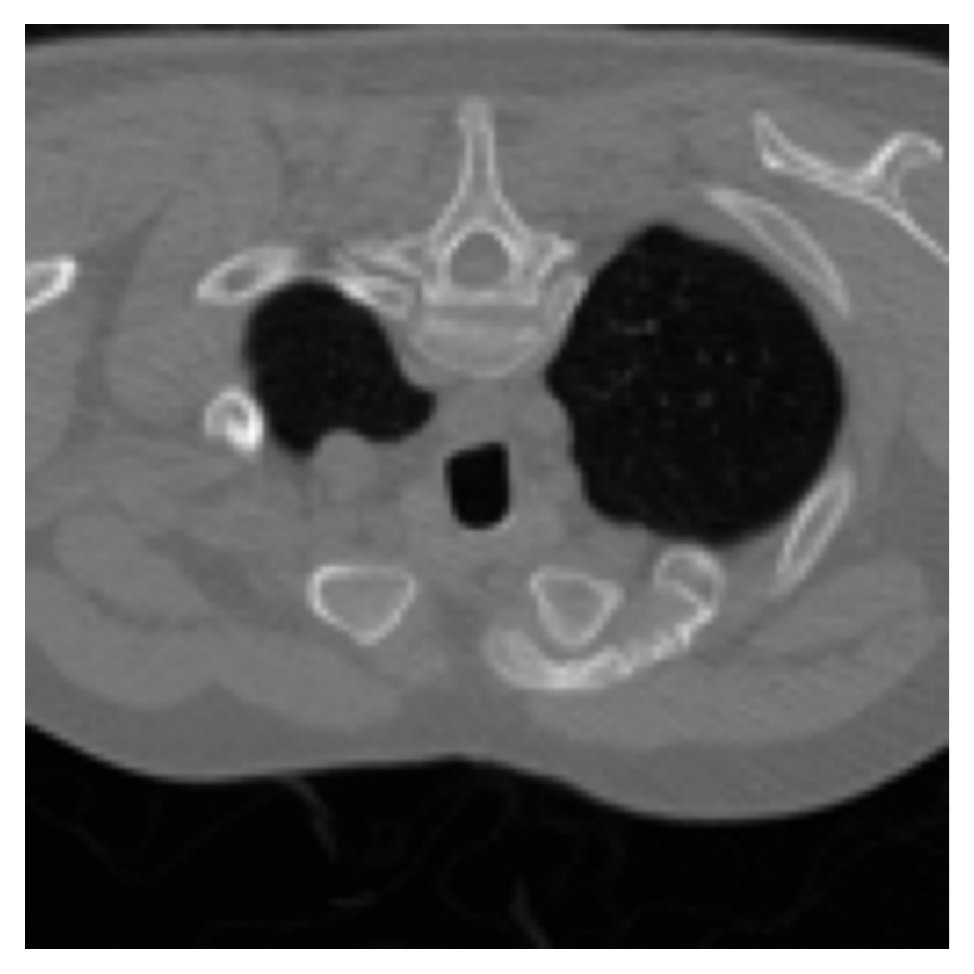}}; 
                \begin{scope}
                \spy[\spycolor,size=0.175\textwidth, every spy on node/.append style={line width = \W}] on (1.5,0.6) in node at (1.44, -1.5);  
                \end{scope}
        \end{tikzpicture}
        &
        \begin{tikzpicture} [spy using outlines={rectangle, magnification=3, size=1cm, connect spies}, rounded corners]
                \node[anchor=south west,inner sep=0] (image) at (0,0) {\adjincludegraphics[angle=180,width=0.185\textwidth]{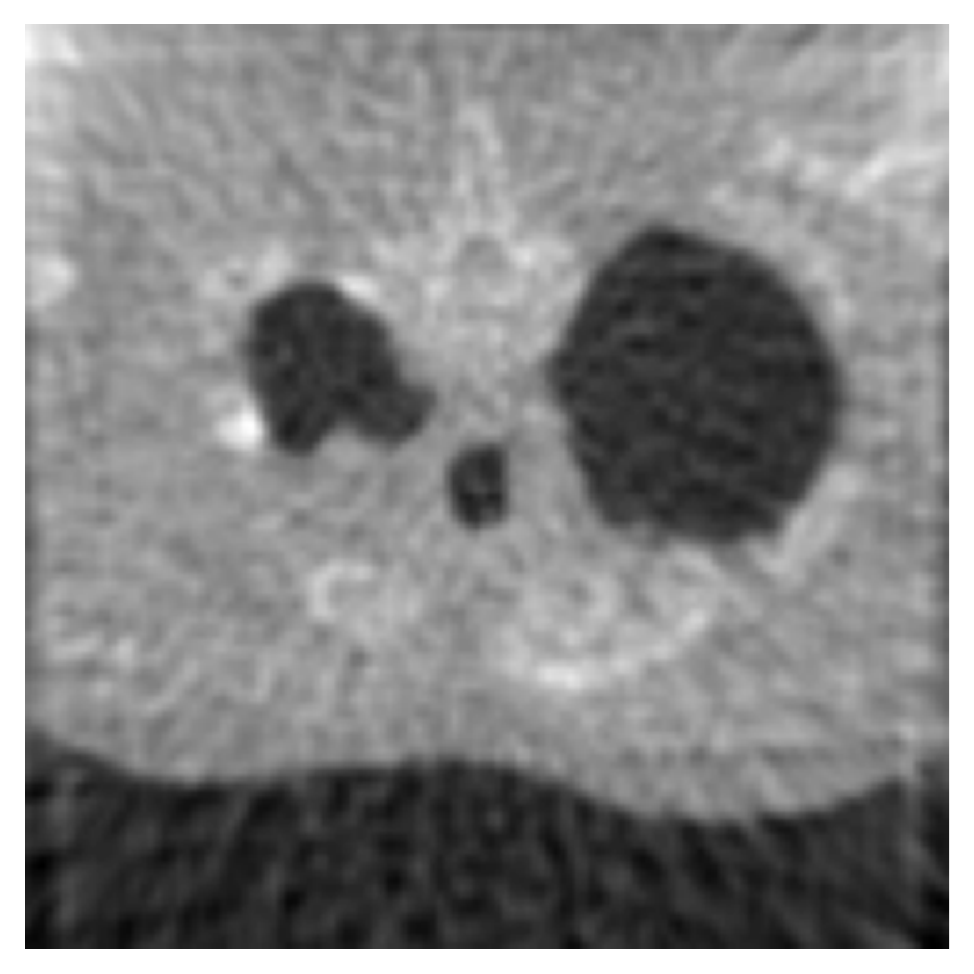}};
                \begin{scope}
                \spy[\spycolor,size=0.175\textwidth, every spy on node/.append style={line width = \W}] on (1.5,0.6) in node at (1.44, -1.5);  
                \end{scope}
        \end{tikzpicture}       
        &
        \begin{tikzpicture} [spy using outlines={rectangle, magnification=3, size=1cm, connect spies}, rounded corners]
                \node[anchor=south west,inner sep=0] (image) at (0,0) {\adjincludegraphics[angle=180,width=0.185\textwidth]{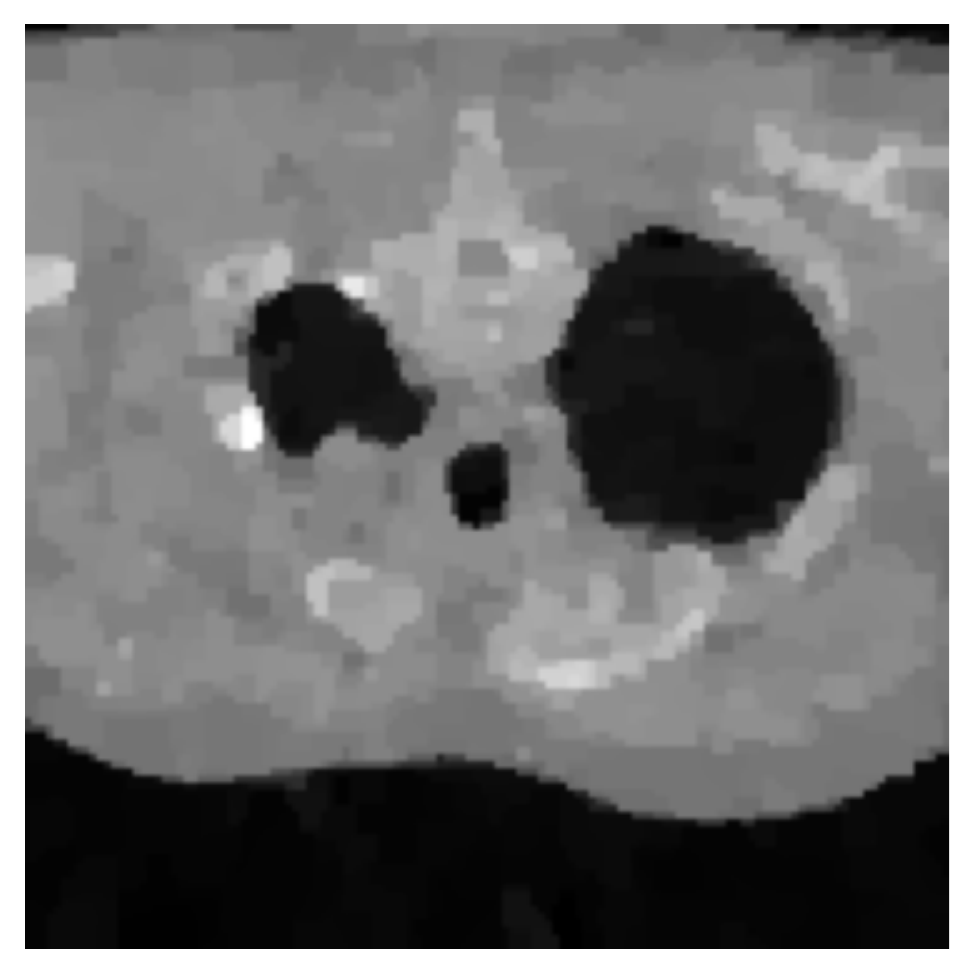}}; 
                \begin{scope}
                \spy[\spycolor,size=0.175\textwidth, every spy on node/.append style={line width = \W}] on (1.5,0.6) in node at (1.44, -1.5);  
                \end{scope}
        \end{tikzpicture}   
        &
        \begin{tikzpicture} [spy using outlines={rectangle, magnification=3, size=1cm, connect spies}, rounded corners]
                \node[anchor=south west,inner sep=0] (image) at (0,0) {\adjincludegraphics[angle=180,width=0.185\textwidth]{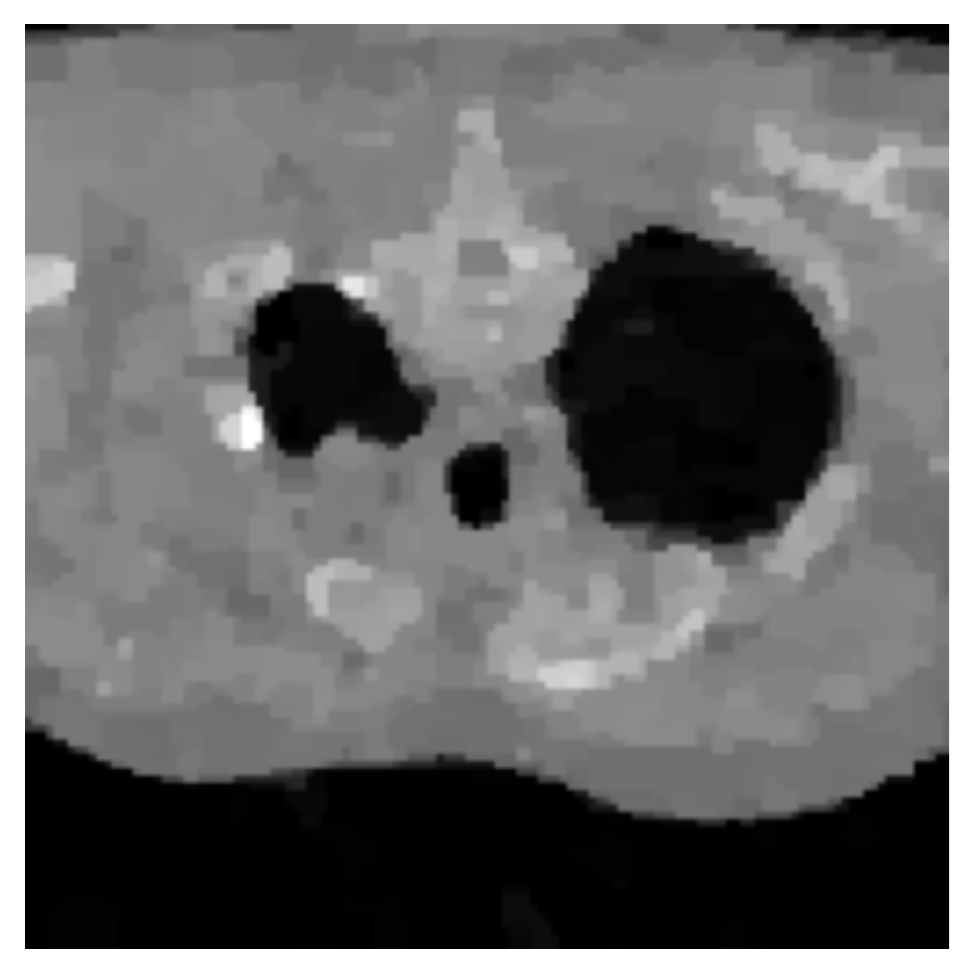}}; 
                \begin{scope}
                \spy[\spycolor,size=0.175\textwidth, every spy on node/.append style={line width = \W}] on (1.5,0.6) in node at (1.44, -1.5);  
                \end{scope}
        \end{tikzpicture}       
        &
        \begin{tikzpicture} [spy using outlines={rectangle, magnification=3, size=1cm, connect spies}, rounded corners]
                \node[anchor=south west,inner sep=0] (image) at (0,0) {\adjincludegraphics[angle=180,width=0.185\textwidth]{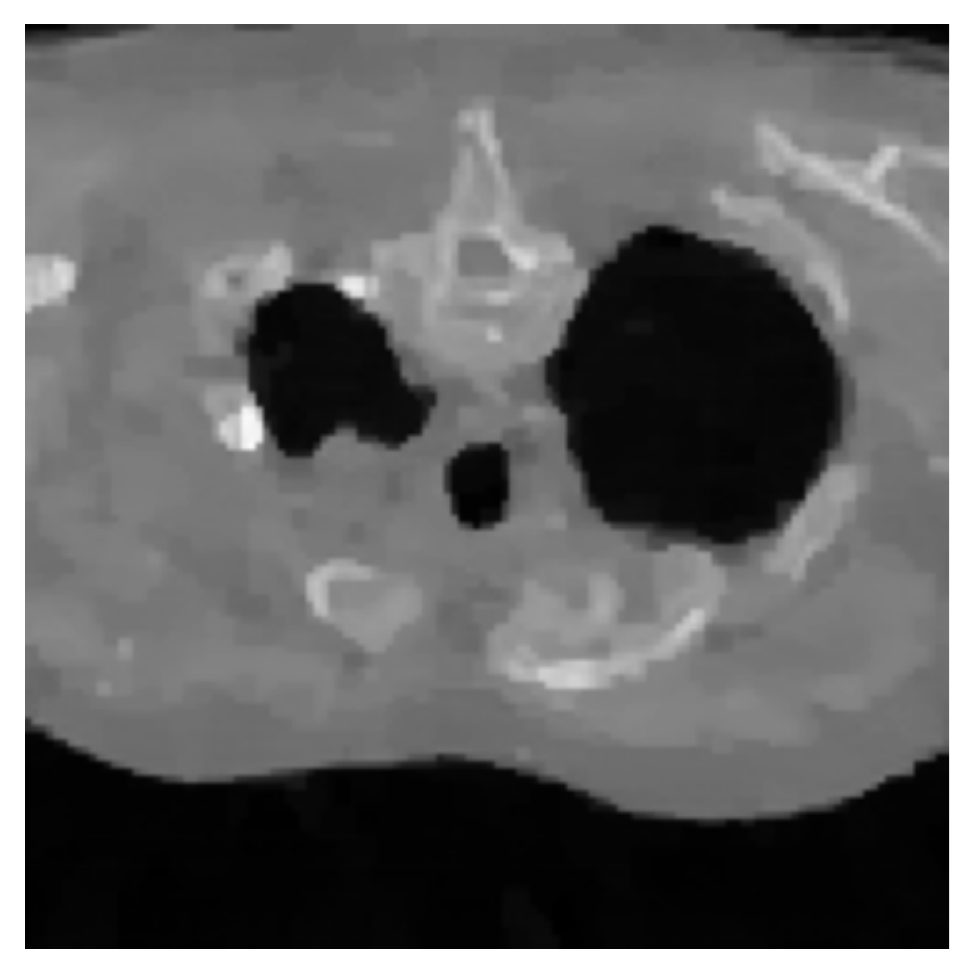}}; 
                \begin{scope}
                \spy[\spycolor,size=0.175\textwidth, every spy on node/.append style={line width = \W}] on (1.5,0.6) in node at (1.44, -1.5);  
                \end{scope}
        \end{tikzpicture}   
        \\
        & SSIM: 0.487 & SSIM: 0.767 & SSIM: 0.707 & SSIM: 0.799
        \\
        & PSNR: 15.62 & PSNR: 22.15 & PSNR: 26.33 & PSNR: 28.65
        \end{tabular}
        \caption{Reconstruction on a validation sample obtained with Filtered Back Projection (FBP) method, TV regularization, Adversarial Regularizer, and Wasserstein-based Projections (left to right). Bottom row shows expanded version of corresponding cropped region indicated by red box.}
        \label{fig:lodopab_reconstructions_3}
    \end{figure}
    
\end{document}